\documentclass[letterpaper]{article}

\usepackage{ijcai13}
\usepackage{times}

\usepackage{amsmath,amssymb}
\usepackage{url}
\usepackage[defblank]{paralist}
\usepackage{xspace}
\usepackage{color}
\usepackage{comment}
\usepackage{enumitem}
\usepackage{multirow}

\usepackage{tikz}
\usetikzlibrary{shapes,decorations,shadows}
\usepgflibrary{decorations.pathmorphing}
\usepgflibrary{arrows}
\usetikzlibrary{calc}
\usetikzlibrary{trees}
\usetikzlibrary{backgrounds}
\usepackage{fancybox}

\newtheorem{theorem}{Theorem}[section]
\newtheorem{definition}[theorem]{Definition}
\newtheorem{example}[theorem]{Example}
\newtheorem{proposition}[theorem]{Proposition}
\newtheorem{lemma}[theorem]{Lemma}
\newtheorem{claim}[theorem]{Claim}

\newcommand{\qed}{\hfill $\Box$}
\newenvironment{proof}{\noindent {\bf Proof.}}{\qed\medskip}

\newcommand{\A}{\mathcal{A}}
\newcommand{\B}{\mathcal{B}}
\newcommand{\C}{\mathcal{C}}

\newcommand{\I}{\mathcal{I}}
\newcommand{\J}{\mathcal{J}}

\newcommand{\K}{\mathcal{K}}
\renewcommand{\L}{\mathcal{L}}
\newcommand{\M}{\mathcal{M}}
\newcommand{\N}{\mathcal{N}}

\newcommand{\R}{\mathcal{R}}

\newcommand{\T}{\mathcal{T}}
\newcommand{\V}{\mathcal{V}}

\newcommand{\tup}[1]{\langle #1\rangle}

\newcommand{\card}[1]{|{#1}|}

\newcommand{\st}{ \ | \ }

\newcommand{\ISA}{\sqsubseteq}

\newcommand{\AND}{\sqcap}

\newcommand{\SOMET}[1]{\ensuremath{\exists #1}}

\newcommand{\NOT}{\neg}

\newcommand{\kb}{\ensuremath{\K=\tup{\T,\A}}\xspace}

\newcommand{\dom}[1][\I]{\Delta^{#1}}
\newcommand{\Int}[2][\I]{#2^{#1}}

\newcommand{\cert}{\mathit{cert}}

\newcommand{\dllite}{\textit{DL-Lite}\xspace}

\newcommand{\dlliter}{\textit{DL-Lite}\ensuremath{_{\R}}\xspace}
\newcommand{\dlliterpos}{\textit{DL-Lite}\ensuremath{^\textit{\,pos}_{\R}}\xspace}
\newcommand{\dlliterdfs}{\textit{DL-Lite}\ensuremath{_{\scriptscriptstyle\mathit{RDFS}}}\xspace}

\newcommand{\owltwo}{OWL\,2\xspace}
\newcommand{\owlql}{OWL\,2\,QL\xspace}

\newcommand{\NLOGSPACE}{\textsc{NLogSpace}\xspace}
\newcommand{\PTIME}{\textsc{PTime}\xspace}
\newcommand{\PSPACE}{\textsc{PSpace}\xspace}
\newcommand{\EXPTIME}{\textsc{ExpTime}\xspace}

\newcommand{\NP}{\textsc{NP}\xspace}

\newcommand{\Sat}{\text{\sc Sat}}
\newcommand{\Mod}{\text{\sc Mod}}

\newcommand{\CQ}{\ensuremath{\mathsf{CQ}}\xspace}
\newcommand{\UCQ}{\ensuremath{\mathsf{UCQ}}\xspace}
\newcommand{\Uni}{{\mathcal U}}
\newcommand{\Ind}{\mathsf{Ind}}
\newcommand{\Null}{\mathsf{Null}}

\newcommand{\tail}{\mathsf{tail}}
\newcommand{\gpath}{\mathsf{path}}
\newcommand{\ttype}[2][]{\mathbf{t}^{#2}_{#1}}
\newcommand{\rtype}[2][]{\mathbf{r}^{#2}_{#1}}
\newcommand{\stype}[2][]{\mathbf{s}^{#2}_{#1}}
\newcommand{\qtype}[2][]{\mathbf{q}^{#2}_{#1}}
\newcommand{\srkb}[1][B]{\mathcal{S}_{#1}}
\newcommand{\tgkb}[1][B]{\mathcal{X}_{#1}}
\newcommand{\consc}[1][\T_1 \cup \T_{12}]{\mathit{cons}_\C(#1)}
\newcommand{\consr}[1][\T_1 \cup \T_{12}]{\mathit{cons}_\R(#1)}
\newcommand{\QQ}{\mathsf{Q}}
\newcommand{\BB}{\mathbf{B}}
\newcommand{\RR}{\mathbf{R}}
\newcommand{\NN}{{\mathbf N}}
\newcommand{\PP}{{\mathbf P}}
\newcommand{\Au}{\mathbb{A}}
\newcommand{\cn}{{\it can}}
\newcommand{\acan}[1][\K]{\Au_{#1}^{\cn}}
\newcommand{\fn}{{\it fin}}
\newcommand{\agfin}{\Au_{\fn}}
\newcommand{\sol}{{\it mod}}
\newcommand{\asol}[1][\K]{\Au_{#1}^{\sol}}
\newcommand{\Bau}{\mathbb{B}}
\renewcommand{\ng}{{\it ng}}
\newcommand{\run}{\mathbf{r}}
\renewcommand{\succ}{\mathsf{succ}}
\newcommand{\len}{\mathsf{len}}
\newcommand{\dep}{\mathsf{dep}}
\newcommand{\pos}{\mathit{pos}}

\newcommand{\concept}[1]{{#1}{(\cdot)}}
\newcommand{\role}[1]{{#1}{(\cdot,\cdot)}}

\tikzset{ %
  taxonomy/.style={anchor=west, grow via three points={one child at (0.2,-0.5)
      and two children at (0.2,-0.5) and (0.2,-1)}, %
    edge from parent path={(\tikzparentnode.south) |- (\tikzchildnode.west)},
    draw=mydarkgray}, %
  highlight/.style={draw=myorange, rounded corners, text=black},%
}

\definecolor{mygreen}{HTML}{008888}
\definecolor{mygray}{HTML}{cccccc}
\definecolor{mydarkgray}{HTML}{aaaaaa}
\definecolor{myorange}{HTML}{cc6633}

\def \drawwidetable[#1,#2,#3]{
  \node[anchor=south east,minimum height=0.5cm] (box) at
  ($(#2)+(0,0.7)$) {\footnotesize $#1$}; %
  \draw[mydarkgray, rounded corners] %
  ($(#3.south west)-(1.8,0)$) rectangle ($(box.south east)+(0.5,0.5)$) %
  ($(#3.south west)-(1.8,0)-(#3.south)+(box.south east)$) -- %
  ($(box.south east)+(0.5,0)$); 
}

\def\entity[#1,#2]#3;{
  \node[draw,drop shadow={opacity=.4,shadow xshift=0.04, shadow
    yshift=-0.04},color=gray!30!black,text=black,fill=white,rounded corners=3] (#1) at #2 {#3};
}

\def\medge[#1,#2,#3,#4];{
  \draw[decorate,decoration={coil,aspect=0,segment length=20pt,amplitude=1.5pt},
  -latex,color=gray!40!black,very thin,#4] (#1) -- #3 (#2);
}

\def\medger[#1,#2,#3,#4];{
  \draw[decorate,decoration={coil,aspect=0,segment length=20pt,amplitude=1.5pt},
  -latex,color=gray!40!black,very thin,double,#4] (#1) -- #3 (#2);
}

\def\medgecurve[#1,#2,#3,#4,#5];{
  \draw[decorate,decoration={coil,aspect=0,segment length=20pt},
  -latex,color=gray!40!black,#4] (#1) .. #3 controls #5 .. (#2);
}

\def\disjmedge[#1,#2,#3,#4];{ 
 \draw[decorate,decoration={coil,aspect=0,segment length=20pt},
  -latex,color=gray!40!black,#4]%
 (#1) -- #3 (#2); 
 \draw[decorate,decoration={markings, mark=at position 0.5 with
   {\draw[-] (-4pt,-4pt) -- (4pt, 4pt);} }]%
 (#1) -- (#2); 
}

\def\redge[#1,#2,#3];{
  \draw[->] (#1) -- #3 (#2);
}

\def\redgecurve[#1,#2,#3,#4];{
  \draw[->] (#1) .. #3 controls #4 .. (#2);
}

\def\isaedge[#1,#2,#3,#4];{ 
  \draw[-triangle 60,color=black!20!black,#4,fill=white] (#1) -- #3
  (#2);  
}

\def\isaedgecurve[#1,#2,#3,#4,#5];{ 
  \draw[-triangle 60,color=black!20!black,#4,fill=white] (#1) .. #3
  controls #5 .. (#2); 
}

\def\disedge[#1,#2,#3,#4];{ 
  \draw[triangle 60-triangle 60,color=black!20!black,#4,fill=white]
  [postaction=decorate,decoration={markings, mark=at position 0.5 with {\draw[-]
      (-3pt,-3pt) -- (3pt, 3pt);} }] 
  (#1) -- #3 (#2); 
}

\def\tpath[#1,#2,#3,#4];{
    \draw[#4,-diamond,fill=white,decorate,decoration=zigzag] (#1) -- #3 (#2);
}

\tikzset{
  sshadow/.style={opacity=.25, shadow xshift=0.05, shadow yshift=-0.06},
}

\def\schemel[#1,#2,#3,#4,#5,#6]#7{ %
  \node[draw, diamond, shape aspect=#3, rotate=#2, minimum size=#1, %
  bottom color=green!55, top color=green!25, color=green!65!black, %
  drop shadow={sshadow,color=green!60!black}] (#5) at #6
  {\textcolor{green!40}{bla}}; %
  \node at #6 {#7};%

  \node[anchor=north] at (#5.south) {#4}; 
}

\def\schemer[#1,#2,#3,#4,#5,#6]#7{ %
  \node[draw, diamond, shape aspect=#3, rotate=#2, minimum size=#1, %
  bottom color=green!65, top color=green!30, color=green!60!black, %
  drop shadow={sshadow,color=green!65!black}] (#5) at #6
  {\textcolor{green!53}{bla}}; %
  \node at #6 {#7}; %

  \node[anchor=north] at (#5.south) {#4}; 
}

\def\tboxl[#1,#2,#3,#4,#5]#6{%
  \node[draw, drop shadow={opacity=.35}, minimum height=#1, minimum width=#2, %
  inner color=blue!45, outer color=blue!55, color=blue!40!black] (#4) at #5 {}; %
  \node[anchor=#3,inner sep=2pt] at (#4.#3) {#6};%
}

\def\tboxr[#1,#2,#3,#4,#5]#6{%
  \node[draw, drop shadow={opacity=.35}, minimum height=#1, minimum width=#2, %
  inner color=blue!35, outer color=blue!45, color=blue!50!black] (#4) at #5 {}; %
  \node[anchor=#3,inner sep=2pt] at (#4.#3) {#6}; %
}

\def\aboxlold[#1,#2,#3,#4,#5]#6;{
  \node[cylinder,draw,shape border rotate=90,aspect=#3,minimum height=#1,
  minimum width=#2, top color=white, bottom color=orange!80] 
  (#4) at #5 {#6};
}

\def\aboxrold[#1,#2,#3,#4,#5]#6;{
  \node[cylinder,draw,shape border rotate=90,aspect=#3,minimum
  height=#1, minimum width=#2, top color=white, bottom
  color=orange!60] (#4) at #5 {#6};  
}

\def\aboxl[#1,#2,#3,#4,#5]#6{%
  \node[draw, cylinder, alias=cyl, shape border rotate=90, aspect=#3, %
  minimum height=#1, minimum width=#2, outer sep=-0.5\pgflinewidth, %
  color=orange!40!black, left color=orange!70, right color=orange!80, middle
  color=white] (#4) at #5 {};%
  \node at #5 {#6};%
  \fill [orange!30] let \p1 = ($(cyl.before top)!0.5!(cyl.after top)$), \p2 =
  (cyl.top), \p3 = (cyl.before top), \n1={veclen(\x3-\x1,\y3-\y1)},
  \n2={veclen(\x2-\x1,\y2-\y1)} in (\p1) ellipse (\n1 and \n2); }
    
\def\aboxr[#1,#2,#3,#4,#5]#6{%
  \node[draw, cylinder, alias=cyl, shape border rotate=90, aspect=#3, %
  minimum height=#1, minimum width=#2, outer sep=-0.5\pgflinewidth, %
  color=orange!50!black, left color=orange!50, right color=orange!60, middle
  color=white] (#4) at #5 {};%
  \node at #5 {#6};%
  \fill [orange!20] let \p1 = ($(cyl.before top)!0.5!(cyl.after top)$), \p2 =
  (cyl.top), \p3 = (cyl.before top), \n1={veclen(\x3-\x1,\y3-\y1)},
  \n2={veclen(\x2-\x1,\y2-\y1)} in (\p1) ellipse (\n1 and \n2); }

\def\kbbox[#1,#2,#3,#4,#5]#6{%
  \draw[dashed] node[draw,color=gray!50,minimum height=#1,minimum width=#2] 
  (#4) at #5 {}; 
  \node[anchor=#3,inner sep=2pt] at (#4.#3) {#6}; 
}

\def\soledge[#1,#2,#3,#4];{
        \draw[dashed,-latex,#4] (#1) -- #3 (#2);
}

\def\ind[#1,#2,#3,#4,#5];{
      \node[circle,draw,inner sep=1.5pt,outer sep=0.5, color=#1!80!black,fill=#1!80,#2,#3] (#4) at #5 {};  
}

\def\inda[#1,#2,#3,#4];{
      \node[circle,draw,inner sep=0.7pt,outer sep=0.5, #1,#2] (#3) at #4 {};  
}

\def \rectnode[#1,#2]{
  \node[color=gray,draw,minimum width=#1,minimum height=#2]
}

\newif\ifextendedversion
\extendedversiontrue

\title{Exchanging OWL 2 QL Knowledge Bases}

\author{Marcelo Arenas\\
 PUC Chile \&\\
 Univ.~of Oxford, U.K.\\
 marenas@ing.puc.cl
 \And
 Elena Botoeva\\
 Free U.\ of Bolzano\\
 Italy\\
 botoeva@inf.unibz.it
 \And
 Diego Calvanese\\
 Free U.\ of Bolzano, Italy \&\\
 TU Vienna, Austria\\
 calvanese@inf.unibz.it
 \And
 Vladislav Ryzhikov\\
 Free U.\ of Bolzano\\
 Italy\\
 ryzhikov@inf.unibz.it
}

\begin{document}

\maketitle

\begin{abstract}
  Knowledge base exchange is an important problem in the area of data exchange
  and knowledge representation, where one is interested in exchanging
  information between a source and a target knowledge base connected through a
  mapping. In this paper, we study this fundamental problem for knowledge bases
  and mappings expressed in \owlql, the profile of \owltwo based on the
  description logic \dlliter. More specifically, we consider the problem of
  computing universal solutions, identified as one of the most desirable
  translations to be materialized, and the problem of computing
  \UCQ-representations, which optimally capture in a target TBox the
  information that can be extracted from a source TBox and a mapping by means
  of unions of conjunctive queries. For the former we provide a novel
  automata-theoretic technique, and complexity results that range from \NP to
  \EXPTIME, while for the latter we show \NLOGSPACE-completeness.
\end{abstract}

\section{Introduction}
\label{sec-introduction}

Complex forms of information, maintained in different formats and organized
according to different structures, often need to be shared between agents.  In
recent years, both in the data management and in the knowledge representation
communities, several settings have been investigated that address this problem
from various perspectives: in \emph{information integration}, uniform access is
provided to a collection of data sources by means of an ontology (or global
schema) to which the sources are mapped \cite{Lenz02}; in \emph{peer-to-peer
  systems}, a set of peers declaratively linked to each other collectively
provide access to the information assets they maintain
\cite{KeAM03,ACGRS06,FKMT06}; in \emph{ontology matching}, the aim is to
understand and derive the correspondences between elements in two ontologies
\cite{EuSh07,SE13}; finally, in \emph{data exchange}, the information stored
according to a source schema needs to be restructured and translated so as to
conform to a target schema \cite{FKMP05,Barc09}.

The work we present in this paper is inspired by the latter setting,
investigated in databases.  We study it, however, under the assumption of
incomplete information typical of knowledge representation \cite{ArPR11}.
Specifically, we investigate the problem of \emph{knowledge base exchange},
where a source knowledge base (KB) is connected to a target KB by means of a
declarative mapping specification, and the aim is to exchange knowledge from
the source to the target by exploiting the mapping.  We rely on a framework for
KB exchange based on lightweight Description Logics (DLs) of the \dllite family
\cite{CDLLR07}, recently proposed in \cite{ABCRS12,ABCRS12b}: both source and
target are KBs constituted by a DL TBox, representing implicit information, and
an ABox, representing explicit information, and mappings are sets of DL concept
and role inclusions. Note that in data and knowledge base exchange, differently
from ontology matching, mappings are first-class citizens. In fact, it has been
recognized that building schema mappings is an important and complex activity,
which requires the designer to have a thorough understanding of the source and
how the information therein should be related to the target.  Thus, several
techniques and tools have been developed to support mapping design, e.g.,
exploiting lexical information \cite{FHHM*09}. Here, similar to data exchange,
we assume that for building mappings the target signature is given, but no
further axioms constraining the target knowledge are available.  In fact, such
axioms are derived from the source KB and the mapping.

We consider two key problems:
\begin{inparaenum}[\it (i)]
\item computing \emph{universal solutions}, which have been identified as one
  of the most desirable translations to be materialized;
\item \emph{\UCQ-representability} of a source TBox by means of a target TBox
  that captures at best the intensional
  information that can be extracted from the source according to a mapping
  using union of conjunctive queries.  Determining \UCQ-representability is a
  crucial task, since it allows one to use the obtained target TBox to infer
  new knowledge in the target, thus reducing the amount of extensional
  information to be transferred from the source.
  Moreover,
  it has been noticed that in many data exchange applications users only
  extract information from the translated data by using specific queries
  (usually conjunctive queries), so query-based notions of translation
  specifically tailored to store enough information to answer such queries have
  been widely studied in the data exchange
  area~\cite{MH03,FKNP08,APRR09,FK12,PSS13}.
\end{inparaenum}
For these two problems, we investigate both the task of checking
\emph{membership}, where a candidate universal solution (resp.,
\UCQ-representation) is given and one needs to check its correctness, and
\emph{non-emptiness}, where the aim is to determine the existence of a
universal solution (resp., \UCQ-representation).

We significantly extend previous results in several directions.
First of all, we establish results for \owlql \cite{OWL2QL}, one of the profiles
of the standard Web Ontology Language~OWL~2 \cite{OWL2}, which is based on the
DL \dlliter.  To do so, we have to overcome the difficulty of dealing with null
values in the ABox, since these become necessary in the target to represent
universal solutions. Also, for the first time, we address disjointness
assertions in the TBox, a construct that is part of \owlql.
The main contribution of our work is then a detailed analysis of the
computational complexity of both membership and non-emptiness for universal
solutions and \UCQ-representability.
For the non-emptiness problem of universal solutions, previous known results
covered only the simple case of $\dlliterdfs$, the RDFS fragment of \owlql, in
which no new facts can be inferred, and universal solutions always exist and can
be computed in polynomial time via a chase procedure (see \cite{CDLLR07}).  We
show that in our case, instead, the problem is \PSPACE-hard, hence significantly
more complex, and provide an \EXPTIME upper bound based on a novel approach
exploiting two-way alternating automata.  We provide also NP upper bounds for
the simpler case of ABoxes without null values, and for the case of the
membership problem.
As for \UCQ-representability, we adopt the notion of
\emph{\UCQ-representability} introduced in \cite{ABCRS12,ABCRS12b} and extend it
to take into account disjointness of \owlql.  For that case we show
\NLOGSPACE-completeness of both non-emptiness and membership, improving on the
previously known \PTIME upper bounds.

The paper is organized as follows. In Section~\ref{sec-preliminaries}, we give
preliminary notions on DLs and queries. In Section~\ref{sec-kbe}, we define our
framework of KB exchange and discuss the problem of computing solutions.  In
Section~\ref{sec-cont}, we overview our contributions, and then we provide our
results on computing universal solutions in Section~\ref{sec-univ}, and on
UCQ-representability in Section~\ref{sec-ucq-rep}.  Finally, in
Section~\ref{sec-conclusions}, we draw some conclusions and outline issues for
future work.

\ifextendedversion
\else
The proofs are available in an extended technical report accessible at
\url{http://arxiv.org/abs/1304.5810}. 
\fi

\section{Preliminaries}
\label{sec-preliminaries}

The DLs of the \dllite family~\cite{CDLLR07} of light-weight DLs are
characterized by the fact that standard reasoning can be done in polynomial
time.
We adapt here \dlliter, the DL underlying \owlql, and present now its
syntax and semantics.  Let $N_C$, $N_R$, $N_a$, $N_\ell$ be pairwise
disjoint sets of \emph{concept names}, \emph{role names}, \emph{constants}, and
\emph{labeled nulls}, respectively.  Assume in the following that $A\in N_C$ and
$P\in N_R$; in \dlliter, $B$ and $C$ are used to denote basic and arbitrary (or
complex) concepts, respectively, and $R$ and $Q$ are used to denote basic and
arbitrary (or complex) roles, respectively, defined as follows:
\vspace{-0.01cm}
\[
\begin{array}{r@{~~}c@{~~}l}
  R & ::= & P ~\mid~ P^- \\
  Q & ::= & R ~\mid~ \NOT R
\end{array}
\qquad\qquad
\begin{array}{r@{~~}c@{~~}l}
  B & ::= & A ~\mid~ \SOMET{R} \\
  C & ::= & B ~\mid~ \NOT B
\end{array}
\]
From now on, for a basic role $R$, we use $R^-$ to denote $P^-$ when $R=P$, and
$P$ when $R=P^-$.

A TBox is a finite set of \emph{concept inclusions} $B \ISA C$ and \emph{role
  inclusions} $R\ISA Q$.  We call an inclusion of the form $B_1 \ISA \neg B_2$
or $R_1 \ISA \neg R_2$ a \emph{disjointness assertion}.  An ABox is a finite set
of \emph{membership assertions} $B(a)$, $R(a,b)$, where $a, b \in N_a$.  In this
paper, we also consider extended ABoxes, which are obtained by allowing labeled
nulls in membership assertions. Formally, an \emph{extended ABox} is a finite
set of membership assertions $B(u)$ and $R(u,v)$, where $u, v \in (N_a \cup
N_\ell)$. Moreover, a(n \emph{extended}) \emph{KB} $\K$ is a pair $\tup{\T,\A}$,
where $\T$ is a TBox and $\A$ is an (extended) ABox.

A \emph{signature} $\Sigma$ is a finite set of concept and role names.
A KB $\K$ is said to be \emph{defined over} (or simply, \emph{over}) $\Sigma$ if
all the concept and role names occurring in $\K$ belong to $\Sigma$
(and likewise for TBoxes, ABoxes, concept inclusions, role inclusions
and membership assertions).
Moreover, an \emph{interpretation} $\I$ \emph{of} $\Sigma$ is a pair
$\tup{\dom,\Int{\cdot}}$, where $\dom$ is a non-empty domain and $\Int{\cdot}$
is an interpretation function such that:
\begin{inparaenum}[(1)]
\item $\Int{A} \subseteq \dom$, for every concept name $A \in \Sigma$;
\item $\Int{P} \subseteq \dom \times \dom$, for every role name $P \in \Sigma$;
  and
\item \label{it:int-obj} $\Int{a} \in \dom$, for every constant $a \in N_a$.
\end{inparaenum}
Function $\Int{\cdot}$ is extended to also interpret concept and role
constructs:
\[
  \begin{array}{r@{~~}c@{~~}l}
    \Int{(\SOMET{R})} &=& \{x \in \dom \st \exists y \in \dom\text{ such that
    }(x,y) \in \Int{R}\};\\
    \Int{(P^-)} &=& \{(y,x) \in \dom \times \dom \st (x,y) \in \Int{P}\};\\
    \Int{(\NOT B)} &=& \dom \setminus \Int{B}; \qquad
    \Int{(\NOT R)}~=~(\dom \times \dom) \setminus \Int{R}.
  \end{array}
\]
Note that, consistently with the semantics of \owlql, we do \emph{not} make the
unique name assumption (UNA), i.e., we allow distinct constants $a,b\in
N_a$ to be interpreted as the same object, i.e., $\Int{a}=\Int{b}$.
Note also that labeled nulls are \emph{not} interpreted by $\I$.

Let $\I = \tup{\dom,\Int{\cdot}}$ be an interpretation over a signature
$\Sigma$.  Then $\I$ is said to satisfy a concept inclusion $B \ISA C$ over
$\Sigma$, denoted by $\I \models B \ISA C$, if $\Int{B} \subseteq \Int{C}$; $\I$
is said to satisfy a role inclusion $R \ISA Q$ over $\Sigma$, denoted by $\I
\models R \ISA Q$, if $\Int{R} \subseteq \Int{Q}$; and $\I$ is said to satisfy a
TBox $\T$ over $\Sigma$, denoted by $\I \models \T$, if $\I \models \alpha$ for
every $\alpha \in \T$.  Moreover, satisfaction of membership assertions over
$\Sigma$ is defined as follows. A \emph{substitution} over $\I$ is a function $h
: (N_a \cup N_\ell) \to \Int{\Delta}$ such that $h(a) = \Int{a}$ for every $a
\in N_a$.  Then $\I$ is said to satisfy an (extended) ABox $\A$, denoted by $\I
\models \A$, if there exists a substitution $h$ over $\I$ such that: %
\begin{itemize}
\item[--] for every $B(u) \in \A$, it holds that $h(u) \in \Int{B}$; and
\item[--] for every $R(u,v) \in \A$, it holds that $(h(u),h(v)) \in \Int{R}$.
\end{itemize}
Finally, $\I$ is said to
\emph{satisfy} a(n extended) KB $\kb$, denoted by $\I \models \K$, if $\I
\models \T$ and $\I \models \A$.  Such $\I$ is called a \emph{model} of $\K$,
and we use $\Mod(\K)$ to denote the set of all models of $\K$.  We say that $\K$
is \emph{consistent} if $\Mod(\K) \neq \emptyset$.

As is customary, given an (extended) KB $\K$ over a signature
$\Sigma$ and a membership assertion or an inclusion $\alpha$ over $\Sigma$, we
use notation $\K \models \alpha$ to indicate that for every interpretation $\I$
of $\Sigma$, if $\I\models \K$, then $\I \models \alpha$.

\subsection{Queries and certain answers}

A $k$-ary query $q$ over a signature $\Sigma$, with $k \geq 0$, is a function
that maps every interpretation $\tup{\dom, \Int{\cdot}}$ of $\Sigma$ into a
$k$-ary relation $\Int{q} \subseteq (\dom)^k$.  In particular, if $k = 0$, then
$q$ is said to be a Boolean query, and $\Int{q}$ is either a relation containing
the empty tuple $()$ (representing the value true) or the empty relation
(representing the value false).  Given a KB $\K$ over $\Sigma$, the set of
\emph{certain answers} to $q$ over $\K$, denoted by $\cert(q, \K)$, is defined
as:
\[
\textstyle \bigcap_{\I\in\Mod(\K)} \{
\begin{array}[t]{@{}l}
  (a_1, \ldots, a_k) \ \mid\\
  \{a_1, \ldots, a_k\} \subseteq N_a \text{ and } (a_1^\I, \ldots, a_k^\I) \in
  q^\I
  \},
\end{array}
\]
Notice that the certain answer to a query does \emph{not} contain labeled nulls.
Besides, notice that if $q$ is a Boolean query, then $\cert(q, \K)$ evaluates to
true if $\Int{q}$ evaluates to true for every $\I \in \Mod(\K)$, and it
evaluates to false otherwise.

A \emph{conjunctive query $(\CQ)$ over a signature $\Sigma$} is a formula of the
form $q(\vec x) = \exists \vec y.\, \varphi(\vec x, \vec y)$, where $\vec x$,
$\vec y$ are tuples of variables and $\varphi(\vec x, \vec y)$ is a conjunction
of atoms of the form $A(t)$, with $A$ a concept name in $\Sigma$, and $P(t,t')$,
with $P$ a role name in $\Sigma$, where each of $t,t'$ is either a constant from
$N_a$ or a variable from $\vec x$ or $\vec y$.
Given an interpretation $\I = \tup{\dom, \Int{\cdot}}$ of $\Sigma$, the answer
of $q$ over $\I$, denoted by $\Int{q}$, is the set of tuples $\vec a$ of
elements from $\dom$ for which there exist a tuple $\vec b$ of elements from
$\dom$ such that $\I$ satisfies every conjunct in $\varphi(\vec a, \vec b)$. A
union of conjunctive queries ($\UCQ$) over a signature $\Sigma$ is a
formula of the form $q(\vec{x}) = \textstyle \bigvee_{i=1}^n q_i(\vec{x})$,
where each $q_i$ ($1 \leq i \leq n$) is a \CQ over $\Sigma$, whose
semantics is defined as $\Int{q} = \bigcup_{i=1}^n \Int{q_i}$.

\section{Exchanging \owlql Knowledge Bases}
\label{sec-kbe}

We generalize now, in Section~\ref{sec-sol}, the setting proposed in
\cite{ArPR11} to \owlql, and we formalize in Section~\ref{sec-problems} the
main problems studied in the rest of the paper.

\subsection{A knowledge base exchange framework for \owlql}
\label{sec-sol}
Assume that $\Sigma_1$, $\Sigma_2$ are signatures with no concepts or roles in
common. An inclusion $E_1 \ISA E_2$ is said to be \emph{from} $\Sigma_1$
\emph{to} $\Sigma_2$, if $E_1$ is a concept or a role over $\Sigma_1$ and $E_2$
is a concept or a role over $\Sigma_2$.  A mapping is a tuple $\M = (\Sigma_1,
\Sigma_2, \T_{12})$, where $\T_{12}$ is a TBox consisting of inclusions from
$\Sigma_1$ to $\Sigma_2$~\cite{ABCRS12}. Recall that in this paper, we deal with
$\dlliter$ TBoxes only, so $\T_{12}$ is assumed to be a set of $\dlliter$
concept and role inclusions.  The semantics of such a mapping is defined in
\cite{ABCRS12} in terms of a notion of satisfaction for interpretations, which
has to be extended in our case to deal with interpretations not satisfying the
UNA (and, more generally, the standard name assumption).
More specifically, given interpretations $\I$,
$\J$ of $\Sigma_1$ and $\Sigma_2$, respectively, pair $(\I,\J)$ \emph{satisfies}
TBox $\T_{12}$, denoted by $(\I, \J) \models \T_{12}$, if
\begin{inparaenum}[(\it i)]
\item for every $a \in N_a$, it holds that $a^\I = a^\J$,
\item for every concept inclusion $B \ISA C \in \T_{12}$, it holds that
  $\Int[\I]{B} \subseteq \Int[\J]{C}$, and
\item for every role inclusion $R \ISA Q \in \T_{12}$, it holds that
  $\Int[\I]{R} \subseteq \Int[\J]{Q}$.
\end{inparaenum}
Notice that the connection between the information in $\I$ and $\J$ is
established through the constants that move from source to target according to
the mapping.  For this reason, we require constants to be interpreted in the
same way in $\I$ and $\J$, i.e., they preserve their meaning when they are
transferred.
Besides, notice that this is the only restriction imposed on the domains of
$\I$ and $\J$ (in particular, we require neither that $\Delta^\I = \Delta^\J$
nor that $\Delta^\I \subseteq \Delta^\J$).
Finally, $\Sat_\M(\I)$ is defined as the set of
interpretations $\J$ of $\Sigma_2$ such that $(\I,\J) \models \T_{12}$, and
given a set ${\cal X}$ of interpretations of $\Sigma_1$, $\Sat_\M({\cal X})$ is
defined as $\bigcup_{\I\in \mathcal{X}} \Sat_\M(\I)$.

The main problem studied in the knowledge exchange area is the problem of
translating a KB according to a mapping, which is formalized
through several different notions of translation (for a thorough comparison of
different notions of solutions see \cite{ABCRS12}).
 The first such notion is the
concept of solution, which is formalized as follows. Given a mapping $\M =
(\Sigma_1, \Sigma_2, \T_{12})$ and KBs $\K_1$, $\K_2$ over $\Sigma_1$ and
$\Sigma_2$, respectively, $\K_2$ is a \emph{solution} for $\K_1$ under $\M$ if
$\Mod(\K_2) \subseteq \Sat_\M(\Mod(\K_1))$. Thus, $\K_2$ is a solution for
$\K_1$ under $\M$ if every interpretation of $\K_2$ is a valid translation of an
interpretation of $\K_1$ according to $\M$. Although natural, this is a mild
restriction, which gives rise to the stronger notion of universal
solution. Given $\M$, $\K_1$ and $\K_2$ as before,
$\K_2$ is a \emph{universal solution} for $\K_1$ under $\M$ if
$\Mod(\K_2) = \Sat_\M(\Mod(\K_1))$. Thus, $\K_2$ is designed to
exactly represent the space of interpretations obtained by translating
the interpretations of $\K_1$ under $\M$~\cite{ABCRS12}.
Below is a simple example demonstrating the notion of universal solutions. This
example also illustrates some issues regarding the absence of 
the UNA, which has to be given up to comply with the \owlql standard, and
regarding the use of disjointness assertions.

\begin{example}\label{exa-disj-source}
  Assume $\M = (\{\concept{F}, \concept{G}\}, \{\concept{F'},
  \concept{G'}\}$, $\T_{12})$, where $\T_{12}=\{F \ISA F', G \ISA G'\}$, and let
  $\K_1 = \tup{\T_1, \A_1}$, where $\T_1 = \{\}$ and $\A_1 = \{F(a),
  G(b)\}$. Then the ABox $\A_2 = \{F'(a), G'(b)\}$ is a universal solution for
  $\K_1$ under $\M$.

  Now, if we add a seemingly harmless disjointness assertion $\{F \ISA \NOT G\}$
  to $\T_1$, we obtain that $\A_2$ is no longer a universal solution (not even a
  solution) for $\K_1$ under $\M$. The reason for that is the lack of the UNA on
  the one hand, and the presence of the disjointness assertion in $\T_1$ on the
  other hand. In fact, the latter forces $a$ and $b$ to be interpreted
  differently in the source. Thus, for a model $\J$ of $\A_2$ such that
  $\Int[\J]{a} = \Int[\J]{b}$ and $\Int[\J]{F'} = \Int[\J]{G'} =
  \{\Int[\J]{a}\}$, there exists no model $\I$ of $\K_1$ such that $(\I,\J)
  \models \T_{12}$ (which would require $\Int{a} = \Int[\J]{a}$ and $\Int{b} =
  \Int[\J]{b}$). In general, there exists no universal solution for $\K_1$ under
  $\M$, even though $\K_1$ and $\T_{12}$ are consistent with each other.
\end{example}

A second class of translations is obtained in \cite{ABCRS12} by observing that
solutions and universal solutions are too restrictive for some applications, in
particular when one only needs a translation storing enough information to
properly answer some queries. For the particular case of $\UCQ$, this gives rise
to the notions of $\UCQ$-solution and universal $\UCQ$-solution.  Given a
mapping $\M = (\Sigma_1, \Sigma_2, \T_{12})$, a KB $\K_1 = \tup{\T_1,\A_1}$ over
$\Sigma_1$ and a KB $\K_2$ over $\Sigma_2$, $\K_2$ is a \emph{$\UCQ$-solution}
for $\K_1$ under $\M$ if for every query $q \in \UCQ$ over $\Sigma_2$: $\cert(q,
\tup{\T_1 \cup \T_{12}, \A_1}) \subseteq \cert(q, \K_2)$, while $\K_2$ is a
\emph{universal $\UCQ$-solution} for $\K_1$ under $\M$ if for every query $q \in
\UCQ$ over $\Sigma_2$: $\cert(q, \tup{\T_1 \cup \T_{12}, \A_1}) = \cert(q,
\K_2)$.

Finally, a last class of solutions is obtained in \cite{ABCRS12} by considering
that users want to translate as much of the knowledge in a TBox as possible, as
a lot of effort is put in practice when constructing a TBox. This observation
gives rise to the notion of $\UCQ$-representation~\cite{ABCRS12}, which
formalizes the idea of translating a source TBox according to a mapping.  Next,
we present an alternative formalization of this notion, which is appropriate for
our setting where disjointness assertions are considered.\footnote{If
  disjointness assertions are not allowed, then this new notion can be shown to
  be equivalent to the original formalization of $\UCQ$-representation proposed
  in \cite{ABCRS12}.} Assume that $\M = (\Sigma_1, \Sigma_2, \T_{12})$ and
$\T_1$, $\T_2$ are TBoxes over $\Sigma_1$ and $\Sigma_2$, respectively.  Then
$\T_2$ is a {\em $\UCQ$-representation} of $\T_1$ under $\M$ if for every query
$q \in \UCQ$ over $\Sigma_2$ and every ABox $\A_1$ over $\Sigma_1$ that is
consistent with $\T_1$:
\begin{multline}\label{eq-ucq-rep}
  \tag{\dag}
  \cert(q, \tup{\T_1 \cup \T_{12}, \A_1}) \ = \\
  \mathop{\bigcap_{\A_2 \,:\, \A_2 \text{ is an ABox over } \Sigma_2 \text{
        that}}}_{\text{is a $\UCQ$-solution for } \A_1 \text{ under } \M}
  \cert(q, \tup{\T_2,\A_2}).
\end{multline}
Notice that in the previous definition, $\A_2$ is said to be a $\UCQ$-solution
for $\A_1$ under $\M$ if the KB $\tup{\emptyset,\A_2}$ is a $\UCQ$-solution for
the KB $\tup{\emptyset,\A_1}$ under $\M$. Let us explain the intuition behind
the definition of the notion of $\UCQ$-representation.  Assume that $\T_1$,
$\T_2$, $\M$ satisfy \eqref{eq-ucq-rep}. First, $\T_2$ captures the information
in $\T_1$ that is translated by $\M$ and that can be extracted by using a \UCQ,
as for every ABox $\A_1$ over $\Sigma_1$ that is consistent with $\T_1$ and
every \UCQ $q$ over $\Sigma_2$, if we choose an arbitrary $\UCQ$-solution $\A_2$
for $\A_1$ under $\M$, then it holds that $\cert(q, \tup{\T_1 \cup \T_{12},
  \A_1}) \subseteq \cert(q, \tup{\T_2,\A_2})$. Notice that $\A_1$ is required to
be consistent with $\T_1$ in the previous condition, as we are interested in
translating data that make sense according to $\T_1$. Second, $\T_2$ does not
include any piece of information that can be extracted by using a \UCQ and it is
not the result of translating the information in $\T_1$ according to $\M$. In
fact, if $\A_1$ is an ABox over $\Sigma_1$ that is consistent with $\T_1$ and
$q$ is a \UCQ over $\Sigma_2$, then it could be the case that $\cert(q,
\tup{\T_1 \cup \T_{12}, \A_1}) \subsetneq \cert(q, \tup{\T_2,\A^\star_2})$ for
some $\UCQ$-solution $\A^\star_2$ for $\A_1$ under $\M$. However, the extra
tuples extracted by query $q$ are obtained from the extra information in
$\A^\star_2$, as if we consider a tuple $\vec{a}$ that belong to $\cert(q,
\tup{\T_2,\A_2})$ for every $\UCQ$-solution $\A_2$ for $\A_1$ under $\M$, then
it holds that $\vec{a} \in \cert(q, \tup{\T_1 \cup \T_{12}, \A_1})$.

%


\begin{example}\label{exa-ucq-rep}
  Assume that $\M = (\{\concept{F}, \concept{G}, \concept{H}, \concept{D}\}$,
  $\{\concept{F'}, \concept{G'}, \concept{H'}\}$, $\T_{12})$, where $\T_{12} =
  \{F \ISA F', G \ISA G', H \ISA H'\}$, and let $\T_1 = \{F \ISA G\}$. As
  expected, 
  TBox $\T_2 = \{ F' \ISA G'\}$ is a $\UCQ$-representation of $\T_1$ under
  $\M$. Moreover, we can add the inclusion $D \ISA \NOT H'$ to $\T_{12}$, and
  $\T_2$ will still remain a \UCQ-representation of $\T_1$ under $\M$.  Notice
  that in this latter setting, our definition has to deal with some ABoxes
  $\A_1$ that are consistent with $\T_1$ but not with $\T_1 \cup \T_{12}$, for
  instance $\A_1 = \{H(a), D(a)\}$ for some constant $a$. In those cases,
  Equation~\eqref{eq-ucq-rep} is trivially satisfied, since $\Mod(\tup{\T_1 \cup
    \T_{12}, \A_1}) = \emptyset$ and the set of $\UCQ$-solutions for $\A_1$
  under $\M$ is empty.
\end{example}

\subsection{On the problem of computing solutions}
\label{sec-problems}

Arguably, the most important problem in knowledge
exchange~\cite{ArPR11,ABCRS12}, as well as in data exchange~\cite{FKMP05,K05},
is the task of computing a translation of a KB according to a mapping. To study
the computational complexity of this task for the different notions of solutions
presented in the previous section, we introduce the following decision problems.
The \emph{membership} problem for universal solutions (resp. universal
\UCQ-solutions) has as input a mapping $\M = (\Sigma_1, \Sigma_2, \T_{12})$ and
KBs $\K_1$, $\K_2$ over $\Sigma_1$ and $\Sigma_2$, respectively. Then the
question to answer is whether $\K_2$ is a universal solution (resp. universal
\UCQ-solution) for $\K_1$ under $\M$.
Moreover, the membership problem for $\UCQ$-representations has as input a
mapping $\M = (\Sigma_1, \Sigma_2, \T_{12})$ and TBoxes $\T_1$, $\T_2$ over
$\Sigma_1$ and $\Sigma_2$, respectively, and the question to answer is whether
$\T_2$ is a $\UCQ$-representation of $\T_1$ under $\M$.

In our study, we cannot leave aside the existential versions of the previous
problems, which are directly related with the problem of computing translations
of a KB according to a mapping. Formally,
the \emph{non-emptiness} problem for universal solutions (resp. universal
$\UCQ$-solutions) has as input a mapping $\M = (\Sigma_1, \Sigma_2, \T_{12})$
and a KB $\K_1$ over $\Sigma_1$.  Then the question to answer is whether there
exists a universal solution (resp. universal \UCQ-solution) for $\K_1$ under
$\M$.
Moreover, the non-emptiness problem for $\UCQ$-representations has as input a
mapping $\M = (\Sigma_1, \Sigma_2, \T_{12})$ and a TBox $\T_1$ over $\Sigma_1$,
and the question to answer is
whether there exists a $\UCQ$-representation of $\T_1$ under $\M$.

\begin{figure*}[ht]
  \begin{center}
    \begin{tabular}{cc}
      \begin{tabular}{l|c|c|}\cline{2-3}
        {\bf Membership} &  ABoxes & extended ABoxes\\\hline
        \multicolumn{1}{|l|}{Universal solutions} &
        in \NP &
        \NP-complete \\\hline
        \multicolumn{1}{|l|}{$\UCQ$-representations} &
        \multicolumn{2}{c|}{\NLOGSPACE-complete}\\\hline
      \end{tabular} &
      \begin{tabular}{l|c|c|}\cline{2-3}
        {\bf Non-emptiness} & ABoxes & extended ABoxes\\\hline
        \multicolumn{1}{|l|}{Universal solutions} &
        in \NP & \PSPACE-hard, in \EXPTIME \\\hline
        \multicolumn{1}{|l|}{$\UCQ$-representations} &
        \multicolumn{2}{c|}{\NLOGSPACE-complete}\\\hline
      \end{tabular}
    \end{tabular}
  \end{center}

  \caption{Complexity results obtained in the paper about the membership and
    non-emptiness problems.\label{fig-results}}
\end{figure*}

\section{Our contributions}
\label{sec-cont}

In Section~\ref{sec-problems}, we have introduced the 
problems that
are studied in this paper. It is important to notice that these problems are
defined by considering only KBs (as opposed to extended KBs), as they are the
formal counterpart of \owlql. Nevertheless, as shown in Section~\ref{sec-univ},
there are natural examples of \owlql specifications and mappings where null
values are needed when constructing solutions. Thus, we also study the problems
defined in Section~\ref{sec-problems} in the case where translations can be
extended KBs. It should be noticed that the notions of solution, universal
solution, $\UCQ$-solution, universal $\UCQ$-solution, and $\UCQ$-representation
have to be enlarged to consider extended KBs, which is straightforward to do.
In particular, given a mapping $\M = (\Sigma_1, \Sigma_2, \T_{12})$ and TBoxes
$\T_1$, $\T_2$ over $\Sigma_1$ and $\Sigma_2$, respectively, $\T_2$ is said to
be a $\UCQ$-representation of $\T_1$ under $\M$ in this extended setting if in
Equation \eqref{eq-ucq-rep}, $\A_2$ is an extended ABox over $\Sigma_2$ that is
a $\UCQ$-solution for $\A_1$ under~$\M$.

The main contribution of this paper is to provide a detailed analysis of the
complexity of the membership and non-emptiness problems for the notions of
universal solution and $\UCQ$-representation. In Figure~\ref{fig-results}, we
provide a summary of the main results in the paper, which are explained in more
detail in Sections~\ref{sec-univ} and~\ref{sec-ucq-rep}. It is important to
notice that these results considerably extend the previous known results about
these problems \cite{ABCRS12,ABCRS12b}. In the first place, the problem
of computing universal solutions was studied in \cite{ABCRS12} for the case of
$\dlliterdfs$, a fragment of $\dlliter$ that allows neither for inclusions of
the form $B \ISA \SOMET{R}$ nor for disjointness assertions. In that case, it is
straightforward to show that every source KB has a universal solution that can
be computed by using the chase procedure~\cite{CDLLR07}. Unfortunately, this
result does not provide any information about how to solve the much larger case
considered in this paper, where, in particular, the non-emptiness problem is not
trivial. In fact, for the case of the notion of universal solution, all the
lower and upper bounds provided in Figure~\ref{fig-results} are new results,
which are not consequences of the results obtained in \cite{ABCRS12}. In the
second place, 
a notion of \UCQ-representation that is appropriate for the fragment of \dlliter
not including disjointness assertions was studied in \cite{ABCRS12,ABCRS12b}. %
In particular, it was shown 
that the membership and non-emptiness problems for this notion are solvable in
polynomial time. In this paper, we considerably strengthen these results:
\begin{inparaenum}[(\it i)]
\item by generalizing the definition of the notion of $\UCQ$-representation to
  be able to deal with \owlql, that is, with the entire language $\dlliter$
  (which includes disjointness assertions); and
\item by showing that the membership and non-emptiness problems are both
  $\NLOGSPACE$-complete in this larger scenario.
\end{inparaenum}

It turns out that reasoning about universal $\UCQ$-solutions is much more
intricate. In fact, as a second contribution of our paper, we provide a
$\PSPACE$ lower bound for the complexity of the membership problem for the
notion of universal $\UCQ$-solution, which is in sharp contrast with the $\NP$
and $\NLOGSPACE$ upper bounds for this problem for the case of universal
solutions and $\UCQ$-representations, respectively (see
Figure~\ref{fig-results}). Although many questions about universal
$\UCQ$-solutions remain open, we think that this is an interesting first result,
as universal $\UCQ$-solutions have only been investigated before for the very
restricted fragment $\dlliterdfs$ of $\dlliter$~\cite{ABCRS12}, which is
described in the previous paragraph.

\section{Computing universal solutions}
\label{sec-univ}


In this section, we study the membership and non-emptiness problems for
universal solutions, in the cases where nulls are not allowed (Section
\ref{sec-univ-owlql}) and are allowed (Section \ref{sec-non-emp-uni-sol}) in
such solutions. But before going into this, we give an example that shows the
shape of universal solutions in \dlliter.
%
%

\begin{example}\label{exa-null}
  Assume that $\M = (\{\concept{F}, \role{S}\}, \{\concept{G'}\}$, $\{\SOMET{S^-}
  \ISA G'\})$, and let $\K_1 = \tup{\T_1, \A_1}$, where $\T_1 = \{F \ISA
  \SOMET{S}\}$ and $\A_1 = \{F(a)\}$. Then a natural way to construct a
  universal solution for $\K_1$ under $\M$ is to `populate' the target with all
  implied facts (as it is usually done in data exchange). Thus, the ABox $\A_2 =
  \{G'(n)\}$, where $n$ is a labeled null, is a universal solution for $\K_1$
  under $\M$ if nulls are allowed.
  Notice that here, a universal solution with non-extended ABoxes does not
  exist: substituting $n$ by any constant is too restrictive, ruining
  universality.
  %
\end{example}
\begin{example}
  Now, assume $\M = (\{\concept{F}, \role{S}, \role{T}\}$, $\{\role{S'}\}$,
  $\{S \ISA S', T \ISA S'\})$, and $\K_1 = \tup{\T_1, \A_1}$, where $\T_1 =
  \{F \ISA \SOMET{S}, \SOMET{S^-} \ISA \SOMET{S}\}$ and $\A_1 = \{F(a)$,
  $T(a,a)\}$. In this case, we cannot use the same approach as in Example
  \ref{exa-null}
  to construct a universal solution, as now we would need of an infinite number
  of labeled nulls to construct such a solution. However, as $S$ and $T$ are
  transferred to the same role $S'$, it is possible to use constant $a$ to
  represent all implied facts. In particular, in this case $\A_2 = \{S'(a,a)\}$
  is a universal solution for $\K_1$ under $\M$.
  %
  %
\end{example}

\subsection{Universal solutions without null values}
\label{sec-univ-owlql}
We explain here how the $\NP$ upper bound for the non-emptiness problem for
universal solutions is obtained,
when ABoxes are not allowed to contain null values. 

Assume given a mapping $\M = (\Sigma_1$, $\Sigma_2$, $\T_{12})$ and a KB $\K_1 =
\tup{\T_1, \A_1}$ over $\Sigma_1$. To check whether $\K_1$ has a universal
solution under $\M$, we use the following non-deterministic polynomial-time
algorithm. First, we construct an ABox $\A_2$ over $\Sigma_2$ containing every
membership assertion $\alpha$ such that $\tup{\T_1 \cup \T_{12}, \A_1} \models
\alpha$, where $\alpha$ is of the form either $B(a)$ or $R(a,b)$, and
$a,b$ are constants mentioned in $\A_1$. Second, we guess an interpretation $\I$
of $\Sigma_1$ such that $\I \models \K_1$ and $(\I, \Uni_{\A_2}) \models
\T_{12}$, where $\Uni_{\A_2}$ is the interpretation of $\Sigma_2$ naturally
corresponding\footnote{Interpretation $\Uni_{\A_2}$ can be defined as the
  Herbrand model of $\A_2$ extended with fresh domain elements to satisfy
  assertions of the form $\SOMET{R}(a)$ in $\A_2$.} to $\A_2$. The correctness
of the algorithm is a consequence of the facts that:
\begin{itemize}
\item[a)] there exists a universal solution for $\A_1$ under $\M$ if and only if
  $\A_2$ is a solution for $\A_1$ under $\M$; and
\item[b)] $\A_2$ is a solution for $\A_1$ under $\M$ if and only if there exists a
  model $\I$ of $\K_1$ such that $(\I, \Uni_{\A_2}) \models \T_{12}$.
\end{itemize}
Moreover, the algorithm can be implemented in a non-deterministic
polynomial-time Turing machine given that:
\begin{inparaenum}[(\it i)]
\item $\A_2$ can be constructed in polynomial time;
\item if there exists a model $\I$ of $\K_1$ such that $(\I, \Uni_{\A_2})
  \models \T_{12}$, then there exists a model of $\K_1$ of polynomial-size
  satisfying this condition; and
\item it can be checked in polynomial time whether $\I \models \K_1$ and $(\I,
  \Uni_{\A_2}) \models \T_{12}$.
\end{inparaenum}

In addition, in this case, the membership problem can be reduced to the
non-emptiness problem, thus, we have that:
\begin{theorem}
  \label{the:non-emp-uni-sol-owlql-aboxes-np}
  The non-emptiness and membership problems for universal solutions 
  are in \NP.
\end{theorem}
The exact complexity of these problems remains open. In fact, we conjecture that
these problems are in \PTIME.

%
We conclude by showing that reasoning about universal $\UCQ$-solutions is harder
than reasoning about universal solutions, which can be explained by the fact
that TBoxes have bigger impact on the structure of universal \UCQ-solutions
rather than of universal solutions. In fact, by using a reduction from the
validity problem for quantified Boolean formulas, similar to a reduction in
\cite{KKLSWZ11}, we are able to prove the following:
\begin{theorem}
  \label{the:memb-uni-UCQ-sol-pspace-hard}
  The membership problem for universal \UCQ-solutions is \PSPACE-hard.
\end{theorem}

\subsection{Universal solutions with null values}
\label{sec-non-emp-uni-sol}
We start by considering the non-emptiness problem for universal solutions with
null values, that is, when extended ABoxes are allowed in universal
solutions. As our first result, similar to the reduction above, we show that
this problem is \PSPACE-hard, and identify 
the inclusion of inverse roles as one of the main sources of complexity.
%

To obtain an upper bound for this problem, we use
\emph{two-way alternating automata on infinite trees (2ATA)}, which are a
generalization of nondeterministic automata on infinite trees~\cite{Vard98}
well suited for handling inverse roles in \dlliter.
%
More precisely, given a KB $\K$, we first show that it is possible to construct
the following automata:
\begin{itemize}
\item[--] $\acan$ is a 2ATA that accepts trees corresponding to the canonical model
  of $\K$~\footnote{If $\K = \tup{\T,\A}$, then this model essentially
    corresponds to the chase of $\A$ with $\T$ (see~ \cite{KKLSWZ11} for a
    formal definition).} with nodes arbitrary labeled with a special symbol $G$;
\item[--] $\asol$ is a 2ATA that accepts a tree if its subtree labeled with $G$
  corresponds to a tree model $\I$ of $\K$ (that is, a model forming a tree on
  the labeled nulls); and
\item[--] $\agfin$ is a (one-way) non-deterministic automaton that accepts a tree if
  it has a finite prefix where each node is marked with $G$, and no other node
  in the tree is marked with $G$.
\end{itemize}
Then to verify whether a KB $\K_1 = \tup{\T_1, \A_1}$ has a universal solution
under a mapping $\M=(\Sigma_1, \Sigma_2, \T_{12})$, we solve the non-emptiness
problem for an automaton $\Bau$ defined as the product automaton of
$\pi_{\Gamma_{\K}}(\acan)$, $\pi_{\Gamma_\K}(\asol)$ and $\agfin$, where
$\K=\tup{\T_1 \cup \T_{12}, \A_1}$, 
$\pi_{\Gamma_\K}(\acan)$ is the projection of $\acan$ on a vocabulary
$\Gamma_\K$ not mentioning symbols from $\Sigma_1$, and likewise for
$\pi_{\Gamma_\K}(\asol)$. If the language accepted by $\Bau$ is empty, then
there is no universal solution for $\K_1$ under $\M$, otherwise a universal
solution (possibly of exponential size) exists, and we can compute it by
extracting the ABox encoded in 
some tree accepted by $\Bau$ . Summing up, we get:
%
%
\begin{theorem}
\label{the:non-emp-uni-sol-pspace-hard-and-exptime}
If extended ABoxes are allowed in universal solutions, then the non-emptiness
problem for universal solutions is \PSPACE-hard and in \EXPTIME.
\end{theorem}
%
%
%
%
%
Interestingly, the membership problem can be solved more efficiently in this
scenario, as now the candidate universal solutions are part of the input.
In the following theorem, we pinpoint the exact complexity of this problem.
\begin{theorem}
\label{the:memb-uni-sol-npcomplete}
If extended ABoxes are allowed in universal solutions, then the membership
problem for universal solutions is \NP-complete.
\end{theorem}

\section{Computing {\large $\UCQ$}-representations}
\label{sec-ucq-rep}

In Section~\ref{sec-univ}, we show that the complexity of the membership and
non-emptiness problems for universal solutions differ depending on whether
ABoxes or extended ABoxes are considered. On the other hand, we show in the
following proposition that the use of null values in ABoxes does not make any
difference in the case of $\UCQ$-representations. In this proposition, given a
mapping $\M$ and TBoxes $\T_1$, $\T_2$, we say that $\T_2$ is a
$\UCQ$-representation of $\T_1$ under $\M$ {\em considering extended ABoxes} if
$\T_1$, $\T_2$, $\M$ satisfy Equation~\eqref{eq-ucq-rep} in
Section~\ref{sec-sol}, but assuming that $\A_2$ is an extended ABox over
$\Sigma_2$ that is a $\UCQ$-solution for $\A_1$ under $\M$.
\begin{proposition}
  A TBox $\T_2$ is a $\UCQ$-representation of a TBox $\T_1$ under a mapping $\M$
  if and only if $\T_2$ is a $\UCQ$-representation of $\T_1$ under $\M$
  considering extended ABoxes.
\end{proposition}
Thus, from now on we study the membership and non-emptiness problems for
$\UCQ$-representations assuming that ABoxes can contain null values.

We start by considering the membership problem for $\UCQ$-representations. In
this case, one can immediately notice some similarities between this task and
the membership problem for universal $\UCQ$-solutions, which was shown to be
$\PSPACE$-hard in Theorem~\ref{the:memb-uni-UCQ-sol-pspace-hard}.
However, the universal quantification over ABoxes in the definition of the
notion of $\UCQ$-representation makes the latter problem computationally
simpler,
which is illustrated by the following example.
\begin{example}
  \label{ex:repres-pieces}
  Assume that $\M = (\Sigma_1, \Sigma_2, \T_{12})$, where $\Sigma_1= \{
  \concept{F}, \role{S_1}, \role{S_2}, \role{T_1}, \role{T_2}\}$, $\Sigma_2 = \{
  \concept{F'}, \role{S'}, \role{T'}, \concept{G'}\}$ and $\T_{12} = \{ F \ISA
  F', S_1 \ISA S'$, $S_2 \ISA S', T_1 \ISA T', T_2 \ISA T', \SOMET{T_1^-} \ISA
  G'\}$. Moreover, assume that $\T_1 = \{F \ISA \SOMET{S_1}, F \ISA
    \SOMET{S_2}, \SOMET{S_1^-} \ISA \SOMET{T_1}, \SOMET{S_2^-} \ISA
    \SOMET{T_2}\}$ and $\T_2 = \{ F' \ISA \SOMET{S'}, \SOMET{{S'}^-} \ISA
  \SOMET{T'}, \SOMET{{T'}^-} \ISA G'\}$.
  If we were to verify whether $\tup{\T_2, \{F'(a)\}}$ is a universal
  $\UCQ$-solution for $\tup{\T_1, \{F(a)\}}$ under $\M$ (which it is in this
  case),
  then we would first need to construct the {\em path} $\pi= \tup{F'(a),
    S'(a,n), T'(n,m), G'(m)}$ formed by the inclusions in $\T_2$, where $n,m$
  are fresh null values, and then we would need to explore the translations
  according to $\M$ of all paths formed by the inclusions in $\T_1$ to find one
  that matches $\pi$.

  On the other hand, to verify whether $\T_2$ is a $\UCQ$-representation of
  $\T_1$ under $\M$, one does not need to execute any ``backtracking'', as it is
  sufficient to consider independently a polynomial number of pieces $\C$ taken
  from the paths formed by the inclusions in $\T_1$, each of them of polynomial
  size, and then checking whether the translation $\C'$ of $\C$ according to
  $\M$ matches with the paths formed from $\C'$ by the inclusions in $\T_2$.
  If any of these pieces does not satisfy this condition, then it can be
  transformed into a witness that Equation~\eqref{eq-ucq-rep} is not satisfied,
  showing that $\T_2$ is not a $\UCQ$-representation of $\T_1$ under $\M$ (as we
  have a universal quantification over the ABoxes over $\Sigma_1$ in the
  definition of $\UCQ$-representations).
  In fact, one of the pieces considered in this case is $\C = \tup{T_2(n,m)}$,
  where $n$, $m$ are null values, which does not satisfy the previous condition
  as the translation $\C'$ of $\C$ according to $\M$ is $\tup{T'(n,m)}$, and
  this does not match with the path $\tup{T'(n,m), G'(m)}$ formed from $\C'$ by
  the inclusions in $\T_2$. This particular case is transformed into an ABox
  $\A_1=\{T_2(b, c)\}$ and a query $q = T'(b,c) \wedge G'(c)$, where $b$, $c$
  are fresh constants, for which we have that Equation \eqref{eq-ucq-rep} is not
  satisfied.
\end{example}

Notice that disjointness assertions in the mapping may cause $\tup{\T_1 \cup
  \T_{12},\A_1}$ to become inconsistent for some source ABoxes $\A_1$ (which
will make all possible tuples to be in the answer to every query), therefore
additional conditions have to be imposed on $\T_2$. To give more intuition about
how the membership problem for $\UCQ$-representations is solved, we give an
example showing how one can deal with some of these inconsistency issues.
\begin{example}
  Assume that $\M = (\Sigma_1, \Sigma_2, \T_{12})$, where $\Sigma_1 = \{
  \concept{F}, \concept{G}, \concept{H} \}$, $\Sigma_2 = \{ \concept{F'},
  \concept{G'}, \concept{H'}\}$ and $\T_{12}= \{F \ISA F', G \ISA G', H \ISA
  H'\}$. Moreover, assume that $\T_1 = \{ F \ISA G\}$ and $\T_2 =\{F' \ISA
  G'\}$. In this case, it is clear that $\T_2$ is a $\UCQ$-representation of
  $\T_1$ under $\M$. However, if we add inclusion $H \ISA \neg G'$ to $\T_{12}$,
  then $\T_2$ is no longer a $\UCQ$-representation of $\T_1$ under $\M$. To see
  why this is the case,
  consider an ABox $\A_1=\{ F(a), H(a)\}$, which is consistent with $\T_1$, and
  a query $q=F'(b)$, where $b$ is a fresh constant. Then we have that $\cert(q,
  \tup{\T_1 \cup \T_{12}, \A_1}) =\{()\}$ as KB $\tup{\T_1 \cup \T_{12}, \A_1}$
  is inconsistent, while $\cert(q,\tup{\T_2, \A_2}) =\emptyset$ for
  $\UCQ$-solution $\A_2 = \{F'(a),H'(a)\}$ for $\A_1$ under $\M$. Thus, we
  conclude that Equation~\eqref{eq-ucq-rep} is violated in this case.
\end{example}
One can deal with the issue raised in the previous example by checking that on
every pair $(B,B')$ of $\T_1$-consistent basic concepts over
$\Sigma_1$,\footnote{A pair $(B, B)'$ is $\T$-consistent for a TBox $\T$, if the
  KB $\tup{\T, \{B(a), B'(a)\}}$ is consistent, where $a$ is an arbitrary
  constant.} it holds that:
%
$(B,B')$ is $(\T_1 \cup \T_{12})$-consistent if and only if $(B,B')$ is
$(\T_{12} \cup \T_{2})$-consistent,
%
and likewise for every pair of basic roles over $\Sigma_1$. This condition
guarantees that for every ABox $\A_1$ over $\Sigma_1$ that is consistent with
$\T_1$, it holds that: $\tup{\T_1 \cup \T_{12}, \A_1}$ is consistent if and only
if there exists an extended ABox $\A_2$ over $\Sigma_2$ such that $\A_2$ is a
$\UCQ$-solution for $\A_1$ under $\M$ and $\tup{\T_2, \A_2}$ is
consistent. Thus, the previous condition ensures that
the sets on the left- and right-hand side of Equation~\eqref{eq-ucq-rep}
coincide whenever the intersection on either of these sides is taken over an
empty set.

The following theorem, which requires of a lengthy and non-trivial proof, shows
that there exists an efficient algorithm for the membership problem for
$\UCQ$-representations that can deal with all the aforementioned issues.
%
\begin{theorem}\label{eq:memb-repr-nl}
  The membership problem for $\UCQ$-representations is
  \NLOGSPACE-complete.
\end{theorem}
%

We conclude by pointing out that the non-emptiness problem for
$\UCQ$-representations can also be solved efficiently. We give an intuition of
how this can be done in
the following example, where we say that $\T_1$ is \emph{\UCQ-representable
  under} $\M$ if there exists a \UCQ-representation $\T_2$ of $\T_1$ under $\M$.
\begin{example}
  \label{ex:repres-non-emp}
  Assume that $\M = (\Sigma_1, \Sigma_2, \T_{12})$, where $\Sigma_1 = \{
  \concept{F}, \concept{G}, \concept{H} \}$, $\Sigma_2 = \{ \concept{F'},
  \concept{G'}\}$ and $\T_{12}= \{F \ISA F', G \ISA G', H \ISA F'\}$. Moreover,
  assume that $\T_1 = \{ F \ISA G\}$. Then it follows that $\T_1 \cup \T_{12}
  \models F \ISA G'$, and in order for $\T_1$ to be $\UCQ$-representable under
  $\M$, the following condition must be satisfied:
  \begin{itemize}
  \item[$(\star)$]there exists a concept $B'$ over $\Sigma_2$ s.t.\
    $\T_{12} \models F \ISA B'$, and for each concept $B$ over $\Sigma_1$ with
    $\T_1 \cup \T_{12} \models B \ISA B'$ it follows that $\T_1 \cup \T_{12}
    \models B \ISA G'$.
  \end{itemize}
  The idea is then to add the inclusion $B' \ISA G'$ to a \UCQ-representation
  $\T_2$ so that $\T_{12} \cup \T_{2} \models F \ISA G'$ as well. In our case,
  concept $F'$ satisfies the condition $\T_{12} \models F \ISA F'$, but it does
  not satisfy the second requirement as $\T_1 \cup \T_{12} \models H \ISA F'$
  and $\T_1 \cup \T_{12}\not\models H\ISA G'$. In fact, $F' \ISA G'$ cannot be
  added to $\T_2$ as it would result in $\T_{12} \cup \T_2 \models H \ISA G'$,
  hence in Equation~\eqref{eq-ucq-rep}, the inclusion from right to left would
  be violated. There is no way to reflect the inclusion $F \ISA G'$ in the
  target, so in this case \mbox{$\T_1$ is not $\UCQ$-representable under $\M$}.
\end{example}
The proof of the following result requires of some
involved extensions of the techniques used to prove Theorem~\ref{eq:memb-repr-nl}.
\begin{theorem}\label{eq:ne-repr-nl}
  The non-emptiness problem for $\UCQ$-representations is \NLOGSPACE-complete.
\end{theorem}
The techniques used to prove Theorem~\ref{eq:ne-repr-nl}, which is sketched
in the example below.
\begin{example}
  Consider $\M$ and $\T_1$ from Example~\ref{ex:repres-non-emp}, but assuming
  that $\T_{12}$ does not contain the inclusion $H \ISA F'$. Again, $\T_1 \cup
  \T_{12} \models F \ISA G'$, but now condition~$(\star)$ is satisfied. Then, an
  algorithm for computing a representation essentially needs to take any $B'$
  given by condition~($\star$) and add the inclusion $B' \ISA F'$ to $\T_2$. In
  this case, $\T_2 = \{F' \ISA G'\}$ is a \UCQ-representation of $\T_1$ under
  $\M$.
\end{example}

\section{Conclusions}
\label{sec-conclusions}

In this paper, we have studied the problem of KB exchange for \owlql, improving
on previously known results with respect to both
the expressiveness of the ontology language and
the understanding of the computational properties of the problem.
Our investigation leaves open several issues, which we intend to address in the
future. First, it would be good to have characterizations of classes of source
KBs and mappings for which universal (\UCQ-)solutions are guaranteed to
exist. As for the computation of universal solutions, while we have pinned-down
the complexity of membership for extended ABoxes as \NP-complete, an exact bound
for the other case is still missing.  Moreover, it is easy to see that allowing
for 
inequalities between terms (e.g., $a \neq b$ in Example~\ref{exa-disj-source})
and for negated atoms in the (target) ABox would allow one to obtain more
universal solutions, but a full understanding of this case is still missing.
Finally, we intend to investigate the challenging problem of computing
universal \UCQ-solutions, adopting also here an automata-based approach.

\section{Acknowledgements}
This research has been partially supported by the EU IP project Optique (grant
FP7-318338), by EU Marie Curie action FP7-PEOPLE-2009-IRSES (grant 24761), by
ERC grant agreement DIADEM, no. 246858, and by Fondecyt grant 1131049.  Diego
Calvanese has been partially supported by the Wolfgang Pauli Institute Vienna. %
The authors are also grateful to Evgeny Sherkhonov and Roman Kontchakov for
helpful comments and discussions.



\ifextendedversion
\clearpage
\advance\oddsidemargin 1.5cm
\advance\evensidemargin 1.5cm
\advance\textwidth -3cm
\onecolumn
\appendix
\section{Definitions and Preliminary Results}

Let $\Sigma$ be a $\dlliter$ signature; a concept name $A$ (role name
$P$) is said to be over $\Sigma$, if $A \in \Sigma$ ($P \in
\Sigma$). A basic role $R$ is said to be over $\Sigma$, if, either it
is a role name over $\Sigma$, or $R=P^-$ for a role $P$ over $\Sigma$;
a basic concept $B$ is said to be over $\Sigma$, if either it is a
concept name, which is over $\Sigma$, or $B=\exists R$ and $R$ is a
basic role over $\Sigma$. We naturally extend these definitions to
TBoxes, ABoxes, KBs, and queries; so we can refer to $\Sigma$-TBoxes
or TBoxes over $\Sigma$, and analogiously for ABoxes, KBs, and
queries.

Define relation $\ISA^\R_\T$ to be the reflexive and transitive closure of the
following relation on the set of all basic roles over $N_R$:
\begin{eqnarray*}
\{(R_1,R_2) \mid R_1 \ISA R_2 \in \T \text{ or } R_1^- \ISA R_2^- \in \T\},
\end{eqnarray*}
and let $\ISA^\C_\T$ be the reflexive and transitive closure of the following
relation on the set of all basic concepts over $N_C$:
\begin{eqnarray*}
\{(B_1,B_2) \mid B_1 \ISA B_2  \in \T\} \cup \{(\exists R_1, \exists R_2) \mid R_1 \ISA_\T^\R R_2\}.
\end{eqnarray*}
Then define the relation $\vdash$ between $\K$ and the $\dlliter$ membership
assertions over $\Sigma$ as:
\begin{align*}
  & \{(\K, B(a)) \ \mid \ \text{there exists a basic concept } B'  \text{ s.t. }
  \A \models B'(a) \text{ and } B' \ISA^\C_\T B\} \ \cup\\ 
  & \{(\K, R(a,b)) \ \mid \ \text{there exists a basic role } R' \text{ s.t. }
  \A \models R'(a,b) \text{ and } R' \ISA^\R_\T R\}.
\end{align*}
Notice that for consistent $\K$, for every membership assertion $\alpha$ it
holds that $\K \vdash \alpha$ if and only if $\K \models \alpha$. Moreover, for
every basic role $R$ over $N_R$, define $[R]$ as $\{S \mid R \ISA^\R_\T S$ and
$S \ISA^\R_\T R\}$, and then let $\leq_\T$ be a partial order on the set $\{[R]
\mid R$ is a basic role over $N_R\}$ defined as $[R] \leq_\T [S]$ if $R
\ISA^\R_\T S$. For each set $[R]$, where $R$ is a basic role, consider an
element $w_{[R]}$, \emph{witness for } $[R]$. Now, define a \emph{generating
  relationship} $\leadsto_\K$ between the set $N_a \cup \{ w_{[R]} \mid R$ is a
basic role$\}$ and the set $\{ w_{[R]} \mid R$ is a basic role$\}$, as follows:
\begin{itemize}
\item $a \leadsto_\K w_{[R]}$, if
\begin{inparaenum}[(1)]
\item $\K \vdash \SOMET{R}(a)$;
\item $\K \not\vdash R(a,b)$ for every $b\in N_a$;
\item $[R']=[R]$ for every $[R']$ such that $[R'] \leq_\T [R]$ and $\K
  \vdash \exists R'(a)$.
\end{inparaenum}
\item $w_{[S]} \leadsto_\K w_{[R]}$, if \begin{inparaenum}[(1)]
\item $\T \vdash \SOMET{S^-} \ISA \SOMET{R}$;
\item $[S^-] \neq [R]$;
\item $[R']=[R]$ for every $[R']$ such that $[R'] \leq_\T [R]$ and
  $\T \vdash \SOMET{S^-} \ISA \SOMET{R'}$.
\end{inparaenum}
\end{itemize}

Denote by $\gpath(\K)$ the set of all $\K$-paths, where a $\K$-path is
a sequence $a \cdot w_{[R_1]} \cdot \ldots \cdot w_{[R_n]}$ (sometimes
we simply write $a w_{[R_1]} \ldots w_{[R_n]}$) such that $a \in N_a$,
$a \leadsto_\K w_{[R_1]}$ and $w_{[R_i]} \leadsto_\K w_{[R_{i+1}]}$
for every $i \in \{1, \ldots, n-1\}$.  Moreover, for every $\sigma \in
\gpath(\K)$, denote by $\tail(\sigma)$ the last element in $\sigma$.

With all the previous notation, we can finally define the
\emph{canonical model} $\Uni_\K$. The domain $\dom[\Uni_\K] $ of
$\Uni_\K$ is defined as $\gpath(\K)$, and $\Int[\Uni_\K]{a} = a$ for
every $a \in N_a$. Moreover, for every concept $A$:
\begin{eqnarray*}
  \Int[\Uni_\K]{A} & = & \{\sigma \in \gpath(\K) \mid \K \vdash A(\tail(\sigma)) \text{ or }
  \tail(\sigma) = w_{[R]} \text{ and } \T \vdash \SOMET{R^-} \ISA A\},
\end{eqnarray*}
and for every role $P$, we have that $\Int[\Uni_\K]{P}$ is defined as
follows:
\begin{align*}
  \{(\sigma_1,\sigma_2) \in \gpath(\K) \times \gpath(\K) \mid \ &  \K \vdash P(\tail(\sigma_1),\tail(\sigma_2)); \text{ or }\\
  &\sigma_2 = \sigma_1 \cdot w_{[R]},\, \tail(\sigma_1) \leadsto_\K w_{[R]} \text{ and } [R]\leq_\T [P];\text{ or }\\
  &\sigma_1 = \sigma_2 \cdot w_{[R]},\, \tail(\sigma_2) \leadsto_\K w_{[R]}
  \text{ and } [R] \leq_\T [P^-]\}.
\end{align*}
Notice that $\Uni_\K$ defined above can be treated (by ignoring sets
$N^{\Uni_\K}$ for some concepts and role names $N$) as a
$\Sigma$-interpretation, for any $\Sigma$. Denote also by $\Ind(\A)$
the set of constants occuring in $\A$.

Let us point out the similarity of our definition of $\Uni_\K$ with
the definition of the \emph{canonical model} $\M_\K$ defined in
\cite{KKLSWZ11}. When $\K$ is consistent, many results proved there
for $\M_\K$ apply to $\Uni_\K$. In particular, from the proof of
Theorem~5 in \cite{KKLSWZ11} we can immediately conclude:
\begin{claim}\label{th:uni-sat-kb}
  If $\K$ is consistent, $\Uni_\K$ is a model of $\K$.
\end{claim}

We are going to introduce the notions of $\Sigma$-types and
$\Sigma$-homomorphisms, heavily employed in the proofs. For an
interpretation $\I$ and a signature $\Sigma$, the
\emph{$\Sigma$-types} $\ttype[\Sigma]{\I}(x)$ and
$\rtype[\Sigma]{\I}(x,y)$ for $x,y \in \Delta^\I$ are given by
\begin{align*}
  \ttype[\Sigma]{\I}(x)=&\{B \text{ - basic concept over } \Sigma \mid
  x \in
  B^\I\},\\
  \rtype[\Sigma]{\I}(x,y)=&\{R \text{ - basic role over } \Sigma \mid
  (x,y) \in R^\I\}.
\end{align*}
We also use $\ttype{\I}(x)$ and $\rtype{\I}(x,y)$ to refer to the
types over the signature of all \dllite concepts and roles.  A
\emph{$\Sigma$-homomorphism} from an interpretation $\I$ to $\I'$ is a
function $h: \Delta^\I \mapsto \Delta^{\I'}$ such that $h(a^\I) =
a^{\I'}$, for all individual names $a$ interpreted in $\I$,
$\ttype[\Sigma]{\I}(x) \subseteq \ttype[\Sigma]{\I'}(h(x))$ and
$\rtype[\Sigma]{\I}(x,y) \subseteq \rtype[\Sigma]{\I'}(h(x),h(y))$ for
all $x,y \in \Delta^\I$. We say that $\I$ is (\emph{finitely})
$\Sigma$-\emph{homomorphically embeddable} into $\I'$ if, for every
(finite) subinterpretation $\I_1$ of $\I$, there exists a
$\Sigma$-homomorphism from $\I_1$ to $\I'$. If $\Sigma$ is a set of
all \dllite concepts and roles, we call $\Sigma$-homomorphism simply
\emph{homomorphism}.

The claim below from the proof of Theorem~5 in \cite{KKLSWZ11}
establishes the relation between $\Uni_\K$ and the models of $\K$.
\begin{claim}\label{lem:konev18}
  For every model $\I \models \K$, there exists a homomorphism from
  $\Uni_\K$ to $\I$.
\end{claim}
%
Another result follows from Theorem~5 in \cite{KKLSWZ11}:
\begin{claim}\label{th:query-kb-vs-uni}
  For each consistent KB $\K$, every UCQ $q(\vec{x})$ and tuple
  $\vec{a} \subseteq N_a$, it holds $\K \models q[\vec{a}]$ iff
  $\Uni_{\K} \models q[\vec{a}]$.
\end{claim}
It is important to notice that the notion of certain answers can be
characterized through the notion of canonical model. Finally, for a
signature $\Sigma$ and two KBs $\K_1=\tup{\T_1, \A_1}$ and
$\K_2=\tup{\T_2, \A_2}$, we say that $\K_1$ \emph{$\Sigma$-query
  entails} $\K_2$ if, for all $\Sigma$-queries $q(\vec{x})$ and all
$\vec{a} \subseteq N_a$, $\K_2 \models q[\vec{a}]$ implies
$\K_1 \models q[\vec{a}]$. The KBs $\K_1$ and $\K_2$ are said to be
\emph{$\Sigma$-query equivalent} if $\K_1$ $\Sigma$-query entails
$\K_2$ and vice versa.
The following is a consequence of Theorem~7 in \cite{KKLSWZ11}:
\begin{claim}\label{th:query-entail-homo}
  Let $\K_1$ and $\K_2$ be consistent KBs. Then $\K_1$ $\Sigma$-query
  entails $\K_2$ iff $\Uni_{\K_2}$ is finitely
  $\Sigma$-homomorphically embeddable into $\Uni_{\K_1}$.
\end{claim}

\vspace{0.3cm}

\section{Proofs in Section~\ref{sec-univ}}





\subsection{Definitions and Preliminary Results: Characterization of Universal Solutions}
\newcommand{\target}{\mathsf{InTarget}}
\newcommand{\targetnull}{\mathsf{InTargetNull}}
\newcommand{\memb}{\mathit{memb}}
\newcommand{\Edge}{\mathit{Edge}}

\newcommand{\Var}{\mathsf{Var}}
\newcommand{\Const}{\mathsf{Const}}

First, we define the notion of canonical model for extended ABoxes.
Let $\A$ be an extended ABox. Without loss of generality, assume that
$\A$ does not contain assertions of the form $\SOMET{R}(x)$. Then the
\emph{canonical model} of $\A$, denoted $\V_{\A}$ is defined as
follows: $\dom[\V_{\A}] = \Null(\A) \cup N_a$, where $\Null(\A)$ is
the set of labeled nulls mentioned in $\A$, $\Int[\V_{\A}]{a} = a$ for
each $a \in N_a$, $\Int[\V_{\A}]{A} = \{x \in \dom[\V_{\A}] \mid A(x)
\in \A\}$ for each atomic concept $A$, and $\Int[\V_{\A}]{P} = \{(x,y)
\in \dom[\V_{\A}]\times \dom[\V_{\A}] \mid P(x,y) \in \A\}$ for each
atomic role $P$. Let $h$ be a function from $N_a \cup N_l \to
\dom[\V_{\A_2}]$ such that $h(a) = \Int{a}$ for every $a \in N_a$ and $h(x)
= x$ for every $x \in N_l$. Then

\begin{lemma}
  $\V_{\A_2}$ is a model of $\A_2$ with substitution $h$.
\end{lemma}

\begin{lemma}
  For every model $\I \models \A_2$, there exists a homomorphism from
  $\V_{\A_2}$ to $\I$.
\end{lemma}
\begin{proof}
  Let $\I$ be a model of $\A_2$ with a substitution $h'$. Then $h'$ is the
  desired homomorphism from $\V_{\A_2}$ to $\I$.
\end{proof}

Given an extended ABox $\A$, we denote by $\dom[\A]$ the set of all constants
and nulls mentioned in $\A$, $\dom[\A] = \Ind(\A) \cup \Null(\A)$. Moreover,
given an interpretation $\I$, the \emph{size} of $\I$, denoted $\card{\I}$, is
the sum of the cardinalities of the interpretations of all predicates (the
domain is not included as it is always infinite).

Let us denote by \dlliterpos the positive fragment of \dlliter. More
precisely, a \dlliterpos TBox is a finite set of concept inclusions $B_1 \ISA
B_2$, where $B_1$, $B_2$ are basic concepts, and role inclusions $R_1 \ISA
R_2$, where $R_1$, $R_2$ are basic roles, 
and a \dlliterpos KB $\K$ is a pair $\tup{\T,\A}$, where $\T$ is a \dlliterpos
TBox and $\A$ is an (extended) \dlliter ABox (without inequalities).

The following lemma is a characterization of universal solutions in
\dlliterpos. Recall that in Proposition~4.1 from \cite{ABCRS12} we showed that
if $\tup{\T_1 \cup \T_{12}, \A_1}$ is consistent and a KB $\K_2$ is a universal
solution for $\tup{\T_1,\A_1}$ under $\M = (\Sigma_1, \Sigma_2, \T_{12})$, then
$\T_2$ is a trivial TBox, i.e., a TBox that admits the same models as the empty
TBox. Therefore, without loss of generality, in the rest of this section when we
talk about universal solutions, we mean target ABoxes.
\begin{lemma}
  \label{lem:A2-uni-sol-dlliter-pos-iff-homo-equiv}
  Let $\M = (\Sigma_1, \Sigma_2, \T_{12})$ be a \dlliterpos mapping, $\K_1
  =\tup{\T_1, \A_1}$ a \dlliterpos KBs over $\Sigma_1$, and $\A_2$ an (extended,
  without inequalities, without negation) ABox over $\Sigma_2$. Then, $\A_2$ is
  a universal solution (with extended ABoxes) for $\K_1$ under $\M$ iff
  $\V_{\A_2}$ is $\Sigma_2$-homomorphically equivalent to $\Uni_{\tup{\T_1 \cup
      \T_{12},\A_1}}$.
\end{lemma}

\begin{proof}
  ($\Rightarrow$) Let $\A_2$ be a universal solution for $\K_1$ under $\M$.
  Then $\V_{\A_2}$ is $\Sigma_2$-homomorphically equivalent to
  $\Uni_{\tup{\T_1 \cup \T_{12}, \A_1}}$: since $\A_2$ is a solution, there
  exists $\I$ a model of $\K_1$ such that $(\I,\V_{\A_2}) \models \T_{12}$. Then
  $\I \cup \V_{\A_2}$ is a model of $\tup{\T_1 \cup \T_{12}, \A_1}$, therefore
  there is a homomorphism $h$ from $\Uni_{\tup{\T_1 \cup \T_{12}, \A_1}}$ to $\I
  \cup \V_{\A_2}$. As $\Sigma_1$ and $\Sigma_2$ are disjoint signatures it
  follows that $h$ is a $\Sigma_2$-homomorphism from $\Uni_{\tup{\T_1 \cup
      \T_{12}, \A_1}}$ to $\V_{\A_2}$. On the other hand, as $\A_2$ is a
  universal solution, $\J$, the interpretation of $\Sigma_2$ obtained from
  $\Uni_{\tup{\T_1 \cup \T_{12}, \A_1}}$ is a model of $\A_2$ with a
  substitution $h'$. This $h'$ is exactly a homomorphism from $\V_{\A_2}$ to
  $\Uni_{\tup{\T_1 \cup \T_{12}, \A_1}}$.

  ($\Leftarrow$) Assume $\V_{\A_2}$ is $\Sigma_2$-homomorphically equivalent to
  $\Uni_{\tup{\T_1 \cup \T_{12},\A_1}}$. We show that $\A_2$ is a universal
  solution for $\K_1$ under $\M$.

  First, $\A_2$ is a solution for $\K_1$ under $\M$. Let $\J$ be a model of
  $\A_2$, and $h_1$ a homomorphism from $\V_{\tup{\emptyset,\A_2}}$ to
  $\J$. Furthermore, let $h$ be a $\Sigma_2$-homomorphism from $\Uni_{\tup{\T_1
      \cup \T_{12}, \A_1}}$ to $\V_{\A_2}$. Then $h' = h_1 \circ h$ is a
  $\Sigma_2$-homomorphism from $\Uni_{\tup{\T_1 \cup \T_{12},\A_1}}$ to
  $\J$. Let $\I$ be the interpretation of $\Sigma_1$ defined as the image of
  $h'$ applied to $\Uni_{\K_1}$, $\I = h'(\Uni_{\K_1})$. The it is easy to see
  that $\I$ is a model of $\K_1$ and $(\I,\J)\models \M$ as $\K_1$ and $\M$
  contain only positive information. Indeed, $\A_2$ is a solution for $\K_1$
  under $\M$.

  Second, $\A_2$ is a universal solution. Let $\I$ be a model of $\K_1$ and $\J$
  an interpretation of $\Sigma_2$ such that $(\I,\J) \models \M$. Then, since
  $\Uni_{\tup{\T_1 \cup \T_{12},\A_1}}$ is the canonical model of $\K_1 \cup
  \T_{12}$, there exists a homomorphism $h$ from $\Uni_{\tup{\T_1 \cup
      \T_{12},\A_1}}$ to $\I \cup \J$ ($\I \cup \J$ is a model of $\K_1 \cup
  \T_{12}$). In turn, there is a homomorphism $h_1$ from $\V_{\A_2}$ to
  $\Uni_{\tup{\T_1 \cup \T_{12},\A_1}}$, therefore $h' = h \circ h_1$ is a
  homomorphism from $\V_{\A_2}$ to $\I \cup \J$, and a $\Sigma_2$-homomorphism
  from $\V_{\A_2}$ to $\J$. Hence, $\J$ is a model of $\A_2$: take $h'$ as the
  substitution for the labeled nulls. By definition of universal solution,
  $\A_2$ is a universal solution for $\K_1$ under $\M$.
\end{proof}

The definition below is used in the characterization of universal solutions in
the general case. Its purpose is to single out the cases when a universal
solution does not exist due to the need to represent in the target a form of
negative information (for instance, in the form of inequalities of negated
atoms).
\begin{definition}
  Let $\M = (\Sigma_1, \Sigma_2, \T_{12})$ be a \dlliter mapping, and $\K_1
  =\tup{\T_1, \A_1}$ a \dlliter KB over $\Sigma_1$. Then, we say that $\K_1$ and
  $\M$ are \emph{$\Sigma_2$-positive} if
  \begin{enumerate}[label=\textbf{(\alph*)}]
  \item for each $b \in \Int[\Uni_{\tup{\T_1 \cup \T_{12}, \A_1}}]{B}$ and $c
    \in \Int[\Uni_{\tup{\T_1 \cup \T_{12}, \A_1}}]{C}$ with $\T_1 \models B \AND
    C \ISA \bot$, it is not the case that
    \[b \in \target\text{\quad and \quad}c \in \target,\]
  \item for each $(b_1,b_2) \in \Int[\Uni_{\tup{\T_1 \cup \T_{12}, \A_1}}]{R}$
    and $(c_1,c_2) \in \Int[\Uni_{\tup{\T_1 \cup \T_{12}, \A_1}}]{Q}$ with $\T_1
    \models R \AND Q \ISA \bot$ for basic roles $R,Q$, it is not the case that
    \[b_i \in \target\text{\quad and \quad}c_i \in \target\text{\quad for
    }i=1,2,\] 
  \item for each $(a,b) \in \Int[\Uni_{\tup{\T_1 \cup \T_{12}, \A_1}}]{R}$ and
    $(a,c) \in \Int[\Uni_{\tup{\T_1 \cup \T_{12}, \A_1}}]{Q}$ with $\T_1 \models
    R \AND Q \ISA \bot$ for basic roles $R,Q$, it is not the case that
    \[b \in \target\text{\quad and \quad}c \in \target,\] where 
    \[\target = \{x \in \dom[\Uni_{\tup{\T_1 \cup \T_{12}, \A_1}}] \mid
    \ttype[\Sigma_2]{\Uni_{\tup{\T_1 \cup \T_{12}, \A_1}}}(x) \neq \emptyset\}
    \cup N_a
    \]
  \item for each $B \ISA \NOT B' \in \T_{12}$, $\Int[\Uni_{\tup{\T_1 \cup
        \T_{12},\A_1}}]{B} = \emptyset$ and \\
    for each $R \ISA \NOT R' \in \T_{12}$, $\Int[\Uni_{\tup{\T_1 \cup
        \T_{12},\A_1}}]{R} = \emptyset$.
  \end{enumerate}
\end{definition}

In the following, given a TBox $\T$, we denote by $\T^{\pos}$ the subset of
$\T$ without disjointness assertions, and given a KB $\K = \tup{\T, \A}$, we
denote by $\K^{\pos}$ the KB $\tup{\T^{\pos}, \A}$
. Moreover, if $\M = (\Sigma_1, \Sigma_2, \T_{12})$ is a \dlliter mapping,
then $\M^{\pos}$ denotes the mapping $ (\Sigma_1, \Sigma_2,
\T_{12}^{\pos})$. Finally, we provide a characterization of universal solutions
in \dlliter.
\begin{lemma}
  \label{lem:A2-uni-sol-iff-homo-equiv-and-sigma2-positive}
  Let $\M = (\Sigma_1, \Sigma_2, \T_{12})$ be a \dlliter mapping, $\K_1
  =\tup{\T_1, \A_1}$ a \dlliter KBs over $\Sigma_1$, and $\A_2$ an (extended,
  without inequalities, without negation) ABox over $\Sigma_2$. Then, $\A_2$ is
  a universal solution (with extended ABoxes) for $\K_1$ under $\M$ iff
  \begin{enumerate}
  \item $\K_1$ and $\M$ are $\Sigma_2$-positive,
  \item $\A_2$ is a universal solution for $\K_1^{\pos}$ under $\M^{\pos}$.
  \end{enumerate}
\end{lemma}
\begin{proof}
  ($\Rightarrow$) Let $\A_2$ be a universal solution for $\K_1$ under $\M$.
  Then $\A_2$ is a universal solution for $\K_1^{pos}$ under $\M^{\pos}$.

  For the sake of contradiction, assume that $\K_1$ and $\M$ are not
  $\Sigma_2$-positive, and e.g., (\textbf{a}) does not hold, i.e., there is a
  disjointness constraint in $\T_1$ of the form $B \AND C \ISA \bot$, such that
  $b \in \Int[\Uni_{\tup{\T_1 \cup \T_{12}, \A_1}}]{B}$ and $c \in
  \Int[\Uni_{\tup{\T_1 \cup \T_{12}, \A_1}}]{C}$, and
  \[\begin{array}{c}
    \ttype[\Sigma_2]{\Uni_{\tup{\T_1 \cup \T_{12}, \A_1}}}(b) \neq
    \emptyset\text{ \quad or \quad}b \in N_a, \\
    \ttype[\Sigma_2]{\Uni_{\tup{\T_1 \cup \T_{12}, \A_1}}}(c) \neq
    \emptyset\text{ \quad or \quad}c \in N_a. 
  \end{array}\] 
  Let $h$ be a $\Sigma_2$-homomorphism from $\Uni_{\tup{\T_1 \cup \T_{12},
      \A_1}}$ to $\V_{\A_2}$ (it exists by
  Lemma~\ref{lem:A2-uni-sol-dlliter-pos-iff-homo-equiv}). Then it follows that
  \[\begin{array}{c}
    \ttype[\Sigma_2]{\V_{\A_2}}(h(b)) \neq \emptyset\text{ \quad or \quad}b \in N_a\\
    \ttype[\Sigma_2]{\V_{\A_2}}(h(c)) \neq \emptyset\text{ \quad or \quad}c \in N_a
  \end{array}\]
  Take a minimal model $\J$ of $\A_2$ with a substitution $h'$ such that
  $h'(h(b)) = h'(h(c))$. Assume that both $b$ and $c$ are constants (i.e.,
  $\Int[\J]{b} = \Int[\J]{c}$). Then, obviously there exists no model $\I$ of
  $\Sigma_1$ such that $\I \models \K_1$ and $(\I,\J) \models \T_{12}$: in
  every such $\I$, $\Int{b}$ must be equal to $\Int{c}$ which contradicts $B
  \AND C \ISA \bot$, and $\Int{b} \in \Int{B}$ and $\Int{c} \in \Int{C}$. Now,
  assume that at least $b$ is not a constant and $\tail(b)=w_{[R]}$ for some
  role $R$ over $\Sigma_1$ (hence, $b \in \Int[\Uni_{\tup{\T_1 \cup \T_{12},
      \A_1}}]{(\SOMET{R^-})}$ and $\T_1 \models \SOMET{R^-} \ISA B$). Let $B'
  \in \ttype[\Sigma_2]{\Uni_{\tup{\T_1 \cup \T_{12}, \A_1}}}(b)$, then by
  construction of the canonical model, $\T_1 \cup \T_{12} \models \SOMET{R^-}
  \ISA B'$, by homomorphism, $B'(h(b)) \in \A_2$, $h'(h(b)) \in \Int[\J]{B'}$,
  and since $\J$ is a minimal model, $\Int[\J]{B'}$ is minimal. As $\A_2$ is a
  universal solution, let $\I$ be a model of $\K_1$ such that $(\I,\J)$
  satisfy $\T_{12}$. Then $\Int{(\SOMET{R^-})}$ is not empty, and by
  minimality of $\Int[\J]{B'}$, it must be the case that $h'(h(b)) \in
  \Int{(\SOMET{R^-})}$, hence $h'(h(b)) \in \Int{B}$. By a similar argument,
  it can be shown that $h'(h(c))$ must be in $\Int{C}$. As we took $\J$ such
  that $h'(h(b)) = h'(h(c))$, it contradicts that $\I$ is a model of $B \AND C
  \ISA \bot$. Contradiction with $\A_2$ being a universal solution.
  Similar to (\textbf{a}) we can derive a contradiction if assume that
  (\textbf{b}) or (\textbf{c}) does not hold.

  Finally, assume (\textbf{d}) does not hold, i.e., $B \ISA \NOT B' \in \T_{12}$
  and $\Int[\Uni_{\tup{\T_1 \cup \T_{12}, \A_1}}]{B} \neq \emptyset$. Note that
  $\A_2$ is an extended ABox, i.e., it contains only assertions of the form
  $A(u)$, $P(u,v)$ for $u,v \in \N_a \cup N_l$. Take a model $\J$ of $\A_2$ such
  that $\Int[\J]{B'} = \dom[\J]$. Such $\J$ exists as $\A_2$ contains only
  positive facts. Since $\A_2$ is a universal solution, there exist a model $\I$
  of $\K_1$ such that $(\I,\J) \models \T_{12}$. Then, $\Int{B} \neq \emptyset$,
  and it is easy to see that $(\I,\J)\not\models B \ISA \NOT B'$ because
  $\Int{B} \not\subseteq \dom[\J] \setminus \Int[\J]{B'} = \emptyset$. In every
  case we derive a contradiction, hence $\K_1$ and $\M$ are $\Sigma_2$-positive.

  ($\Leftarrow$) Assume conditions~1-2 are satisfied. We show that $\A_2$ is a
  universal solution for $\K_1$ under $\M$. 

  First, $\A_2$ is a solution for $\K_1$ under $\M$. Let $\J$ be a model of
  $\A_2$, then there exists $\I$ a model of $\K_1^{pos}$ such that $(\I,\J)
  \models \T_{12}^{\pos}$. Let $h$ be a homomorphism from $\Uni_{\K_1}$ to $\I$,
  and w.l.o.g., $\I = h(\Uni_{\K_1})$. Define a new function $h':
  \dom[\Uni_{\K_1}] \to \Delta \cup \dom[\I]$, where $\Delta$ is an infinite set
  of domain elements disjoint from $\dom$, as follows:
  \begin{itemize}
  \item $h'(x) = h(x)$ if $\ttype[\Sigma_2]{\Uni_{\tup{\T_1 \cup
          \T_{12},\A_1}}}(x) \neq \emptyset$ or $x \in N_a$.
  \item $h'(x) = d_x$, a fresh domain element from $\Delta$, otherwise.
  \end{itemize}
  We show that interpretation $\I'$ defined as the image of $h'$ applied to
  $\Uni_{\K}$, is a model of $\K_1$ and $(\I',\J)\models \M$.  Clearly, $\I'$
  is a model of the positive inclusions in $\T_1$ and $(\I',\J)$ satisfy the
  positive inclusions from $\T_{12}$. Let $\T_1 \models B \AND C \ISA \bot$
  for basic concepts $B,C$. By contradiction, assume $\I' \not \models B \AND
  C \ISA \bot$, i.e., for some $d \in \dom[\I']$, $d \in \Int[\I']{B} \cap
  \Int[\I']{C}$. We defined $\I'$ as the image of $h'$ on $\Uni_{\K_1}$, hence
  there must exist $b,c\in \dom[\Uni_{\K_1}]$ such that $b \in
  \Int[\Uni_{\K_1}]{B}$, $c \in \Int[\Uni_{\K_1}]{C}$, and $h'(b) = h'(c) =
  d$. Then it cannot be the case that \big[$\ttype[\Sigma_2]{\Uni_{\tup{\T_1
        \cup \T_{12},\A_1}}}(b) \neq \emptyset$ or $b$ is a constant \big],
  and \big[ $\ttype[\Sigma_2]{\Uni_{\tup{\T_1 \cup \T_{12},\A_1}}}(c) \neq
  \emptyset$ or $c$ is a constant~\big] as it contradicts (\textbf{a}) in the
  definition of $\K_1$ and $\M$ are $\Sigma_2$-positive. Assume $b$ is a null
  and $\ttype[\Sigma_2]{\Uni_{\tup{\T_1 \cup \T_{12},\A_1}}}(b) =
  \emptyset$. Then by definition of $h'$, $h'(b) = d_b \in \Delta$ (and $d =
  d_b$). In either case $c$ is a constant, or
  $\ttype[\Sigma_2]{\Uni_{\tup{\T_1 \cup \T_{12},\A_1}}}(c) \neq \emptyset$,
  or $\ttype[\Sigma_2]{\Uni_{\tup{\T_1 \cup \T_{12},\A_1}}}(c) = \emptyset$,
  we obtain contradiction with $h'(b) = d_b = h'(c)$ (remember, $\Delta$ and
  $\dom[\I]$ are disjoint). Contradiction rises from the assumption $\I
  \not\models B \AND C \ISA \bot$. Next, assume $\T_1 \models R \AND Q \ISA
  \bot$ for roles $R,Q$, and $\I' \not\models R \AND Q \ISA \bot$, i.e., for
  some $d_1,d_2 \in \dom[\I']$, $(d_1,d_2) \in \Int[\I']{R} \cap
  \Int[\I']{Q}$. We defined $\I'$ as the image of $h'$ on $\Uni_{\K_1}$, hence
  there must exist $b_1, b_2, c_1, c_2 \in \dom[\Uni_{\K_1}]$ such that
  $(b_1,b_2) \in \Int[\Uni_{\K_1}]{R}$, $(c_1,c_2) \in \Int[\Uni_{\K_1}]{Q}$,
  and $h'(b_i) = h'(c_i) = d_i$ for $i=1,2$. Then it cannot be the case that
  \big[$\ttype[\Sigma_2]{\Uni_{\tup{\T_1 \cup \T_{12},\A_1}}}(b_i) \neq
  \emptyset$ or $b_i$ is a constant \big], and \big[
  $\ttype[\Sigma_2]{\Uni_{\tup{\T_1 \cup \T_{12},\A_1}}}(c_i) \neq \emptyset$
  or $c_i$ is a constant~\big] as it contradicts (\textbf{a}) in the
  definition of $\K_1$ and $\M$ are $\Sigma_2$-positive. Consider the
  following cases:
  \begin{itemize}
  \item $b_1$ is a null and $\ttype[\Sigma_2]{\Uni_{\tup{\T_1 \cup
          \T_{12},\A_1}}}(b_1) = \emptyset$. Then by definition of $h'$,
    $h'(b_1) = d_{b_1} \in \Delta$ (and $d_1 = d_{b_1}$).
    \begin{itemize}
    \item $c_1$ is a null and $\ttype[\Sigma_2]{\Uni_{\tup{\T_1 \cup
            \T_{12},\A_1}}}(c_1) = \emptyset$, then $h'(c_1) = d_{c_1} = d_1$,
      hence $c_1 = b_1$ and $(b_1,b_2) \in \Int[\Uni_{\K_1}]{R}$, $(b_1,c_2)
      \in \Int[\Uni_{\K_1}]{Q}$. By (\textbf{c}) in the definition of $\K_1$
      and $\M$ are $\Sigma_2$-positive, it cannot be the case that
      \big[$\ttype[\Sigma_2]{\Uni_{\tup{\T_1 \cup \T_{12},\A_1}}}(b_2) \neq
      \emptyset$ or $b_2$ is a constant \big], and \big[
      $\ttype[\Sigma_2]{\Uni_{\tup{\T_1 \cup \T_{12},\A_1}}}(c_2) \neq
      \emptyset$ or $c_2$ is a constant~\big]. Assume $b_2$ is a null and
      $\ttype[\Sigma_2]{\Uni_{\tup{\T_1 \cup \T_{12},\A_1}}}(b_2) =
      \emptyset$. Then $h'(b_2) = d_{b_2} \in \Delta$ and in either case $c_2$
      is a constant, or $\ttype[\Sigma_2]{\Uni_{\tup{\T_1 \cup
            \T_{12},\A_1}}}(c_2) \neq \emptyset$, or
      $\ttype[\Sigma_2]{\Uni_{\tup{\T_1 \cup \T_{12},\A_1}}}(c_2) =
      \emptyset$, we obtain contradiction with $h'(b_2) = d_{b_2} = h'(c_2)$
    \item otherwise we obtain contradiction with $h'(b_1) = d_{b_1} = h'(c_1)$
    \end{itemize}
  \end{itemize}
  The cases $b_2$ or $c_i$ are nulls with the empty $\Sigma_2$-type are
  covered by swapping $R$ and $Q$ or by taking their inverses.
  Finally, assume $B \ISA \NOT C \in \T_{12}$ and $(\I',\J) \not\models
  \T_{12}$, i.e., for some $d\in \Int[\I']{B}$, $d \notin \dom[\J] \setminus
  \Int[\J]{C}$. Then there must exist $b \in \Int[\Uni_{\K_1}]{B}$ such that
  $h'(b) = d$. Contradiction with (\textbf{d}). Therefore, indeed, $\I$ is a
  model of $\K_1$ and $(\I,\J) \models \T_{12}$. This concludes the proof
  $\A_2$ is a solution for $\K_1$ under $\M$.

  Second, $\A_2$ is a universal solution. Let $\I$ be a model of $\K_1$ and $\J$
  an interpretation of $\Sigma_2$ such that $(\I,\J) \models \T_{12}$. Then,
  $\I$ is a model of $\K_1^{\pos}$ and $(\I,\J) \models \T_{12}^{\pos}$, and as
  $\A_2$ is a universal solution for $\K_1^{\pos}$ under $\M^{\pos}$, it follows
  that $\J$ is a model $\A_2$.
\end{proof}

The following lemma shows that $\Sigma_2$-positiveness can be checked in
polynomial time.
\begin{lemma}\label{lem:check-sigma2-pos-in-ptime}
  Let $\M = (\Sigma_1, \Sigma_2, \T_{12})$ be a mapping, and $\K_1 =\tup{\T_1,
    \A_1}$ a KB over $\Sigma_1$. Then it can be decided in polynomial time
  whether $\K_1$ and $\M$ are $\Sigma_2$-positive.
\end{lemma}
\begin{proof}
  We check (\textbf{a}) as follows: 
  \begin{itemize}
  \item for each concept disjointness axiom $B_1 \AND B_2 \ISA \bot \in \T_1$,
    check for $i=1,2$ if $\K_1 \models B_i(b_i)$ for some $b_i \in \Ind(\A_1)$
    or there exists a $\K_1$-path $x = a \cdot w_{[S_1]} \dots w_{[S_n]}$ such
    that $B_i \in \ttype{\Uni_{\tup{\T_1 \cup \T_{12},\A_1}}}(x)$ and
    $\ttype[\Sigma_2]{\Uni_{\tup{\T_1 \cup \T_{12},\A_1}}}(x) \neq
    \emptyset$. If yes, then (\textbf{a}) does not hold, otherwise it holds.
  \end{itemize}

  We check (\textbf{b}) as follows: 
  \begin{itemize}
  \item for each role disjointness axiom $R \AND Q \ISA \bot \in \T_1$, check
    for $i=1,2,3,4$ if $\K_1 \models B_i(b_i)$ for some $b_i \in \Ind(\A_1)$
    or there exists a $\K_1$-path $x=a \cdot w_{[S_1]} \dots w_{[S_n]}$ such
    that $B_i \in \ttype{\Uni_{\tup{\T_1 \cup \T_{12},\A_1}}}(x)$ and
    $\ttype[\Sigma_2]{\Uni_{\tup{\T_1 \cup \T_{12},\A_1}}}(x) \neq \emptyset$,
    where $B_1 = \SOMET{R}$, $B_2 = \SOMET{R^-}$, $B_3 = \SOMET{S}$, $B_4 =
    \SOMET{S^-}$. If yes, then (\textbf{b}) does not hold, otherwise it holds.
  \end{itemize}

  We check (\textbf{c}) as follows:
  \begin{itemize}
  \item for each role disjointness axiom $R_1 \AND R_2 \ISA \bot \in \T_1$,
    check if there exists a $\K_1$-path $x = a \cdot w_{[S_1]} \dots
    w_{[S_n]}$ such that $\SOMET{R_1}, \SOMET{R_2} \in \ttype{\Uni_{\tup{\T_1
          \cup \T_{12},\A_1}}}(x)$, then check for $i=1,2$ if $\K_1 \models
    R_i(x,b_i)$ for some $b_i \in \Ind(\A_1)$ or there exists a $\K_1$-path
    $y_i = a' \cdot w_{[Q_1]} \dots w_{[Q_n']}$ such that $R_i \in
    \rtype{\Uni_{\tup{\T_1 \cup \T_{12},\A_1}}}(x,y_i)$ and
    $\ttype[\Sigma_2]{\Uni_{\tup{\T_1 \cup \T_{12},\A_1}}}(y_i) \neq
    \emptyset$. If yes, then (\textbf{c}) does not hold, otherwise it holds.
  \end{itemize}
  Note that in the previous three checks, it is sufficient to look for paths
  where $n$ is bounded by the number of roles in $\K_1$, moreover in the last
  check $|n - n'| = 1$.

  We check (\textbf{d}) as follows:
  \begin{itemize}
  \item for each concept disjointness axiom $B \ISA \NOT B' \in \T_{12}$,
    check if $\K_1$ implies that $B$ is necessarily non-empty. If yes, then
    (\textbf{d}) does not hold, otherwise
  \item for each role disjointness axiom $R \ISA \NOT R' \in \T_{12}$, check
    if $\K_1$ implies that $R$ is necessarily non-empty. If yes, then
    (\textbf{d}) does not hold, otherwise it holds.
  \end{itemize}

  It is straightforward to see that each of the checks can be done in
  polynomial time as the standard reasoning in \dlliter is in \NLOGSPACE.
\end{proof}

\begin{lemma}
  \label{lem:uni-sol-exists-iff-chase-fin-part}
  Let $\M = (\Sigma_1, \Sigma_2, \T_{12})$ be a mapping, and $\K_1 =\tup{\T_1,
    \A_1}$ a KB over $\Sigma_1$ such that $\K_1$ and $\M$ are
  $\Sigma_2$-positive. Then, a universal solution (with extended ABoxes) for
  $\K_1$ under $\M$ exists iff $\Uni_{\tup{\T_1 \cup \T_{12},\A_1}}$ is
  $\Sigma_2$-homomorphically embeddable into a finite subset of itself.
\end{lemma}
\begin{proof}
  ($\Leftarrow$) Let ABox $\A_2$ be an ABox over $\Sigma_2$ such that
  $\V_{\A_2}$ is a finite subset of $\Uni_{\tup{\T_1 \cup \T_{12}, \A_1}}$
  and there exists a $\Sigma_2$-homomorphism $h$ from $\Uni_{\tup{\T_1 \cup
      \T_{12}, \A_1}}$ to $\V_{\A_2}$. Then, $\Uni_{\tup{\emptyset,\A_2}}$
  is trivially homomorphically embeddable into $\Uni_{\tup{\T_1 \cup
      \T_{12},\A_1}}$. Hence by
  Lemma~\ref{lem:A2-uni-sol-iff-homo-equiv-and-sigma2-positive}, $\A_2$ is a
  universal solution for $\K_1$ under $\M$.

  ($\Rightarrow$) Let $\A_2$ be a universal solution for $\K_1$ under
  $\M$. Then $\V_{\A_2}$ is $\Sigma_2$-homomorphically equivalent to
  $\Uni_{\tup{\T_1 \cup \T_{12}, \A_1}}$ by
  Lemma~\ref{lem:A2-uni-sol-iff-homo-equiv-and-sigma2-positive}.  Let $h$ be a
  homomorphism from $\V_{\A_2}$ to $\Uni_{\tup{\T_1 \cup \T_{12}, \A_1}}$,
  and $h(\V_{\A_2})$ the image of $h$. Then, $h(\V_{\A_2})$ is a finite
  subset of $\Uni_{\tup{\T_1 \cup \T_{12}, \A_1}}$, moreover it is
  homomorphically equivalent to $\V_{\A_2}$ and to $\Uni_{\tup{\T_1 \cup
      \T_{12}, \A_1}}$. Therefore, it follows that $\Uni_{\tup{\T_1 \cup
      \T_{12}, \A_1}}$ is $\Sigma_2$-homomorphically embeddable to a finite
  subset of itself.
\end{proof}

\subsection{Definitions and Preliminary Results: The Automata Construction for
  Theorem~\ref{the:non-emp-uni-sol-pspace-hard-and-exptime}}
\label{subsec:automata-construction}
\subsubsection{Definition of alternating two-way automatas}
Infinite trees are represented as prefix closed (infinite) sets of words over
$\mathbb{N}$ (the set of positive natural numbers). Formally, an infinite tree
is a set of words $T \subseteq \mathbb{N}^*$, such that if $x \cdot c \in T$,
where $x\in \mathbb{N}^*$ and $c\in \mathbb{N}$, then also $x\in T$. The
elements of $T$ are called nodes, the empty word $\epsilon$ is the root of
$T$, and for every $x \in T$, the nodes $x\cdot c$, with $c \in \mathbb{N}$,
are the successors of $x$. By convention we take $x \cdot 0 = x$, and $x\cdot
i \cdot -1 = x$. The branching degree $d(x)$ of a node $x$ denotes the number
of successors of $x$. If the branching degree of all nodes of a tree is
bounded by $k$, we say that the tree has branching degree $k$. An infinite
path $P$ of $T$ is a prefix closed set $P \subseteq T$ such that for every $i
\geq 0$ there exists a unique node $x \in P$ with $\card{x} = i$. A labeled
tree over an alphabet $\Sigma$ is a pair $(T,V)$, where $T$ is a tree and $V
:T \to \Sigma$ maps each node of $T$ to an element of $\Sigma$.

\emph{Alternating automata on infinite trees} are a generalization of
nondeterministic automata on infinite trees, introduced in [9]. They allow for
an elegant reduction of decision problems for temporal and program logics [3,
1]. Let $\B(I)$ be the set of positive boolean formulae over $I$, built
inductively by applying $\wedge$ and $\vee$ starting from true, false, and
elements of $I$. For a set $J \subseteq I$ and a formula $\phi \in \B(I)$, we
say that $J$ satisfies $\phi$ if and only if, assigning true to the elements
in $J$ and false to those in $I\setminus J$, makes $\phi$ true. For a positive
integer $k$, let $[k]=\{-1, 0, 1,\ldots, k\}$. A \emph{two-way alternating
  tree automaton (2ATA)} running over infinite trees with branching degree
$k$, is a tuple $\Au = \tup{\Sigma,Q,\delta,q_0,F}$, where $\Sigma$ is the
input alphabet, $Q$ is a finite set of states, $\delta : Q \times \Sigma \to
\B([k] \times Q)$ is the transition function, $q_0 \in Q$ is the initial
state, and $F$ specifies the acceptance condition.

The transition function maps a state $q \in Q$ and an input letter $\sigma \in
\Sigma$ to a positive boolean formula over $[k]\times Q$. Intuitively, if
$\delta(q, \sigma) = \phi$, then each pair $(c, q')$ appearing in $\phi$
corresponds to a new copy of the automaton going to the direction suggested by
$c$ and starting in state $q'$. For example, if $k = 2$ and
$\delta(q_1,\sigma) = ((1,q_2) \wedge (1,q_3)) \vee ((-1, q_1) \wedge (0,
q_3))$, when the automaton is in the state $q_1$ and is reading the node $x$
labeled by the letter $\sigma$, it proceeds either by sending off two copies,
in the states $q_2$ and $q_3$ respectively, to the first successor of $x$
(i.e., $x\cdot 1)$, or by sending off one copy in the state $q_1$ to the
predecessor of $x$ (i.e., $x\cdot -1)$ and one copy in the state $q_3$ to $x$
itself (i.e., $x\cdot 0$).

A run of a 2ATA $\Au$ over a labeled tree $(T, V)$ is a labeled tree
$(T_{\run}, \run)$ in which every node is labeled by an element of $T \times
Q$. A node in $T_\run$ labeled by $(x, q)$ describes a copy of $A$ that is in
the state $q$ and reads the node $x$ of $T$. The labels of adjacent nodes have
to satisfy the transition function of $\Au$. Formally, a run $(T_\run,\run)$
is a $T \times Q$-labeled tree satisfying:
\begin{itemize}
\item $\epsilon \in T_\run$ and $\run(\epsilon)=(\epsilon,q_0)$.  
\item Let $y \in T_\run$, with $\run(y) = (x,q)$ and $\delta(q,V(x)) =
  \phi$. Then there is a (possibly empty) set $S = \{(c_1, q_1), \ldots ,
  (c_n, q_n)\} \subseteq [k] \times Q$ such that:
  \begin{itemize}
  \item $S$ satisfies $\phi$ and
  \item for all $1\leq i\leq n$, we have that $y \cdot i \in T_\run$, $x\cdot
    c_i$ is defined ($x \cdot c_i \in T$), and $\run(y\cdot i)= (x\cdot c_i,
    q_i)$.
  \end{itemize}

\end{itemize}
A run $(T_\run,\run)$ is accepting if all its infinite paths satisfy the
acceptance condition. Given an infinite path $P \in T_\run$, let
$\mathit{inf}(P) \subseteq Q$ be the set of states that appear infinitely
often in $P$ (as second components of node labels). We consider here B\"uchi
acceptance conditions. A B\"uchi condition over a state set $Q$ is a subset
$F$ of $Q$, and an infinite path $P$ satisfies $F$ if $\mathit{inf}(P) \cap F
\neq \emptyset$.

The non-emptiness problem for 2ATAs consists in determining, for a given 2ATA,
whether the set of trees it accepts is nonempty. It is known that this problem
can be solved in exponential time in the number of states of the input
automaton $\Au$, but in linear time in the size of the alphabet as well as in
the size of the transition function of~$\Au$.

\subsubsection{The automata construction}
Now, we are going to construct two 2ATA automatas and a one-way
non-deterministic automata to use them as a mechanism to decide the
non-emptiness problem for universal solutions.
More specifically, let $\Sigma_1$, $\Sigma_2$ be signatures with no concepts
or roles in common, and $\K = \tup{\T, \A}$ a KB over $\Sigma_1 \cup
\Sigma_2$, $\NN = \{a_1, \ldots, a_n\}$ be the set of individuals in $\A_1$,
$\BB$ be the set of basic concepts and $\RR$ be the set of basic roles over
the signature of $\K$ (that is, over $\Sigma_1 \cup \Sigma_2$). Finally,
assume that $r$, $G$ are special characters not mentioned in $\NN \cup \BB
\cup \RR$, and let $\PP = \{P_{ij} \mid P \text{ is an atomic role over the
  signature of } \K \text{ and } 1\leq i,j \leq n\}$. Then assuming that
$\Sigma_\K = 2^{\NN \cup \BB \cup \RR \cup \PP \cup \{r, G\}}$ and $\Gamma_\K
= \{ \sigma \in \Sigma_\K \mid r \in \sigma$, $\sigma \cap \NN \neq
\emptyset$, or every basic concept and every basic role in $\sigma$ is over
$\Sigma_2\}$, we construct the following automata:
\begin{itemize}
\item $\acan$: The alphabet of this automaton is $\Sigma_\K$, and it accepts
  trees that are essentially the tree corresponding to the canonical model of
  $\K$, but with nodes arbitrary labeled with the special character $G$.

\item $\asol$: The alphabet of this automaton is $\Sigma_\K$, and it accepts a
  tree if its subtree labeled with $G$ corresponds to a tree model $\I$ of
  $\K$ (tree models are models which from trees on the labeled nulls).

\item $\agfin$: The alphabet of this automaton is $\Gamma_\K$, and it accepts
  a tree if it has a finite prefix where each node is marked with the special
  symbol $G$, and no other node in the tree is marked with $G$.
\end{itemize}

\subsubsection{Automaton $\acan$ for the canonical model of $\K = \tup{\T,\A}$}

$\acan$ is a two way alternating tree automaton (2ATA) that accepts the tree
corresponding to the canonical model of the \dlliter KB $\K = \tup{\T,\A}$,
with nodes arbitrarily labeled with a special character $G$. Formally, $\acan
= \tup{\Sigma_{\K}, Q_{\cn}, \delta_{\cn}, q_0 , F_{\cn}}$, where
\begin{eqnarray*}
  Q_{\cn} & = & \{q_0, q_s,q^*_{\NOT r},q_d\} \cup \{q^*_{X}, q^*_{\NOT X} \mid X \in
  \NN \cup \BB \cup \RR \cup \PP\} \cup \{q_{\SOMET{R}}, q_R \mid R \in \RR\}, 
\end{eqnarray*}
and the transition function $\delta_{\cn}$ is defined as follows. Assume
without loss of generality that the number of basic roles over the signature
of $\K$ is equal to $n$ (this can always be done by adding the required
assertions to the ABox), and let $f : \RR \to \{1, \ldots, n\}$ be a
one-to-one function.  Then $\delta_{\cn}: Q_{\cn} \times \Sigma_\K \rightarrow
\B([n] \times Q_{\cn})$ is defined as:
\begin{enumerate}
\item For each $\sigma \in \Sigma_\K$ such that $r \in \sigma$,
  $\delta_{\cn}(q_0,\sigma)$ is defined as:
  \begin{align*}
    \bigwedge_{i=1}^n \bigg[ & (i, q_s) \wedge (i, q^*_{\NOT r}) \wedge (i,
    q^*_{a_i}) \wedge \bigg(\bigwedge_{j \in \{1, \ldots, n\} \,:\, j \neq i}
    (i,q^*_{\NOT a_j})\bigg)
    \ \wedge\\
    & \bigwedge_{j = 1}^n \bigg(\bigwedge_{P \in \PP \,:\, \K \models
      P(a_i,a_j)} (0, q^*_{P_{ij}}) \wedge \bigwedge_{P \in \PP \,:\, \K
      \not\models P(a_i,a_j)} (0, q^*_{\NOT P_{ij}})\bigg) \ \wedge\\
    & \bigg(\bigwedge_{B \in \BB \,:\, \K \models B(a_i)} (i,q^*_{B})\bigg) \wedge
    \bigg(\bigwedge_{B \in \BB  \,:\,  \K \not\models B(a_i)} (i,q^*_{\NOT B})\bigg)  \wedge {} \\
    & \bigg(\bigwedge_{R \in \RR \,:\, \substack{R\text{ is $\leq_\T$-minimal s.t. }\\
        \K \models \SOMET{R}(a_i) \text{ and}\\
        \bigwedge_{j=1}^n \K \not\models R(a_i,a_j)}} (i,q_{\SOMET{R}}) \bigg) \wedge
    \bigg(\bigwedge_{R \in \RR \,:\, \substack{\K \not\models \SOMET{R}(a_i), \text{ or}\\
        \bigvee_{j=1}^n \K \models R(a_i,a_j), \text{ or} \\
        R\text{ is not $\leq_\T$-minimal}}} (i,q^{\ng}_{\SOMET{R}})\bigg)
    \bigg]
  \end{align*}

\item For each $\sigma \in \Sigma_\K$:
  \begin{multline*}
    \delta_{\cn}(q_s,\sigma)=\bigwedge_{i=1}^n \bigg[ (i, q_s) \wedge (i, q^*_{\NOT r}) \wedge 
    \bigwedge_{j=1}^n (i,q^*_{\NOT a_j})  \wedge
    \bigg((i,q_d) \vee \bigvee_{R \in \RR} (i,q^*_{R}) \bigg)
    \bigg]
  \end{multline*}

\item For each $\sigma \in \Sigma_\K$:
  \begin{eqnarray*}
    \delta_{\cn}(q_d,\sigma) & = & \bigwedge_{R \in \RR} (0,q^*_{\NOT R}) \wedge
    \bigwedge_{i=1}^n (i,q_d)      
  \end{eqnarray*}

\item For each $\sigma \in \Sigma_\K$ and each basic role $[R]$ from~$\RR$:
  \begin{equation*}\delta_{\cn}(q^{\ng}_{\SOMET{R}}, \sigma) = 
    {\displaystyle \bigwedge_{R' \in \RR}  (f(R),q^*_{\NOT R'})} 
  \end{equation*}

\item For each $\sigma \in \Sigma_\K$ and each basic role $[R]$ from~$\RR$:
  \begin{equation*}\delta_{\cn}(q_{\SOMET{R}}, \sigma) = 
    {\displaystyle (f(R),q_{R}) 
    }  
  \end{equation*}

\item For each $\sigma \in \Sigma_\K$ such that $\sigma \cap \NN =
  \emptyset$ and each basic role $[R]$ from~$\RR$, $\delta_{\cn}(q_{R}, \sigma)$
  is defined as
  \begin{multline*}
    {\displaystyle \bigg(\bigwedge_{R' \in \RR \,:\, \K \models R \ISA R'} %
      (0, q^*_{R'}) \bigg)} \wedge %
    {\displaystyle \bigg(\bigwedge_{R' \in \RR \,:\, \K \not\models R \ISA R'} %
      (0, q^*_{\NOT R'}) \bigg)} \wedge {}\\
    {\displaystyle\bigg(\bigwedge_{B \in \BB \,:\, \K \models\SOMET{R^-} \ISA B} %
      (0, q^*_B) \bigg)} \wedge %
    {\displaystyle \bigg(\bigwedge_{B \in \BB \,:\, \K \not\models\SOMET{R^-} \ISA B} %
      (0,q^*_{\NOT B}) \bigg)} \wedge {} \\
    {\displaystyle \bigg(\bigwedge_{S \in \RR \,:\, \substack{\, S\text{ is
            $\leq_\T$-minimal s.t. }\\ \K \models\SOMET{R^-}\ISA\SOMET{S}, \, [R^-]\neq [S]}} %
      (0,q_{\SOMET{S}}) \bigg)} \wedge %
    {\displaystyle \bigg(\bigwedge_{S \in \RR \,:\, \substack{ \K \not\models \SOMET{R^-} \ISA
          \SOMET{S}, \text{ or } [R^-] = [S],\\ \text{or }S\text{ is not $\leq_\T$-minimal}}}
      (0,q^{\ng}_{\SOMET{S}}) \bigg)} %
  \end{multline*}     

\item For each $\sigma \in \Sigma_\K$:
  \begin{equation*} \delta_{\cn}(q^*_{\NOT r}, \sigma) =
    \begin{cases}
      \text{{\it true}} & \text{if } r \not\in \sigma\\
      \text{{\it false}} & \text{otherwise}
    \end{cases}
  \end{equation*}

\item For each $\sigma \in \Sigma_\K$ and each $X \in \BB \cup \RR \cup \NN \cup
  \PP$:
  \begin{align*} \delta(q^*_{X}, \sigma) =
    \begin{cases}
      \text{{\it true}} & \text{if } X \in \sigma\\
      \text{{\it false}} & \text{otherwise}
    \end{cases}
    \qquad 
    \delta_{\cn}(q^*_{\NOT X}, \sigma) =
    \begin{cases}
      \text{\it true} & \text{if } X \notin \sigma\\
      \text{\it false} & \text{otherwise}
    \end{cases}
  \end{align*}
\end{enumerate}
Finally, the acceptance condition is $F_{\cn} = Q_{\cn}$.

\vspace{0.7cm} To represent the canonical model $\Uni_\K$ of $\K$ as a labeled
tree, we label each individual $x$ with the set of concepts $B$ such that $x
\in \Int[\Uni_\K]{B}$. We also add a basic role $R$ to the label of $x$
whenever $(x',x) \in \Int[\Uni_\K]{R}$ and $x$ is not an individual. Moreover,
we make sure this tree is an infinite full $n$-ary tree, where $n$ is the
number of individuals in $\Ind(\A)$ and basic roles in $\RR$. Thus, let $n^*$
be the set of sequences of numbers from $1$ to $n$ of the form $n^* = \{i_1
\cdot i_2 \cdot \ \cdots \ \cdot i_m \mid 1 \leq i_j \leq n, m \geq j \geq
0\}$, the sequence of length $0$ is denoted by $\epsilon$.

Recall that we have a numbering of individuals $\{a_1, \ldots, a_n\} =
\Ind(\A)$, and each role $R \in \RR$ can be identified through the number
$f(R) \in \{1, \ldots, n\}$. Therefore, the elements of $\dom[\Uni]$ can be
seen as sequences of natural numbers, namely a sequence $a_i \cdot w_{[R_1]}
\cdot \cdots \cdot w_{[R_m]}$ corresponds to the numeric sequence $i \cdot
f(R_1) \cdot \cdots \cdot f(R_m)$. However, for better readability we use the
original notation as $a_i \cdot w_{[R_1]} \cdot \cdots \cdot w_{[R_m]}$. Note,
that $\dom[\Uni] \subseteq n^*$.

In the following, we assume $\K$ is fixed and for simplicity we
use $\Uni$ instead of $\Uni_\K$.

The tree encoding of the canonical model $\Uni$ of $\K = \tup{\T,\A}$ is the
$\Sigma_\K$-labeled tree $T_{\Uni} = (n^*, \Int[\Uni]{V})$, such that
\begin{itemize}
\item $\Int[\Uni]{V}(\epsilon) = \{r\} \cup \{P_{ij} \mid (a_i,a_j) \in
  \Int[\Uni]{P}, P\text{ is an atomic role}\}$,
\item for each $x \in \dom[\Uni]$:
  \[\begin{array}[t]{rcl}
    \Int[\Uni]{V}(x) &= &\{B \mid x \in \Int[\Uni]{B}\} \cup {}\\
    & &\{S \mid (x',x) \in \Int[\Uni]{S}\text{ and }x = x'
    \cdot w_{[R]}\text{ for some role }R\text{ s.t.\ }[R] \leq_\T [S]\} \cup {}\\
    & &\{a \mid a \in \Ind(\A)\text{ and }x=a\}.
  \end{array}
  \] 
\end{itemize}

Conversely, we can see any $\Sigma_\K$-labeled tree as a representation of an
interpretation of $\K$, provided that each individual name occurs in the label
of only one node, a child of the root. Informally, the domain of this
interpretation are the nodes of the tree reachable from the root through a
sequence of roles, except the root itself. The extensions of individuals,
concepts and roles are determined by the node labels.

Given a $\Sigma_\K$-labeled tree $(T,V)$, we call a node $c$ an
\emph{individual node} if $a \in V(c)$ for some $a \in \Ind(\A)$, and we call
$c$ an $a$-node if we want to make the precise $a$ explicit. We say that $T$
is \emph{individual unique} if 
for each $a \in \Ind(\A)$ there is exactly one $a$-node, a child of the root
of $T$.

An individual unique $\Sigma_\K$-labeled tree $(T,V)$, \emph{represents} the
interpretation $\I_T$ defined as follows. For each role name $P$, let:
\[\begin{array}[t]{rcl}
  R_p & = & \{(x,x \cdot i) \mid P \in V(x \cdot i)\} \cup \{(x \cdot i,x) \mid
  P^- \in V(x \cdot i)\} \cup {} \\
  && \{(c, c') \mid a_i \in V(c), a_j \in V(c')\text{ and }P_{ij} \in V(\epsilon)\} 
\end{array}\]
and \[\dom[\I_T] = \{x \mid (i,x) \in \bigcup_{P \in \RR}
(R_P \cup R_P^-)^*, i \in \{1, \ldots, n\}\},\]
where $R_P^-$ denotes the inverse of relation $R_P$. Then the interpretation
$\I_T = (\dom[\I_T], \Int[\I_T]{\cdot})$ is defined as:
\[
\begin{array}{ccll}
  \Int[\I_T]{a_i} & = & c \text{ such that }a_i \in V(c), & \text{for each }a_i
  \in \Ind(\A)\\
  \Int[\I_T]{A} & = & \dom[\I_T] \cap \{x \mid A \in V(x)\}, & \text{for each
    atomic concept }A \in \BB\text{ and}\\
  \Int[\I_T]{P} & = & (\dom[\I_T] \times \dom[\I_T]) \cap R_P, & \text{for each
    atomic role }P \in \RR\\
\end{array}
\]

\begin{proposition}
  The following hold for $\acan$:
  \begin{itemize}
  \item $T_{\Uni} \in \L(\acan)$.
  \item for each $(T,V) \in \L(\acan)$, $(T,V)$ is individual unique and $\I_T$ is
    isomorphic to $\Uni$, the canonical model of $\K$.
  \end{itemize}
\end{proposition}

\begin{proof}
  For the first item, assume $T_{\Uni} = (n^*, \Int[\Uni]{V})$ is the tree
  encoding of the universal model $\Uni$ of $\K$.  We show that a full run of
  $\acan$ over $T_{\Uni}$ exists. 

  The run $(T_\run,\run)$ is built starting from the root $\epsilon$, and
  setting $\run(\epsilon)$ = $(\epsilon,q_0)$. Then, to correctly execute the
  initial transition, the root has children as follows:

  \begin{itemize}
  \item for each $a_k \in \Ind(\A)$
    \begin{itemize}
    \item a child $k_{s}$ with $\run(k_{s}) = (a_k, q_{s})$,
    \item a child $k^*_{\neg r}$ with $\run(k^*_{\neg r}) = (a_k, q^*_{\neg r})$,
    \item a child $k^*_{a_k}$ with $\run(k^*_{a_k}) = (a_k, q^*_{a_k})$,
    \item a child $k^*_{\neg a_j}$ for each $j \neq k$ with $\run(k^*_{\neg
      a_j}) = (a_k, q^*_{\neg a_j})$,
    \item a child $k^*_{B}$ for each $B \in \BB$ such that $a_k \in
      \Int[\Uni]{B}$, with $\run(k^*_{B}) = (a_k, q^*_B)$,
    \item a child $k^*_{\neg B}$ for each $B \in \BB$ such that $a_k \not\in
      \Int[\Uni]{B}$, 
      with $\run(k^*_{\neg B}) = (a_k, q^*_{\neg B})$,
    \item a child $k_{\SOMET{R}}$ for each $\leq_\T$-minimal role $R$ s.t.\
      $\Uni \models \SOMET{R}(a_i)$ and $\Uni \not\models R(a_i,a_j)$ for each
      $j \in \{1, \ldots, n\}$, with $\run(k_{\SOMET{R}}) = (a_k,
      q_{\SOMET{R}})$,
    \item a child $k^{\ng}_{\SOMET{R}}$ for each role $R$ s.t.\ $\Uni
      \not\models \SOMET{R}(a_i)$, or $\Uni \models R(a_i,a_j)$ for some $j \in
      \{1, \ldots, n\}$, or $R$ is not $\leq_\T$-minimal, with
      $\run(k^{\ng}_{\SOMET{R}}) = (a_k, q^{\ng}_{\SOMET{R}})$,
    \end{itemize}

  \item a child $k^*_{P,a_k,a_j}$ for each $a_k,a_j \in \Ind(\A)$ and each
    atomic role $P$ such that $(a_k,a_j) \in \Int[\Uni]{P}$, with
    $\run(k^*_{P,a_k,a_j}) = (\epsilon, q^*_{P_{kj}})$,
  \item a child $k^*_{\neg P,a_k,a_j}$ for each $a_k,a_j \in \Ind(\A)$ and each
    atomic role $P$ such that $(a_k,a_j) \not\in \Int[\Uni]{P}$, with
    $\run(k^*_{\neg P,a_k,a_j}) = (\epsilon, q^*_{\neg P_{kj}})$,
  \end{itemize}
  Note that nodes $y \in T_\run$ with $\run(y) = (x, q^*_{\dots})$ are leafs of
  the tree $T_\run$, as by the transition function $\delta_{\cn}$, all the
  states of the form $q^*_{\dots}$ in $Q_{\cn}$ can be satisfied with the empty
  assignment.

  Other nodes, however, can have children. They are defined inductively as
  follows.
  \begin{enumerate}
    \setcounter{enumi}{1}
  \item Let $y$ be a node in $T_\run$ such that $\run(y)=(x,q_s)$ for some $x
    \in n^*$. Moreover, let $i \in \{1, \ldots, n\}$. Then $y$ has
    \begin{itemize}
    \item a child $y \cdot i_{s}$ with $\run(y \cdot i_{s}) = (x \cdot i,
      q_{s})$,
    \item a child $y \cdot i^*_{\neg r}$ with $\run(y \cdot i^*_{\neg r}) = (x
      \cdot i, q^*_{\neg r})$,
    \item a child $y \cdot i^*_{\neg a_j}$ for each $j \in \{1, \ldots, n\}$
      with $\run(y \cdot i^*_{\neg a_j}) = (x \cdot i, q^*_{\neg a_j})$,
    \item if $x \in \dom[\Uni]$ and for $R \in \RR$ s.t.\ $f(R) = i$, $x \cdot
      w_{[R]} \in \dom[\Uni]$,
      \begin{itemize}
      \item a child $y \cdot i^*_{R}$ with $\run(y \cdot i^*_{R}) = (x \cdot
        w_{[R]}, q^*_{R})$,
      \end{itemize}

    \item otherwise
      \begin{itemize}
      \item a child $y \cdot i_{d}$ with $\run(y \cdot i_{d}) = (x \cdot i, q_{d})$,
      \end{itemize}
    \end{itemize}

  \item Let $y$ be a node in $T_\run$ such that $\run(y)=(x,q_d)$ for some $x
    \in n^*$. Then $y$ has
    \begin{itemize}
    \item a child $y \cdot i_{d}$ for each $i \in \{1, \ldots, n\}$, with
      $\run(y \cdot i_{d}) = (x \cdot i, q_{d})$,
    \item a child $y \cdot 0^*_{\neg R}$ for each $R \in \RR$, with $\run(y
      \cdot 0^*_{\neg R}) = (x, q^*_{\neg R})$,
    \end{itemize}

  \item Let $y$ be a node in $T_\run$ such that $\run(y)=(x,q^{\ng}_{\SOMET{R}})$ 
    for some $x \in \dom[\Uni]$ and $R \in \RR$. Then $y$ has
    \begin{itemize}
    \item a child $y \cdot f(R)^*_{\neg R'}$ for each $R' \in \RR$, with $\run(y
      \cdot f(R)^*_{\neg R'}) = (x \cdot f(R), q^*_{\neg R'})$,
    \end{itemize}

  \item Let $y$ be a node in $T_\run$ such that $\run(y)=(x,q_{\SOMET{R}})$ for
    some $x \in \dom[\Uni]$ and $R \in \RR$. Then $x \cdot w_{[R]} \in
    \dom[\Uni]$ and $y$ has
    \begin{itemize}
    \item a child $y \cdot f(R)_{R}$ with $\run(y \cdot f(R)_{R}) = (x \cdot w_{[R]},
      q_{R})$,
    \end{itemize}

  \item Let $y$ be a node in $T_\run$ such that $\run(y)=(x,q_{R})$ for some $x
    \in \dom[\Uni]$ and $R \in \RR$. Then $y$ has
    \begin{itemize}
    \item a child $y \cdot 0^*_{R'}$ for each $R' \in \RR$ s.t.\ $\K \models R
      \ISA R'$, with $\run(y \cdot 0^*_{R'}) = (x, q^*_{R'})$,
    \item a child $y \cdot 0^*_{\neg R'}$ for each $R' \in \RR$ s.t.\ $\K
      \not\models R \ISA R'$, with $\run(y \cdot 0^*_{\neg R'}) = (x, q^*_{\neg
      R'})$,
    \item a child $y \cdot 0^*_{B}$ for each $B \in \BB$ s.t.\ $\K \models
      \SOMET{R^-} \ISA B$, with $\run(y \cdot 0^*_{B}) = (x, q^*_B)$,
    \item a child $y \cdot 0^*_{\neg B}$ for each $B \in \BB$ s.t.\ $\K
      \not\models \SOMET{R^-} \ISA B$, with $\run(y \cdot 0^*_{\neg B}) = (x,
      q^*_{\neg B})$,
    \item a child $y \cdot 0_{\SOMET{S}}$ for each $\leq_\T$-minimal role $S$
      s.t.\ $\K \models \SOMET{R^-} \ISA \SOMET{S}$ and $[R^-] \neq [S]$, with
      $\run(y \cdot 0_{\SOMET{S}}) = (x, q_{\SOMET{S}})$,
    \item a child $y \cdot 0^{\ng}_{\SOMET{S}}$ for each role $S$ s.t.\ $\K
      \not\models \SOMET{R^-} \ISA \SOMET{S}$, or $[R^-] = [S]$, or $S$ is not
      $\leq_\T$-minimal, with $\run(y \cdot 0^{\ng}_{\SOMET{S}}) = (x,
      q^{\ng}_{\SOMET{S}})$.
    \end{itemize}
  \end{enumerate}
  Each node of $T_\run$ defined as described above satisfies the transition
  function $\delta_{\cn}$.

  It is easy to see that this run is accepting, as for each infinite path $P$ of
  $T_\run$, either $q_s \in \mathit{inf}(P)$, or $q_s \in \mathit{inf}(P)$, or
  $q_R \in \mathit{inf}(P)$ for some $R$. Hence, $T_\Uni \in \L(\acan)$.

  \vspace{0.3cm} To show the second item, let $(T,V) \in \L(\acan)$ and
  $(T_\run,\run)$ an accepting run of $(T,V)$. First, assume $T$ is not
  individual unique, that is,
  \begin{itemize}
  \item there exists an $a$-node $x$ in $T$, such that $x$ is not a child of the
    root, or
  \item there exist two nodes $i$ and $j$ in $T$ such that $a \in V(i)$ and $a
    \in V(j)$.
  \end{itemize}
  In the former case, let $x'$ be the parent of $x$, $x' \neq \epsilon$, then
  there exists a node $y' \in T_\run$ with $\run(y') = (x',q_s)$ and a node $y
  \in \T_\run$ with $\run(y) = (x,q^\star_{\NOT a})$, which contradicts that
  $(\T_\run,\run)$ is an accepting run of $(T,V)$ as $a \in V(x)$. In the
  latter case, assume $a$ is equal to $a_i$. Then we get contradiction with
  $\delta_{\cn}(q_0,\sigma)$.

  Hence, $T$ is individual unique. Let $\I_T$ be the interpretation represented
  by $T$. We show that $\I_T$ is isomorphic to $\Uni$, by constructing a
  function $h$ from $\dom[\I_T]$ to $\dom[\Uni]$ and showing that it is a
  one-to-one and onto homomorphism. We construct $h$ by induction on the length
  of the sequence $x \in \dom[\I_T]$.

  Initially, as $T$ is individual unique, we set for each $i \in \{1, \ldots,
  n\}$, $h(i) = a_i$, where $a_i \in V(i)$. Note that by definition of $\Uni$,
  $a_i \in \dom[\Uni]$ and by definition of $\I_T$, $i \in \dom[\I_T]$.  Then
  the following holds for $i,j \in \{1, \ldots, n\}$.
  \begin{enumerate}
  \item for an atomic role $P$, $(i,j) \in \Int[\I_T]{P}$ iff $(a_i, a_j) \in
    \Int[\Uni]{P}$: let $(i,j) \in \Int[\I_T]{P}$, by definition of $\I_T$ it
    follows that $P_{ij} \in V(\epsilon)$. Assume $\K \not\models P(a_i,a_j)$,
    then $(0, q^*_{\neg P_{ij}}) \in \delta_{\cn}(q_0, V(\epsilon))$ and in
    $T_\run$ there exists a node $y$, s.t.\ $\run(y) = (\epsilon, q^*_{\neg
    P_{ij}})$, hence $y$ does not satisfy the condition on a
    run. Contradiction with $(T_\run,\run)$ being accepting. Therefore, indeed
    $\K \models P(a_i, a_j)$ and $(a_i, a_j) \in \Int[\Uni]{P}$. Similarly for
    the other direction.

  \item for a basic concept $B$, $i \in \Int[\I_T]{B}$ iff $a_i \in
    \Int[\Uni]{B}$: let $i \in \Int[\I_T]{B}$, by definition of $\I_T$ it
    follows that $B \in V(i)$.  Assume $\K \not\models B(a_i)$, then $(i,
    q^*_{\neg B}) \in \delta_{\cn}(q_0, V(\epsilon))$ and there exists $y \in
    T_\run$ with $\run(y) = (i, q^*_{\neg B})$. We get contradiction as $y$ does
    not satisfy the condition on a run. Therefore, indeed $\K \models B(a_i)$
    and $a_i \in \Int[\Uni]{B}$. Similarly for the other direction.
  \end{enumerate}

  For the inductive step we prove two auxiliary claims.

  \begin{claim}[1]
    Let $i \cdot f(R) \in \dom[\I_T]$ for some $i \in \{1, \ldots, n\}$. Then
    $\K \models \SOMET{R}(a_i)$, $\K \not\models R(a_i,a_j)$ for each $j \in
    \{1, \ldots, n\}$ and $R$ is a $\leq_\T$-minimal such role.
  \end{claim}

  \begin{proof} 
    Assume $\K \not\models \SOMET{R}(a_i)$, or $\K \models R(a_i,a_j)$ for some
    $j \in \{1, \ldots, n\}$, or $R$ is not a $\leq_\T$-minimal such role. Then
    by definition of $\delta_{\cn}(q_0,V(\epsilon))$ and of a run, there exists
    a node $y = \epsilon \cdot i^{\ng}_{\SOMET{R}}$ in $T_\run$ such that
    $\run(y) = (i, q^{\ng}_{\SOMET{R}})$ and by
    $\delta_{\cn}(q^{ng}_{\SOMET{R}}, V(i))$ it is required that $R' \notin V(x
    \cdot f(R))$ for each $R' \in \RR$. It means that $i \cdot f(R)$ is not
    connected to $i$ through any role. Contradiction with $i \cdot f(R)$ being
    in $\dom[\I_T]$.
    \qed
  \end{proof}

  \begin{claim}[2]
    Let $x \cdot f(R) \in \dom[\I_T]$, $\len(x) \geq 2$ and there exists $y \in
    T_\run$ with $\run(y) = (x, q_S)$. Then $\K \models \SOMET{S^-} \ISA
    \SOMET{R}$, $[S^-] \neq [R]$ and $R$ is a $\leq_\T$-minimal such role.
  \end{claim}

  \begin{proof} 
    For the sake of contradiction assume $\K \not\models \SOMET{S^-} \ISA
    \SOMET{R}$. Then by definition of $\delta_{\cn}(q_S,V(x))$ and of a run,
    there exists a node $y'' = y \cdot 0^{\ng}_{\SOMET{R}}$ in $T_\run$ such
    that $\run(y'') = (x, q^{\ng}_{\SOMET{R}})$ and by
    $\delta_{\cn}(q^{ng}_{\SOMET{R}}, V(x))$ it is required that $R' \notin V(x
    \cdot f(R))$ for each $R' \in \RR$. It means that $x \cdot f(R)$ is not
    connected to $x$ through any role. Contradiction with $x \cdot f(R)$ being
    in $\dom[\I_T]$.  

    By the same argument it can be shown that $[S^-] \neq [R]$ and $R$ is
    $\leq_\T$-minimal.
    %
    \qed
  \end{proof}

  Let $x \in \dom[\I_T]$, $h(x)$ is defined and $h(x) \in \dom[\Uni]$. Moreover,
  if $\len(x) \geq 2$, let $\tail(x) = f(S)$ and $\tail(h(x)) = w_{[S]}$ for
  some role $S$, and there exist a node $y \in T_\run$ such that $\run(y) = (x,
  q_S)$. Then
  \begin{enumerate}
  \item for each $h(x) \cdot w_{[R]} \in \dom[\Uni]$, $x \cdot f(R)$ is in $\dom[\I_T]$.
  \item for each $x \cdot f(R) \in \dom[\I_T]$, $h(x) \cdot w_{[R]}$ is in $\dom[\Uni]$.
  \end{enumerate}

  Let $h(x) \cdot w_{[R]} \in \dom[\Uni]$. Then $R$ is $\leq_\T$-minimal such
  that $\K \models \SOMET{S^-} \ISA \SOMET{R}$ if $\tail(h(x)) = w_{[S]}$, or
  $\K \models \SOMET{R}(a_i)$ and $\K \not\models R(a_i,a_j)$ for $j \in \{1,
  \ldots, n\}$ if $h(x)=a_i$. By definition of $\delta_{\cn}$, there exist a
  node $y'$ in $T_\run$ with $\run(y') = (x, q_{\SOMET{R}})$. Since $T_\run$ is
  a run, it follows that there exist a node $y'' = y' \cdot f(R)_R$ in $T_\run$
  with $\run(y'') = (x \cdot f(R), q_R)$, and $x \cdot f(R) \in T$. Therefore,
  $R \in V(x \cdot f(R))$ and by definition of $\I_T$, $x \cdot f(R) \in
  \dom[\I_T]$.

  Let $x \cdot f(R) \in \dom[\I_T]$. Then by Claim~(1) and (2), $\tail(h(x))
  \leadsto_\K w_{[R]}$, hence $h(x) \cdot w_{[R]} \in \dom[\Uni]$. Moreover, we
  also obtain that there exists $y''$ in $T_\run$ such that $\run(y'') = (x
  \cdot f(R), q_{R})$.

  Thus, we can set $h(x \cdot f(R))$ to $h(x) \cdot w_{[R]}$. Obviously, $h$ is
  one-to-one and onto. To verify that $h$ is a homomorphism it remains to show
  \begin{itemize}
  \item for each role $R'$, $(x,x \cdot f(R)) \in \Int[\I_T]{R'}$ iff $(h(x),
    h(x) \cdot w_{[R]}) \in \Int[\Uni]{R'}$, 
    and
  \item for each basic concept $B$, $x \cdot f(R) \in \Int[\I_T]{B}$ iff $h(x)
    \cdot w_{[R]} \in \Int[\Uni]{B}$
    .
  \end{itemize}

  Let $(x, x \cdot f(R)) \in \Int[\I_T]{R'}$ for some role $R'$. By
  contradiction assume $(h(x), h(x) \cdot w_{[R]}) \notin \Int[\Uni]{R'}$, this
  implies that $\K \not\models R \ISA R'$. Hence, $(0,q^*_{\neg R'}) \in
  \delta_{\cn}(q_R, V(x \cdot f(R)))$, and in $T_\run$ there is a node $y'''=y''
  \cdot 0^*_{\neg R'}$ with $\run(y''') = (x \cdot f(R), q^*_{\neg R'})$. We get
  a contradiction with $T_\run$ being a run as by definition of $\I_T$, $R' \in
  V(i \cdot f(R))$. Similarly for the other direction.

  Finally, let $x \cdot f(R) \in \Int[\I_T]{A}$ for some concept $A$, and assume
  $h(x) \cdot w_{[R]} \notin \Int[\Uni]{A}$. The latter implies that $\K
  \not\models \SOMET{R^-} \ISA A$. Hence, $(0,q^*_{\neg A}) \in
  \delta_{\cn}(q_R, V(x \cdot f(R)))$, and in $T_\run$ there is a node $y'''=y''
  \cdot 0^*_{\neg A}$ with $\run(y''') = (x \cdot f(R), q^*_{\neg A})$. We get a
  contradiction with $T_\run$ being a run as by definition of $\I_T$, $A \in V(x
  \cdot f(R))$. Similarly for the other direction.
\end{proof}

\subsubsection{Automaton $\asol$ for a model of $\K = \tup{\T, \A}$}

$\asol$ is a 2ATA on infinite trees that accepts a tree if its subtree labeled
with $G$ corresponds to a tree model $\I$ of $\K$. Formally, $\asol$ is defined
as the tuple $\tup{\Sigma_\K, Q_{\sol}, \delta_{\sol}, q_0 , F_{\sol}}$, where
\begin{eqnarray*}
  Q_{\sol} & = & \{q_0\} \cup \{q_{X} \mid X \in \NN \cup \BB \cup \RR \cup \PP\},
\end{eqnarray*}
$F_{\sol} = Q_{\sol}$ and transition function $\delta_{\sol}: Q_{\sol} \times
\Sigma_\K \rightarrow \B([n] \times Q_{\sol})$ is defined as follows:
\begin{enumerate}
\item For each $\sigma \in \Sigma_\K$ such that $\{r,G\} \subseteq \sigma$,
  $\delta_{\sol}(q_0,\sigma)$ is defined as:
  \begin{eqnarray*}
    \bigwedge_{i=1}^n \bigg[  
    (i, q_{a_i}) \wedge \bigg(\bigwedge_{A \in \BB \,:\,  \K \models A(a_i)}
    (i,q_{A})\bigg) \wedge \bigwedge_{j = 1}^n  \bigg(\bigwedge_{P \in \RR \,:\,
      \K \models P(a_i,a_j)}  (0, q_{P_{ij}})\bigg) \bigg] 
  \end{eqnarray*}

  
  
  
\item For each $\sigma \in \Sigma_\K$ such that $\{r,G\} \subseteq \sigma$ and
  each $P_{ij} \in \PP$:
  \begin{eqnarray*}
    \delta_{\sol}(q_{P_{ij}},\sigma) &=& (i,q_{\SOMET{P}}) \wedge (j,q_{\SOMET{P^-}})
  \end{eqnarray*}  
  
\item For each $\sigma \in \Sigma_\K$ such that $\sigma \cap \NN = \{a_i\}$
  and each atomic role $P$ in the signature of $\K$:
  \begin{eqnarray*}\delta_{\sol}(q_{\SOMET{P}}, \sigma) &=& 
    {\displaystyle \bigg(\bigvee_{j=1}^n (j,q_{P})\bigg) \vee
      \bigg(\bigvee_{j=1}^n (-1,q_{P_{ij}})\bigg)}
    \\
    \delta_{\sol}(q_{\SOMET{P^-}}, \sigma) &=& 
    {\displaystyle \bigg(\bigvee_{j=1}^n (j,q_{P^-})\bigg) \vee
      \bigg(\bigvee_{j=1}^n (-1,q_{P_{ji}})\bigg)}
  \end{eqnarray*}
  
\item For each $\sigma \in \Sigma_\K$ such that $\sigma \cap \NN = \emptyset$
  and each basic role $R \in \RR$, 
  %
  \begin{eqnarray*}
    \delta_{\sol}(q_{\SOMET{R}}, \sigma) = 
    (0, q_{R^-}) \vee \bigg(\bigvee_{i=1}^n (i,q_{R}) \bigg)
  \end{eqnarray*}     
  
\item For each $\sigma \in \Sigma_\K$ such that $\sigma \cap \NN = \emptyset$
  and each basic role $R \in \RR$:
  \begin{equation*}\delta_{\sol}(q_{R}, \sigma) =
    {\displaystyle 
      \bigg(\bigwedge_{R' \in \RR \,:\, \K \models R \ISA
        R'} (0, q_{R'})\bigg) \wedge (0,q_{\SOMET{R^-}}) \wedge (-1,q_{\SOMET{R}})
    }
  \end{equation*}

\item For each $\sigma \in \Sigma_\K$ and each $B \in \BB$:
  \begin{equation*}\delta_{\sol}(q_{B}, \sigma) =
    {\displaystyle 
      \bigwedge_{B' \in \BB \,:\, \K \models B \ISA B'} (0, q_{B'})
    }
  \end{equation*}
  
\item For each $\sigma \in \Sigma_\K$ and each $X \in \BB \cup \RR \cup \NN
  \cup \PP$:
  \begin{align*} \delta_{\sol}(q_{X}, \sigma) =
    \begin{cases}
      \text{\it true} & \text{if } G \in \sigma \text{ and } X \in \sigma\\
      \text{\it false} & \text{otherwise}
    \end{cases} 
  \end{align*}

  
\end{enumerate}

If there are several entries of $\delta_{\sol}$ for the same $q \in Q_{\sol}$
and $\sigma \in \Sigma_{\sol}$, $\delta_{\sol}(q, \sigma) = \phi_1$, \ldots,
$\delta_{\sol}(q, \sigma) = \phi_m$, then we assume that $\delta_{\sol}(q,
\sigma) = \bigwedge_{i=1}^m \phi_i$.

Given a model $\I$, a \emph{path} $\pi$ from $x$ to $x'$, $x,x' \in \dom$, is a
sequence of the form $(x=x_1, x_2, \ldots, x_{m},x_{m+1}=x')$, $m \geq 0$, such
that $x_i \in \dom$ and $(x_i, x_{i+1}) \in \Int{R_i}$ for some $R_i$, and $m$
is the length of $\pi$. A model $\I$ of $\K=\tup{\T,\A}$ is said to be a
\emph{tree model} if for each $x \in \dom \setminus \Ind(\A)$ there exists a
unique shortest path from $x$ to $\Ind(\A)$. The \emph{depth} of an object $x$
in a tree model $\I$, denoted $\dep(x)$, is the length of the shortest path from
$x$ to $\Ind(\A)$. It is said that $x'$ is a successor of $x$, $x' \in \succ(x)$
if $x$ belongs to the path from $x'$ to $\Ind(\A)$ and $\dep(x') = \dep(x) + 1$.

Note that given a tree-model $\I$ of $\K$ with branching degree $n$, each domain
element of $\I$ can be seen as an element of $n^*$. For $x' \in \dom$ with
$\dep(x') = m \geq 0$, we assume a one-to-one numbering $g_{m,x'}(x)$ of each $x
\in \succ(x')$, such that $1 \leq g_{m,x'}(x) \leq n$. Then $x \in \dom$
corresponds to
\begin{itemize}
\item $i$ if $x = a_i$, 
\item $x' \cdot i$, where $\dep(x')=m \geq 0$, $x \in \succ(x')$ and $g_{m,x'}(x)=i$.
\end{itemize}
Then, $i \cdot -1$ denotes the empty sequence $\epsilon$. Conversely, each
sequence of natural numbers $x \in n^*$ can be seen as an element of $\dom$.

The \emph{$G$-tree encoding} of a tree-model $\I$ of $\K$ with branching degree
$n$ is the $\Sigma_\K$-labeled tree $T_{\I,G} = (n^*, \Int[\I,G]{V})$, such that
\begin{itemize}
\item $\Int[\I,G]{V}(\epsilon) = \{r,G\} \cup \{P_{ij} \mid (a_i,a_j) \in
  \Int[\I]{P}, P\text{ is an atomic role}\}$,
\item for each $x \in \dom[\I]$:
  \[\begin{array}[t]{rcl}
    \Int[\I,G]{V}(x) &= &\{G\} \cup \{B \mid x \in \Int[\Uni]{B}\} \cup {}\\
    & &\{S \mid (x',x) \in \Int[\Uni]{S}\text{ and }\dep(x) > \dep(x')\} \cup {}\\
    & &\{a \mid a \in \Ind(\A)\text{ and }x=a\}.
  \end{array}
  \] 
\end{itemize}

Given a labeled tree $(T,V)$, the \emph{restriction} of $T$ on $G$ is a set
$T_G$ such that $T_G \subseteq T$ and for each $x \in T$: $x \in T_G$ iff $G \in
V(x)$. 

Given a labeled tree $(T,V)$ and a run $(T_\run, \run)$, the
\emph{interpretation represented} by $T$ and $T_\run$, denoted, $\I_{T,T_\run}$,
is defined similarly to $\I_T$:
\[
\begin{array}{ccll}
  \dom[\I_{T,T_\run}] & = & \dom[\I_T],\\
  \Int[\I_{T,T_\run}]{a_i} & = & \Int[\I_T]{a_i},\\
  \Int[\I_{T,T_\run}]{A} & = & \dom[\I_{T}] \cap \{x \mid A \in
  V(x)\text{ and there exists }y\in T_\run\text{ with } \run(y)=(x, q_A)\}, \\
  && \text{for each atomic concept }A \in \BB\text{ and}\\
  \Int[\I_{T,T_\run}]{P} & = & (\dom[\I_{T}] \times \dom[\I_{T}]) \cap {}\\
  && \{(x,x') \in R_P \mid \text{ there exists }y\in T_\run\text{ s.t.\ }
  \run(y)=(x', q_P) \text{ or }\run(y) = (x, q_{P^-})\}, \\ 
  && \text{for each atomic role }P \in \RR.\\
\end{array}
\]

\begin{proposition}
  The following hold for $\asol$:
  \begin{itemize}
  \item Let $\I$ be a tree model of $\K$ with branching degree $n$. Then
    $T_{\I,G} \in \L(\asol)$.
  \item for each $(T, V) \in \L(\asol)$, if $T_G$ is an individual unique tree
    and $(T_\run, \run)$ is a corresponding run, then $\I_{T_G,T_\run}$ is a
    model of $\K$.
  \end{itemize}
\end{proposition}

\begin{proof}
  For the first item, assume $T_{\I,G} = (n^*, \Int[\I,G]{V})$ is the tree encoding of a
  model $\I$ of $\K$.  We show that a full run of $\asol$ over $T_{\I,G}$
  exists.

  The run $(T_\run,\run)$ is built starting from the root $\epsilon$, and
  setting $\run(\epsilon)$ = $(\epsilon,q_0)$. Then, to correctly execute the
  initial transition, the root has children as follows:

  \begin{itemize}
  \item for each $a_k \in \Ind(\A)$
    \begin{itemize}
    \item a child $k_{a_k}$ with $\run(k_{a_k}) = (a_k, q_{a_k})$,
    \item a child $k_{A}$ for each $A \in \BB$ such that $a_k \in \Int[\I]{A}$,
      with $\run(k_{A}) = (a_k, q_A)$,
    \end{itemize}
  \item a child $k_{P,a_k,a_j}$ for each $a_k,a_j \in \Ind(\A)$ and each atomic
    role $P$ such that $(a_k,a_j) \in \Int[\I]{P}$, with $\run(k_{P,a_k,a_j}) =
    (\epsilon, q_{P_{kj}})$,
  \end{itemize}
  Then the successor relationship in $T_\run$ is defined inductively as follows.
  \begin{enumerate}
    \setcounter{enumi}{1}
  \item Let $y$ be a node in $T_\run$ such that $\run(y)=(x,q_{P_{ij}})$ for $x
    = \epsilon$ and $P \in \RR$. Then $y$ has
    \begin{itemize}
    \item a child $y \cdot i_{\SOMET{P}}$ with $\run(y \cdot i_{\SOMET{P}}) = (x \cdot i,
      q_{\SOMET{P}})$,
    \item a child $y \cdot j_{\SOMET{P^-}}$ with $\run(y \cdot j_{\SOMET{P^-}}) = (x \cdot j,
      q_{\SOMET{P^-}})$,
    \end{itemize}

  \item Let $y$ be a node in $T_\run$ such that $\run(y)=(x,q_{\SOMET{R}})$ for
    some $x \in \dom[\I]$, $\Int[\I,G]{V}(x) \cap \NN=\{a_i\}$, $R \in \RR$, and
    $R_{ij}$ denotes $P_{ij}$ if $R=P$ and $P_{ji}$ if $R=P^-$ for some atomic
    role $P$. Then $y$ has
    \begin{itemize}
    \item  if $R \in \Int[\I,G]{V}(x \cdot j)$ for some $j$
      \begin{itemize}
      \item a child $y \cdot j_{R}$ with $\run(y \cdot j_{R}) = (x \cdot j, q_{R})$,
      \end{itemize}
    \item if $R_{ij} \in \Int[\I,G]{V}(x \cdot -1)$
      \begin{itemize}
      \item a child $y \cdot -1_{R_{ij}}$ with $\run(y \cdot -1_{R_{ij}}) = (x
        \cdot -1, q_{R_{ij}})$,
      \end{itemize}

    \end{itemize}

  \item Let $y$ be a node in $T_\run$ such that $\run(y)=(x,q_{\SOMET{R}})$ for
    some $x \in \dom[\I]$, $\Int[\I,G]{V}(x) \cap \NN=\emptyset$ and $R \in
    \RR$. Then $y$ has
    \begin{itemize}
    \item if $R \in \Int[\I,G]{V}(x \cdot i)$ for some $i$
      \begin{itemize}
      \item a child $y \cdot i_{R}$ with $\run(y \cdot i_{R}) = (x \cdot i, q_{R})$,
      \end{itemize}
    \item if $R^- \in \Int[\I,G]{V}(x)$
      \begin{itemize}
      \item a child $y \cdot 0_{R^-}$ with $\run(y \cdot 0_{R}) = (x, q_{R})$,
      \end{itemize}

    \end{itemize}

  \item Let $y$ be a node in $T_\run$ such that $\run(y)=(x,q_{R})$ for some $x
    \in \dom[\I]$ and $R \in \RR$. Then $y$ has
    \begin{itemize}
    \item a child $y \cdot 0_{R'}$ for each $R' \in \RR$ s.t.\ $\K \models R
      \ISA R'$, with $\run(y \cdot 0_{R'}) = (x, q_{R'})$,
    \item a child $y \cdot 0_{\SOMET{R^-}}$, with $\run(y \cdot 0_{\SOMET{R^-}})
      = (x, q_{\SOMET{R^-}})$,
    \item a child $y \cdot -1_{\SOMET{R}}$, with $\run(y \cdot -1_{\SOMET{R}})
      = (x, q_{\SOMET{R}})$,
    \end{itemize}

  \item Let $y$ be a node in $T_\run$ such that $\run(y)=(x,q_{B})$ for some $x
    \in \dom[\I]$ and $B \in \BB$. Then $y$ has
    \begin{itemize}
    \item a child $y \cdot 0_{B'}$ for each $B' \in \BB$ s.t.\ $\K \models B
      \ISA B'$, with $\run(y \cdot 0_{B'}) = (x, q_{B'})$,
    \end{itemize}
  \end{enumerate}
  Since $\I$ is a model of $\K$, $T_\run$ satisfies the transition function
  $\delta_{\sol}$. In particular, in the rules 3 and 4 in the inductive
  definition of $T_\run$, there will exists a node $x' \in \dom$ such that
  $(x,x') \in \Int{R}$, hence at least one of the conditions will be satisfied.

  It is easy to see that this run is accepting, as for each infinite path $P$ of
  $T_\run$, $q_R \in \mathit{inf}(P)$ for some $R$. Hence, $T_{\I,G} \in
  \L(\asol)$.

  \vspace{0.3cm} To show the second item, let $(T,V) \in \L(\acan)$ and
  $(T_\run,\run)$ an accepting run of $(T,V)$. Moreover, let $T_G$ be a tree
  (i.e., prefix closed) and individual unique. Then $\I_{T_G,T_\run}$ is defined
  and it can be shown that $\I_{T_G,T_\run}$ a model of $\K$:
  \begin{enumerate}
  \item for each $i \in \{1, \ldots, n\}$, $\K \models B(a_i)$ implies $a_i \in
    \Int[\I_{T_G,T_\run}]{B}$,
  \item for each $i,j \in \{1, \ldots, n\}$, $\K \models P(a_i,a_j)$ implies
    $(a_i,a_j) \in \Int[\I_{T_G,T_\run}]{P}$,
  \item if $x \in \Int[\I_{T_G,T_\run}]{B}$, then $x \in
    \Int[\I_{T_G,T_\run}]{B'}$ for each $B'$ s.t.\ $\K \models B \ISA B'$,
  \item if $(x,x') \in \Int[\I_{T_G,T_\run}]{R}$, then $(x,x') \in
    \Int[\I_{T_G,T_\run}]{R'}$ for each $R'$ s.t.\ $\K \models R \ISA R'$,
  \item if $x \in \Int[\I_{T_G,T_\run}]{B}$ and $\K \models B \ISA \SOMET{R}$,
    then there exists $x' \in T_G$ such that $(x,x') \in
    \Int[\I_{T_G,T_\run}]{R}$.
  \item if $(x,x') \in \Int[\I_{T_G,T_\run}]{S}$ and $\K \models \SOMET{S^-}
    \ISA \SOMET{R}$, then there exists $x'' \in T_G$ such that $(x',x'') \in
    \Int[\I_{T_G,T_\run}]{R}$.
  \end{enumerate}

  We show items 5 and 6 hold, the rest can be shown by analogy. 

  Assume $x \in \Int[\I_{T_G,T_\run}]{B}$ and $\K \models B \ISA \SOMET{R}$ for
  some concept $B$ and role $R$. Then by definition of $\I_{T_G,T_\run}$ we have
  that $B,G \in V(x)$ and there exist a node $y \in T_\run$ with $r(y) = (x,
  q_B)$. Since $T_\run$ is a run and by definition of $\delta_{\sol}$, there
  exists a node $y' = y \cdot 0_{\SOMET{R}}$ in $T_\run$ such that $\run(y') =
  (x, q_{\SOMET{R}})$. Then, if $V(x) \cap \NN = \emptyset$, there exists a node
  $y'' = y' \cdot z$ in $T_\run$ such that $\run(y'') = (x \cdot i, q_R)$ or
  $\run(y'') = (x, q_{R^-})$. 
  If $V(x) \cap \NN = a_i$, there exists a node $y'' = y' \cdot z$ in $\T_\run$
  such that $\run(y'') = (x \cdot j, q_R)$ or $\run(y'') = (\epsilon,
  q_{R_{ij}})$. 
  In any case, it is easy to see that
  there is $x' \in T$ with $G \in V(x')$ (i.e., $x' \in T_G$) such that $(x,x')
  \in \Int[\I_{T_G,T_\run}]{R}$.
  
  Assume that $(x,x') \in \Int[\I_{T_G,T_\run}]{S}$ and $x'$ is a successor of
  $x$. Then $S, G \in V(x')$ and there exists $y \in T_\run$ such that $\run(y)
  = (x', q_S)$. Since $T_\run$ is a run, there exists a node $y' = y \cdot
  0_{\SOMET{S^-}}$ such that $\run(y') = (x', q_{\SOMET{S^-}})$. Further, as $\K
  \models \SOMET{S^-} \ISA \SOMET{R}$, there exists $y'' = y' \cdot
  0_{\SOMET{R}}$ with $\run(y'') = (x', q_{\SOMET{R}})$ and as above we obtain
  that there is $x'' \in T$ with $G \in V(x'')$ (i.e., $x'' \in T_G$) such that
  $(x',x'') \in \Int[\I_{T_G,T_\run}]{R}$.

  Assume now that $(x,x') \in \Int[\I_{T_G,T_\run}]{S}$ and $x$ is a successor
  of $x'$. Then $S^-, G \in V(x)$ and there exists $y \in T_\run$ such that
  $\run(y) = (x, q_{S^-})$. Since $T_\run$ is a run, there exists a node $y' = y
  \cdot -1_{\SOMET{S^-}}$ such that $\run(y') = (x', q_{\SOMET{S^-}})$ (recall
  that $x' = x\cdot -1$). Further, as $\K \models \SOMET{S^-} \ISA \SOMET{R}$,
  there exists $y'' = y' \cdot 0_{\SOMET{R}}$ with $\run(y'') = (x',
  q_{\SOMET{R}})$ and as above we obtain that there is $x'' \in T$ with $G \in
  V(x'')$ (i.e., $x'' \in T_G$) such that $(x',x'') \in
  \Int[\I_{T_G,T_\run}]{R}$.

  Thus, $\I_{T_G,T_\run}$ is a model of $\K$.
\end{proof}

\subsubsection{Automaton $\agfin$}

$\agfin$ is a one-way non-deterministic automaton on infinite trees that accepts
a tree if it has a finite prefix where each node is marked with the special
symbol $G$, and no other node in the tree is marked with $G$. Formally, $\agfin
= \tup{\Gamma_\K, Q_{\fn}, \delta_{\fn}, q_0 , F_{\fn}}$, where $Q_{\fn} =
\{q_0, q_1\}$, $F_{\fn} = \{q_1\}$ and transition function $\delta_{\fn}:
Q_{\fn} \times \Gamma_\K \rightarrow \B([n] \times Q_{\fn})$ is defined as
follows: 

\begin{enumerate}
\item For each $\sigma \in \Gamma_\K$:
  \begin{equation*} \delta(q_0, \sigma) = 
    \begin{cases}
      {\displaystyle \bigwedge_{i=1}^n (i, q_0)}, & \text{if } G \in \sigma\\
      {\displaystyle \bigwedge_{i=1}^n (i, q_1)}, & \text{if } G \notin \sigma
    \end{cases}
  \end{equation*}
\item For each $\sigma \in \Gamma_\K$:
  \begin{equation*} \delta(q_1, \sigma) = 
    \begin{cases}
      {\displaystyle \bigwedge_{i=1}^n (i, q_1)}, & \text{if } G \notin \sigma\\
      \text{{\it false}} & \text{if } G \in \sigma
    \end{cases}
  \end{equation*}
\end{enumerate}

\subsection{Proof of Theorem~\ref{the:non-emp-uni-sol-owlql-aboxes-np}}
\begin{proof}
  We prove that the non-emptiness problem for universal solutions is in
  \NP. Assume we are given a mapping $\M=(\Sigma_1, \Sigma_2, \T_{12})$ and a
  source KB $\K_1 =\tup{\T_1, \A_1}$, and we want to decide whether there exists
  a universal solution for $\K_1$ under $\M$ (all ABoxes are considered to be
  \owlql ABoxes without inequalities).

  First, we check whether $\K_1$ and $\M$ are $\Sigma_2$-positive. This check
  can be done in polynomial time, and if it was successful, then by
  Lemma~\ref{lem:A2-uni-sol-iff-homo-equiv-and-sigma2-positive} it remains to
  verify whether there exists a universal solution for $\K_1^{\pos}$ under
  $\M^{\pos}$.

  Second, we construct the maximal target \owlql ABox, a candidate for universal
  solution. Let $\A_2$ be the ABox over $\Sigma_2$ containing every membership
  assertion $\alpha$ of the form $B(a)$ or $R(a,b)$ such that $\tup{\T_1^{\pos}
    \cup \T_{12}^{\pos}, \A_1} \models \alpha$, $a,b \in \Ind(\A_1)$, $B$ is a
  basic concept and $R$ is a basic role. Then $\A_2$ is of polynomial size, and
  \begin{lemma}\label{lem:uni-sol-exists-iff-A2-sol}
    A universal solution for $\K_1^{\pos}$ under $\M^{pos}$ exists iff $\A_2$ is
    a solution for $\K_1^{\pos}$ under $\M^{\pos}$.
  \end{lemma}
  \begin{proof}
    ($\Rightarrow$) Assume a universal solution for $\K_1^{\pos}$ under
    $\M^{\pos}$ exists. As it follows from
    Lemma~\ref{lem:uni-sol-exists-iff-chase-fin-part}, there exists a universal
    solution $\A_3$ such that $\Uni_{\A_3} \subseteq \Uni_{\tup{\T_1 \cup
        \T_{12},\A_1}}$, hence $\A_3 \subseteq \A_2$. As $\A_3$ is a solution,
    there exists $\I$ such that $\I \models \K_1^{\pos}$ and $(\I,\Uni_{\A_3})
    \models \T_{12}^{\pos}$. It follows that for each model $\J$ of $\A_2$, $\J
    \supseteq \Uni_{A_2} \supseteq \Uni_{\A_3}$, and therefore $(\I,\J) \models
    \T_{12}^{\pos}$. By definition of solution, $\A_2$ is a solution.

    ($\Leftarrow$) Assume $\A_2$ is a solution for $\K_1^{\pos}$ under
    $\M^{\pos}$. Then $\A_2$ is a universal solution follows from the proof of
    Lemma~\ref{lem:A2-uni-sol-dlliter-pos-iff-homo-equiv}. Since $\A_2$ is an
    \owlql ABox, we conclude that a universal solution for $\K_1^{\pos}$ under
    $\M^{\pos}$ exists.
  \end{proof}
  
  Thus, it remains only to check whether $\A_2$ is a solution. We need the
  following result to perform this check in \NP.
  \begin{lemma}\label{lem:a2-sol-exists-I-poly-size}
    Let $\A_2$ be an (extended) ABox over $\Sigma_2$ such that it is a solution
    for $\K_1^{\pos}$ under $\M^{\pos}$. Then there exists an interpretation
    $\I$ such that $\I$ is of polynomial size, $\I$ is a model of $\K_1^{\pos}$
    and $(\I,\V_{\A_2}) \models \T_{12}^{\pos}$.
  \end{lemma}

  \begin{proof}
    Assume $\A_2$ is a solution for $\K_1^{\pos}$ under $\M^{\pos}$, then for
    each model of $\A_2$, in particular for $\V_{\A_2}$, there exists $\I'$ such
    that $\I'$ is a model of $\K_1^{\pos}$ and $(\I', \V_{\A_2}) \models
    \T_{12}^{\pos}$. Suppose $\card{\I'}$ is more than polynomial, then since
    $(\I', \V_{\A_2}) \models \T_{12}^{\pos}$ it follows $\Int[\I']{B} \subseteq
    \dom[\A_2]$ and $\Int[\I']{R} \subseteq \dom[\A_2]\times\dom[\A_2]$ for each
    basic concept $B$ and role $R$ that appear on the left hand side of some
    inclusion in $\T_{12}^{\pos}$. Therefore, we construct an interpretation
    $\I$ of polynomial size as follows:
    \begin{itemize}
    \item $\dom = \dom[\A_2] \cup N_a \cup \{d\}$, for a fresh domain element
      $d$,
    \item $\Int{a} = a$ for $a \in N_a$,
    \item $\Int{A} = (\Int[\I']{A} \cap \dom[\A_2]) \cup \{d \mid\text{ if
      }\Int[\I']{A} \setminus \dom[\A_2] \neq \emptyset\}$ for each atomic concept
      $A$,
    \item $
      \begin{array}[t]{rcl}
        \Int{R} &= &(\Int[\I']{R} \cap (\dom[\A_2] \times \dom[\A_2])) \ \cup {} \\
        && \{(a,d) \mid (a,b) \in \Int[\I']{R} \setminus (\dom[\A_2] \times
        \dom[\A_2]), a \in \Int[\I']{(\SOMET{R})} \cap \dom[\A_2]\}\ \cup {}\\ 
        && \{(d,a) \mid (b,a) \in \Int[\I']{R} \setminus (\dom[\A_2] \times
        \dom[\A_2]), a \in \Int[\I']{(\SOMET{R^-})} \cap \dom[\A_2]\}\ \cup {}\\ 
        && \{(d,d) \mid (a,b) \in \Int[\I']{R} \setminus (\dom[\A_2] \times
        \dom[\A_2]), a \notin \Int[\I']{(\SOMET{R})} \cap 
        \dom[\A_2], b \notin \Int[\I']{(\SOMET{R^-})} \cap \dom[\A_2]\}\\
      \end{array}$ \\
      for each atomic role $R$.
    \end{itemize}
    Note that $\V_{\A_2}$ interprets all constants as themselves, and $\I'$
    agrees on interpretation of constants with $\V_{\A_2}$, for this reason
    $\dom \supseteq N_a$.

    It is straightforward to verify that $\I$ is a model of $\K_1^{\pos}$:
    clearly, $\I$ is a model of $\A_1$, we show $\I \models
    \T_1^{\pos}$. Assume, $\T_1^{\pos} \models B \ISA C$ for basic concepts $B,
    C$, and $b \in \Int{B}$. If $b \in \dom[\I'] \cap \dom[\A_2]$, then since
    $\I' \models B \ISA C$, we have that $b \in \Int[\I']{C}$, which implies $b
    \in \Int{C}$. Otherwise, $b=d$ and for some $c \in \dom[\I'] \setminus
    \dom[\A_2]$, $c \in \Int[\I']{B}$, therefore $c \in \Int[\I']{C}$, and thus
    by definition of $\I$, $d \in \Int{C}$. Role inclusions are handled
    similarly. Moreover, as $\I$ and $\I'$ agree on all concepts and roles that
    appear on the left hand side of $\T_{12}^{\pos}$, it follows that
    $(\I,\V_{\A_2}) \models \T_{12}^{\pos}$. Hence, $\I$ is the interpretation
    of polynomial size we were looking for.
  \end{proof}

  Finally, the \NP algorithm for deciding the non-emptiness problem for
  universal solutions is as follows:
  \begin{enumerate}
  \item verify whether $\K_1$ and $\M$ are $\Sigma_2$-positive, if yes,
  \item compute $\A_2$, the $\Sigma_2$-closure of $\A_1$ with respect to
    $\T_1^{\pos} \cup \T_{12}^{\pos}$.
  \item guess a source interpretation $\I$ of polynomial size.
  \item If $\I \models \K_1^{\pos}$ and $(\I,\Uni_{\A_2}) \models \T_{12}^{\pos}$,
    then a universal solution for $\K_1$ under $\M$ exists, and $\A_2$ is a
    universal solution, otherwise a universal solution does not exist.
  \end{enumerate}
  Note that steps~1,2 and 4 can be done in polynomial time, hence this algorithm
  is in fact an \NP algorithm. Below we prove the correctness of the algorithm.

  Assume $\I \models \K_1^{\pos}$ and $(\I,\Uni_{\A_2}) \models
  \T_{12}^{\pos}$. Then $\A_2$ is a solution: for each model $\J$ of $\A_2$, it
  holds $\Uni_{\A_2} \subseteq \J$, therefore $(\I,\J) \models \T_{12}$. By
  Lemma~\ref{lem:uni-sol-exists-iff-A2-sol} we obtain that a universal solution
  for $\K_1$ under $\M$ exists, and from its proof it follows that $\A_2$ is a
  universal solution. Thus, the algorithm is sound.

  We show the algorithm is complete. Assume $\I \not\models \K_1^{\pos}$ or
  $(\I,\Uni_{\A_2}) \not\models \T_{12}^{\pos}$, and to the contrary, $\A_2$ is
  a solution. The by Lemma~\ref{lem:a2-sol-exists-I-poly-size}, there exists a
  model $\I'$ of $\K_1^{\pos}$ of polynomial size such that $(\I',\Uni_{\A_2})
  \models \T_{12}^{\pos}$. Contradiction with the guessing step. Therefore,
  $\A_2$ is not a solution and there exists no universal solution. Thus, the
  algorithm is complete.

  \vspace{0.3cm}
  As a corollary we obtain an upper bound for the membership problem.
  \begin{theorem}
    The membership problem for universal solutions is in \NP.
  \end{theorem}
\end{proof}

\subsection{Proof of Theorem~\ref{the:non-emp-uni-sol-pspace-hard-and-exptime}}
\begin{proof}
  First we provide the \PSPACE lower bound, and then present the \EXPTIME
  automata-based algorithm for deciding the non-emptiness problem for universal
  solutions with extended ABoxes.
  \begin{lemma}
    \label{lem:non-emp-uni-sol-pspace-hard}
    The non-emptiness problem for universal solutions with extended ABoxes in
    \dlliter is \PSPACE-hard.
  \end{lemma}
  \begin{proof}
    The proof is by reduction of the satisfiability problem for quantified
    Boolean formulas, known to be $\PSPACE$-complete. Suppose we are given a
    QBF \[\phi = \QQ_1 X_1 \ldots \QQ_n X_n \bigwedge_{j=1}^m C_j\] where $\QQ_i
    \in \{\forall,\exists\}$ and $C_j$, $1 \leq j \leq m$, are clauses over the
    variables $X_i$, $1 \leq i \leq n$.

    Let $\Sigma_1 = \{A, Y_i^k, X_i^k, S_l, T_l, Q_i^k, P_i^k, R_j, R_j^l \mid 1
    \leq j \leq m, 1 \leq i \leq n, 0 \leq l \leq n, k \in \{0,1\}\}$ where $A,
    Y_i^k, X_i^k$ are concept names and the rest are role names. Let $\T_1$ be
    the following TBox over $\Sigma_1$ for $1 \leq j \leq m$, $1 \leq i \leq n$
    and $k \in \{0,1\}$:
    \[\begin{array}{r@{~~}c@{~~}l@{\qquad}r@{~~}c@{~~}l@{\qquad}r@{~~}c@{~~}l}
      A & \ISA & \SOMET{S_0^-} &
      \SOMET{S_{i-1}^-} & \ISA & \SOMET{Q_i^k} & \multicolumn{3}{l}{\text{if } \QQ_i =
        \forall} \\
      &&&  \SOMET{S_{i-1}^-} & \ISA & \SOMET{S_i} & \multicolumn{3}{l}{\text{if } \QQ_i =
        \exists} \\
      \SOMET{(Q_{i}^k)^-} & \ISA & Y_i^k &
      Q_i^k & \ISA & S_i&
      \SOMET{S_n^-} & \ISA & \SOMET{R_j}\\
      \SOMET{R_j^-} & \ISA & \SOMET{R_j}\\
      \\
      A & \ISA & \SOMET{T_0^-} &
      \SOMET{T_{i-1}^-} & \ISA & \SOMET{P_i^k} &
      P_i^k & \ISA & T_i\\
      \SOMET{(P_{i}^k)^-} & \ISA & X_i^k &
      X_i^0 & \ISA & \SOMET{R_j^i} & \multicolumn{3}{l}{\text{if } \neg X_i \in C_j} \\
      &&& X_i^1 & \ISA & \SOMET{R_j^i} & \multicolumn{3}{l}{\text{if } X_i \in C_j}\\
      \SOMET{(R_j^i)^-} & \ISA & \SOMET{R_j^{i-1}}\\
    \end{array}\]
    and $\A_1 = \{A(a)\}$.

    Let $\Sigma_2 = \{A', Z_i^0, Z_i^1, S', R_j'\}$ where $A', Z_i^0, Z_i^1$ are
    concept names and $S', R_j'$ are role names, $\M=(\Sigma_1, \Sigma_2,
    \T_{12})$, and $\T_{12}$ the following set of inclusions:
    \[\begin{array}{r@{~~}c@{~~}l@{\qquad}r@{~~}c@{~~}l}
      A & \ISA & A' &
      S_i & \ISA & S' \\
      &&& T_i & \ISA & S' \\
      &&& Y_i^k & \ISA & Z_i^k \\
      &&& X_i^k & \ISA & Z_i^k \\
    \end{array}
    \begin{array}{r@{~~}c@{~~}l}
      R_j & \ISA & R_j'\\
      T_i & \ISA & {R_j'}^-\\
      R_j^i & \ISA & R_j' \\
      R_j^0 & \ISA & {R_j'}^-\\
    \end{array}\]
    We verify that $\models \phi$ if and only if $\Uni_{\tup{\T_1 \cup
        \T_{12},\A_1}}$ is $\Sigma_2$-homomorphically embeddable into a finite
    subset of itself. The latter, in turn, is equivalent to the existence of a
    universal solution for $\K_1 = \tup{\T_1, \A_1}$ under $\M$, which is shown
    in Lemma~\ref{lem:uni-sol-exists-iff-chase-fin-part}.

    For $\phi = \exists X_1 \forall X_2 \exists X_3(X_1 \wedge (X_2 \vee \neg
    X_3))$, $\Sigma_2$-reduct of $\Uni_{\tup{\T_1 \cup \T_{12}, \A_1}}$ can be
    depicted as follows:
    
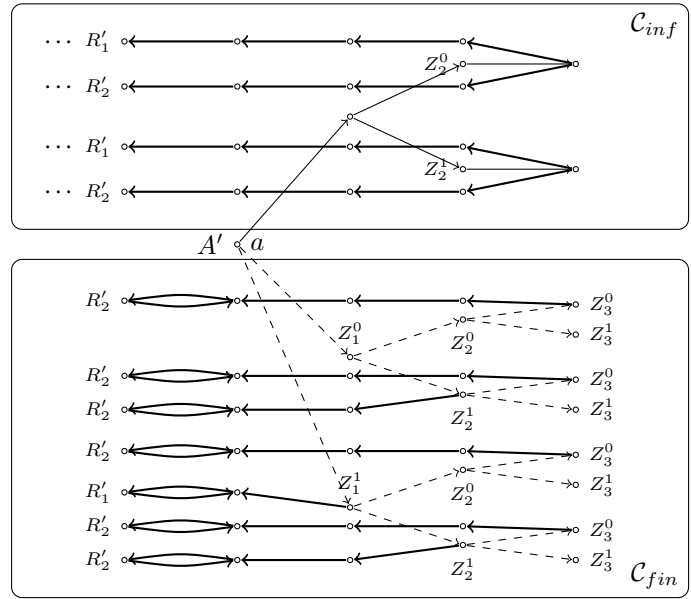
\begin{center}
      \begin{tikzpicture}
        \inda[label=right:{$a$},label=left:{$A'$},a,(0,0)];
        \begin{scope}[shift={(0,1.7)}]
          \foreach \name/\x/\y in { i1/1.5/0, i11/3/0.7, i12/3/-0.7,
            i111/4.5/0.7, i121/4.5/-0.7, d111/3/1, d112/1.5/1, d113/0/1,
            d114/-1.5/1, d211/3/0.4, d212/1.5/0.4, d213/0/0.4, d214/-1.5/0.4,
            d121/3/-0.4, d122/1.5/-0.4, d123/0/-0.4, d124/-1.5/-0.4, d221/3/-1,
            d222/1.5/-1, d223/0/-1, d224/-1.5/-1 } {%
            \inda[,,\name,(\x,\y)]; }
          \node[anchor=east] at (i11.west) {\scriptsize $Z_2^0$};
          \node[anchor=east] at (i12.west) {\scriptsize $Z_2^1$};
          \foreach \name/\lab in { d114/{$R_1'$}, d214/{$R_2'$}, d124/{$R_1'$},
            d224/{$R_2'$} } { \node[anchor=east] (n) at (\name.west)
            {\scriptsize \lab}; \node at ($(n) - (0.5,0)$) {\dots}; }
          \draw[rounded corners] (-3,-1.5) rectangle (6,1.5); \node[anchor=north
          east] at (6,1.5) {$\C_{inf}$};
        \end{scope}
        \begin{scope}[shift={(0,-2.5)}]
          \foreach \name/\x/\y/\wh/\lab in { j1/1.5/1/above/{$Z_1^0$},
            j2/1.5/-1/above/{$Z_1^1$}, j11/3/1.5/below/{$Z_2^0$},
            j12/3/0.5/below/{$Z_2^1$}, j21/3/-0.5/below/{$Z_2^0$},
            j22/3/-1.5/below/{$Z_2^1$}, j111/4.5/1.7/right/{$Z_3^0$},
            j112/4.5/1.3/right/{$Z_3^1$}, j121/4.5/0.7/right/{$Z_3^0$},
            j122/4.5/0.3/right/{$Z_3^1$}, j211/4.5/-0.3/right/{$Z_3^0$},
            j212/4.5/-0.7/right/{$Z_3^1$}, j221/4.5/-1.3/right/{$Z_3^0$},
            j222/4.5/-1.7/right/{$Z_3^1$}, c214/-1.5/1.75/left/{$R_2'$},
            c224/-1.5/0.75/left/{$R_2'$}, c234/-1.5/-0.25/left/{$R_2'$},
            c244/-1.5/-1.25/left/{$R_2'$}, cc214/-1.5/0.3/left/{$R_2'$},
            cc224/-1.5/-1.7/left/{$R_2'$}, c114/-1.5/-0.8/left/{$R_1'$} } {%
            \inda[label=\wh:{\scriptsize \lab},,\name,(\x,\y)]; }
          \foreach \name/\x/\y in { c211/3/1.75, c212/1.5/1.75, c213/0/1.75,
            c221/3/0.75, c222/1.5/0.75, c223/0/0.75, c231/3/-0.25,
            c232/1.5/-0.25, c233/0/-0.25, c241/3/-1.25, c242/1.5/-1.25,
            c243/0/-1.25, cc212/1.5/0.3, cc213/0/0.3, cc222/1.5/-1.7,
            cc223/0/-1.7, c113/0/-0.8 } {%
            \inda[,,\name,(\x,\y)]; }
          \draw[rounded corners] (-3,-2.2) rectangle (6,2.3); \node[anchor=south
          east] at (6,-2.2) {$\C_{fin}$};
        \end{scope}
        \foreach \from/\to in {%
          a/i1, i1/i11, i1/i12, i11/i111, i12/i121}{
          \redge[\from,\to,];
        }
        \begin{scope}[dashed]
          \foreach \from/\to/\lab in { a/j1, a/j2, j1/j11, j1/j12, j2/j21,
            j2/j22, j11/j111, j11/j112, j12/j121, j12/j122, j21/j211, j21/j212,
            j22/j221, j22/j222}{
            \redge[\from,\to,];
          }
        \end{scope}
        \begin{scope}[thick]
          \foreach \from/\to in { j111/c211, c211/c212, c212/c213, j121/c221,
            c221/c222, c222/c223, j211/c231, c231/c232, c232/c233, j221/c241,
            c241/c242, c242/c243, j12/cc212, cc212/cc213, j22/cc222,
            cc222/cc223, j2/c113}{%
            \redge[\from,\to,]; }%
          \foreach \from/\to in { c213/c214, c223/c224, c233/c234, c243/c244,
            cc213/cc214, cc223/cc224, c113/c114}{ \redgecurve[\from,\to, ,
            ($0.5*(\from)+0.5*(\to)+(0,0.1)$)]; \redgecurve[\to,\from, ,
            ($0.5*(\from)+0.5*(\to)-(0,0.1)$)]; }%
          \foreach \from/\to/\wh/\lab in { i111/d111/above/{$R_1'$},
            d111/d112/above/{$R_1'$}, d112/d113/above/{$R_1'$},
            d113/d114/above/{$R_1'$}, i111/d211/below/{$R_2'$},
            d211/d212/below/{$R_2'$}, d212/d213/below/{$R_2'$},
            d213/d214/below/{$R_2'$}, i121/d121/above/{$R_1'$},
            d121/d122/above/{$R_1'$}, d122/d123/above/{$R_1'$},
            d123/d124/above/{$R_1'$}, i121/d221/below/{$R_2'$},
            d221/d222/below/{$R_2'$}, d222/d223/below/{$R_2'$},
            d223/d224/below/{$R_2'$}}{
            \redge[\from,\to,];
          }
        \end{scope}
      \end{tikzpicture}
    \end{center}
    where each edge \tikz \draw[->] (0,0) -- (0.7,0.1); is labeled with $S'$,
    each edge \tikz \draw[->,dashed] (0,0) -- (0.7,0.1); is labeled with $S',
    {R_j'}^-$ for $1 \leq j \leq m$, and the labels of edges \tikz
    \draw[->,thick] (0,0) -- (0.7,0.1); are shown to the left of each infinite
    and finite path. The labels of the nodes (if any) are shown next to each
    node.

    \newcommand{\fin}{\mathit{fin}} \renewcommand{\inf}{\mathit{inf}}
    \renewcommand{\aa}{\mathfrak{a}}

    Let $\C_{\inf}$ and $\C_{\fin}$ be the parts of $\Uni_{\tup{\T_1 \cup
        \T_{12}, \A_1}}$ generated using the first 9 axioms and the last 9
    axioms of $\T_1$ respectively. Note that $\C_{\inf}$ is infinite, while
    $\C_{\fin}$ is finite. One can show that $\C_{\inf}$ is
    $\Sigma_2$-homomorphically embeddable into $\C_{\fin}$ (which is equivalent
    to $\Uni_{\tup{\T_1 \cup \T_{12}, \A_1}}$ is $\Sigma_2$-homomorphically
    embeddable into $\C_{\fin}$) iff $\phi$ is satisfiable.

    The rest of the proof follows the line of the proof of Theorem~11 in
    \cite{KKLSWZ11}.

    ($\Rightarrow$) Suppose $\models \phi$. We show that the canonical model
    $\Uni_{\tup{\T_1 \cup \T_{12},\A_1}}$ is $\Sigma_2$-homomorphically
    embeddable into a finite subset of itself. More precisely, let us denote
    with $\T_1^{\inf}$ the subset of $\T_1$ consisting of the first 9 axioms,
    and $\T_1^{\fin}$ the subset of $\T_1$ consisting of the last 9 axioms. Then
    $\Uni_{\tup{\T_1 \cup \T_{12},\A_1}} = \Uni_{\tup{\T_1^{\inf} \cup
        \T_{12},\A_1}} \cup \Uni_{\tup{\T_1^{\fin} \cup \T_{12},\A_1}}$, and we
    construct a $\Sigma_2$-homomorphism $h: \dom[\Uni_{\tup{\T_1^{\inf} \cup
        \T_{12},\A_1}}] \to \dom[\Uni_{\tup{\T_1^{\fin} \cup
        \T_{12},\A_1}}]$. In the following we use $\Uni_{\inf}$ to denote
    $\Uni_{\tup{\T_1^{\inf} \cup \T_{12},\A_1}}$, and $\Uni_{\fin}$ to denote
    $\Uni_{\tup{\T_1^{\fin} \cup \T_{12},\A_1}}$.

    We begin by setting $h(\Int[\Uni_{\inf}]{a}) = \Int[\Uni_{\fin}]{a}$. Then
    we define $h$ in such a way that, for each path $\pi$ in $\Uni_{\inf}$ of
    length $i+1 \leq n$, $h(\pi)$ is a path $\Int[\Uni_{\fin}]{a}w_1 \ldots w_i$
    of length $i+1$ in $\Uni_{\fin}$ and it defines an assignment $\aa_{h(\pi)}$
    to the variables $X_1, \ldots, X_i$ by taking, for all $1 \leq i' \leq i$,
    \begin{center}
      $\aa_{h(\pi)}(X_{i'}) = \top \Leftrightarrow \Int[\Uni_{\fin}]{a} \cdot
      w_1 \cdot \ldots \cdot w_{i'} \in \Int[\Uni_{\fin}]{(X^1_{i'})}$\\
      $\aa_{h(\pi)}(X_{i'}) = \bot \Leftrightarrow \Int[\Uni_{\fin}]{a} \cdot
      w_1 \cdot \ldots \cdot w_{i'} \in \Int[\Uni_{\fin}]{(X^0_{i'})}$.
    \end{center}

    Such assignments $\aa_{h(\pi)}$ will satisfy the following:
    \begin{description}
    \item[($\aa$)] the QBF obtained from $\phi$ by removing $\QQ_1X_1 \dots
      \QQ_iX_i$ from its prefix is true under $\aa_{h(\pi)}$.
    \end{description}
    For the paths of length 0 the $\Sigma_2$-homomorphism $h$ has been defined
    and ($\aa$) trivially holds. Suppose that we have defined $h$ for all paths
    in $\Uni_{\inf}$ of length $i+1\leq n$. We extend $h$ to all paths of length
    $i+2$ in $\Uni_{\inf}$ such that ($\aa$) holds. Let $\pi$ be a path of
    length $i+1$. In $\Uni_{\fin}$ we have
    \[\tail(h(\pi)) \leadsto_{\tup{\T_1^{\fin} \cup \T_{12},\A_2}}
    w^{\Uni_{\fin}}_{[P_i^k]}, \quad\text{and}\quad h(\pi) \cdot
    w^{\Uni_{\fin}}_{[P_i^k]} \in \Int[\Uni_{\fin}]{(X_i^k)}, \text{ for }k=0,1.\]
    If $\QQ_i = \forall$ then in $\Uni_{\inf}$ we have
    \[\tail(\pi) \leadsto_{\tup{\T_1^{\inf} \cup \T_{12},\A_2}}
    w^{\Uni_{\inf}}_{[Q_i^k]}, \quad\text{and}\quad \pi \cdot
    w^{\Uni_{\inf}}_{[Q_i^k]} \in \Int{(X_i^k)}, \text{ for }k=0,1.\]
    Thus, we set $h(\pi \cdot w^{\Uni_{\inf}}_{[Q_i^k]}) = h(\pi) \cdot
    w^{\Uni_{\fin}}_{[P_i^k]}$, for $k=0,1$. Clearly, ($\aa$) holds. Otherwise,
    $\QQ_i = \exists$ and in $\Uni_{\inf}$ we have
    \[\tail(\pi) \leadsto_{\tup{\T_1^{\inf} \cup \T_{12},\A_2}}
    w^{\Uni_{\inf}}_{[S_i]}.\]
    We know that $\models \phi$ and so, by, ($\aa$), the QBF obtained from $\pi$
    by removing $\QQ_1 X_1 \ldots \QQ_i X_i$ is true under either $\aa_{h(\pi)}
    \cup \{X_i = \top\}$ or $\aa_{h(\pi)} \cup \{X_i = \bot\}$. We set $h(\pi
    \cdot w^{\Uni_{\inf}}_{[S_i]}) = h(\pi) \cdot w^{\Uni_{\fin}}_{[P_i^k]}$
    with $k=1$ in the former case, and $k=0$ in the latter case. Either way,
    ($\aa$) holds.

    Consider now in $\Uni_{\inf}$ a path $\pi$ of length $n+1$ from
    $\Int[\Uni_{\inf}]{a}$ to $\Int[\Uni_{\inf}]{w_n}$. By construction, we have
    \[h(\pi) = \Int[\Uni_{\fin}]{a} \cdot \Int[\Uni_{\fin}]{w_{[P_1^{k_1}]}} \cdot
    \ldots \cdot \Int[\Uni_{\fin}]{w_{[P_n^{k_n}]}}.\]
    Next, on the one hand, the path $\pi$ in $\Uni_{\inf}$ has $m$ infinite
    extensions of the form $\pi \cdot w^{\Uni_{\inf}}_{[R_j]} \cdot
    w^{\Uni_{\inf}}_{[R_j]} \dots$, for $1 \leq j \leq m$. On the other hand, as
    $\models \phi$, by ($\aa$), for each clause $C_j$, there is some $1 \leq i'
    \leq n$ such that $h(\pi)$ contains $w^{\Uni_{\fin}}_{[P_{i'}^1]}$ if
    $X_{i'} \in C_j$, or $w^{\Uni_{\fin}}_{[P_{i'}^0]}$ if $\neg X_{i'} \in
    C_j$. We set for each $1 \leq l \leq n-i'$,
    \[h(\pi \cdot \underbrace{w^{\Uni_{\inf}}_{[R_j]} \cdot \ldots \cdot
      w^{\Uni_{\inf}}_{[R_j]}}_{l\text{ times}}) = \Int[\Uni_{\fin}]{a} \cdot
    \Int[\Uni_{\fin}]{w_{[P_1^{k_1}]}} \cdot \ldots \cdot
    \Int[\Uni_{\fin}]{w_{[P_{n-l}^{k_{n-l}}]}},\]
    for each $n+1 \geq l > n - i'$, 
    \[h(\pi \cdot \underbrace{w^{\Uni_{\inf}}_{[R_j]} \cdot \ldots \cdot
      w^{\Uni_{\inf}}_{[R_j]}}_{l\text{ times}}) = \Int[\Uni_{\fin}]{a} \cdot
    \Int[\Uni_{\fin}]{w_{[P_1^{k_1}]}} \cdot \ldots \cdot
    \Int[\Uni_{\fin}]{w_{[P_{i'}^{k_{i'}}]}} \cdot w^{\Uni_{\fin}}_{[R^{i'}_j]} \cdot \ldots
    \cdot w^{\Uni_{\fin}}_{[R^{n-l+1}_j]},\]
    and for each $l > n+1$
    \[h(\pi \cdot \underbrace{w^{\Uni_{\inf}}_{[R_j]} \cdot \ldots \cdot
      w^{\Uni_{\inf}}_{[R_j]}}_{l\text{ times}}) = \Int[\Uni_{\fin}]{a} \cdot
    \Int[\Uni_{\fin}]{w_{[P_1^{k_1}]}} \cdot \ldots \cdot
    \Int[\Uni_{\fin}]{w_{[P_{i'}^{k_{i'}}]}} \cdot
    \Int[\Uni_{\fin}]{w_{[R^{i'}_j]}} \cdot \Int[\Uni_{\fin}]{w_{[R^{i'-1}_j]}}
    \cdot \ldots \cdot w^{\Uni_{\fin}}_{[R^{i^\star}_j]},\]
    where $i^\star = (n-l+1) \mod 2$. It is immediate to verify that $h$ is a
    $\Sigma_2$-homomorphism from $\Uni_{\inf}$ to $\Uni_{\fin}$.

    ($\Leftarrow$) Let $h$ be a $\Sigma_2$-homomorphism from $\Uni_{\inf}$ to
    $\Uni_{\fin}$. We show that $\models \phi$.

    Let $\pi$ be a path of length $n+1$, $\pi = \Int[\Uni_{\inf}]{a} \cdot w_1
    \cdot \ldots \cdot w_n$, in $\Uni_{\inf}$. Then $(\Int[\Uni_{\inf}]{a},
    \pi_1), (\pi_i, \pi_{i+1}) \in \Int[\Uni_{\inf}]{S'}$, where $\pi_i =
    \Int[\Uni_{\inf}]{a} \cdot w_1 \cdot \ldots \cdot w_i$, for $1 \leq i \leq
    n-1$. Furthermore, let $Z_1^{k_1}, Z_2^{k_2}, \dots, Z_n^{k_n}$ be the
    concepts containing subpaths of $h(\pi_i)$. We show that for every $1 \leq j
    \leq m$, the clause $C_j$ contains at least one of the literals \[\{X_i \mid
    k_i = 1, 1 \leq i \leq n\} \cup \{\neg X_i \mid k_i = 0, 1 \leq i \leq
    n\}.\] Validity of $\phi$ will follow.

    Consider a path of the form $\pi \cdot
    \underbrace{\Int[\Uni_{\inf}]{w_{[R_j]}} \cdot \ldots \cdot
      \Int[\Uni_{\inf}]{w_{[R_j]}}}_{n+1\text{ times}}$ in $\Uni_{\inf}$. Then
    its $h$-image in $\Uni_{\fin}$ must be of the form \[\Int[\Uni_{\fin}]{a}
    \cdot \Int[\Uni_{\fin}]{w_{[P_1^{k_1}]}} \cdot \ldots \cdot
    \Int[\Uni_{\fin}]{w_{[P_i^{k_i}]}} \cdot \Int[\Uni_{\fin}]{w_{[R_j^i]}}
    \cdot \Int[\Uni_{\fin}]{w_{[R_j^{i-1}]}} \cdot \ldots \cdot
    \Int[\Uni_{\fin}]{w_{[R_j^{i'}]}}\] for some $1 \leq i \leq n$, $i' = 0$ or
    $i' = 1$, and $k_i=0$ or $k_i=1$. If $k_i=0$, then $C_j$ must contain $\neg
    X_i$, otherwise $X_i$.
  \end{proof}

  \begin{lemma}
    \label{lem:non-emp-uni-sol-exptime}
    The non-emptiness problem for universal solutions is in \EXPTIME. For a
    given \dlliter mapping $\M$ and a given \dlliter KB $\K_1$, if a universal
    solution $\A_2$ (an extended ABox without inequalities) exists, then it is
    at most exponentially large in the size of $\K_1 \cup \M$.
  \end{lemma}
  \begin{proof}
    First, we provide an algorithm for checking existence of a universal
    solution with extended ABoxes in \dlliterpos. Given a \dlliterpos mapping
    $\M=(\Sigma_1, \Sigma_2, \T_{12})$, to verify that a universal solution for
    $\tup{\T_1, \A_1}$ under $\M$ exists, we check for non-emptiness of the
    automaton $\Bau$ defining the intersection of the automata
    $\pi_{\Gamma_{\K}}(\acan)$, $\pi_{\Gamma_\K}(\asol)$, and $\agfin$, where
    $\K=\tup{\T_1 \cup \T_{12}, \A_1}$, $\pi_{\Gamma_\K}(\acan)$ is the
    projection of $\acan$ on the vocabulary $\Gamma_\K$, and likewise for
    $\pi_{\Gamma_\K}(\asol)$. If the language accepted by $\Bau$ is empty, then
    there is no universal solution, otherwise a universal solution exists and it
    is exactly the tree accepted by $\Bau$.

    \begin{proposition}
      \label{prop:uni-sol-exists-iff-B-non-empty}
      Let $\M = (\Sigma_1, \Sigma_2, \T_{12})$ be a \dlliterpos mapping, and
      $\K_1 =\tup{\T_1, \A_1}$ a \dlliterpos KB over $\Sigma_1$. Then, a
      universal solution with extended ABoxes for $\K_1$ under $\M$ exists iff
      the language of the automata $\Bau = \pi_{\Gamma_{\K}}(\acan) \cap \agfin
      \cap \pi_{\Gamma_\K}(\asol)$, where $\K = \tup{\T_1 \cup \T_{12}, \A_1}$,
      is non-empty.
    \end{proposition}

    \begin{proof}
      ($\Leftarrow$) Assume that $\L(\Bau) \neq \emptyset$ and $T \in
      \L(\Bau)$. Let $T_G$ be the subtree of $T$ defined by the $G$ labels, and
      $\I_{T,G}$ the interpretation represented by $T_G$. Then from the definition of
      $\Bau$ it follows that

      \begin{enumerate}
      \item $\I_{T,G}$ is a finite interpretation of $\Sigma_2$ and $\I_{T,G}
        \subseteq \Uni_{\tup{\T_1 \cup \T_{12}, \A_1}}$,
      \item there exists an interpretation $\I$ of $\Sigma_1$ such that $\I \cup
        \I_{T,G}$ is a model of $\tup{\T_1 \cup \T_{12},\A_1}$.
      \end{enumerate}

      Since $\I_{T,G}$ is finite, let $\A_{T,G}$ be the ABox over $\Sigma_2$
      such that $\Uni_{\A_{T,G}} = \I_{T,G}$. Then, $\A_{T,G}$ is a solution for
      $\K_1$ under $\M$ (by the second item). We show it is a universal
      solution. Let $\J$ be an interpretation of $\Sigma_2$ such that for some
      model $\I$ of $\K_1$, $(\I,\J) \models \M$. Then, since $\Uni_{\tup{\T_1
          \cup \T_{12},\A_1}}$ is the canonical model of $\tup{\T_1 \cup
        \T_{12}, \A_1}$, there exists a homomorphism from $\Uni_{\tup{\T_1 \cup
          \T_{12},\A_1}}$ to $\I \cup \J$ ($\I \cup \J$ is a model of $\tup{\T_1
        \cup \T_{12}, \A_1}$). In particular, there is a homomorphism from
      $\I_{T,G}$ to $\I \cup \J$, and as $\I_{T,G}$ and $\I$ are interpretations
      of disjoint signatures, there is a homomorphism $h$ from $\I_{T,G}$ to
      $\J$. Hence, $\J$ is a model of $\A_{T,G}$: take $h$ as the substitution
      for the labeled nulls. By definition of universal solution, $\A_{T,G}$ is
      a universal solution for $\K_1$ under $\M$.

      ($\Rightarrow$) Assume a universal solution for $\K_1$ under $\M$
      exists. Then by Lemma~\ref{lem:uni-sol-exists-iff-chase-fin-part} there
      exists a universal solution $\A_2$ such that $\V_{\A_2} \subseteq
      \Uni_{\tup{\T_1 \cup \T_{12}, \A_1}}$. Therefore, the language of $\Bau$
      is not empty.
    \end{proof}

    As a corollary of
    Lemma~\ref{lem:A2-uni-sol-iff-homo-equiv-and-sigma2-positive},
    Lemma~\ref{lem:check-sigma2-pos-in-ptime},
    Lemma~\ref{lem:uni-sol-exists-iff-chase-fin-part}, and
    Proposition~\ref{prop:uni-sol-exists-iff-B-non-empty} we obtain the
    exponential time upper bound of the non-emptiness problem for universal
    solutions with extended ABoxes in \dlliter. Moreover, $\A_{T,G}$ is at most
    exponentially large in the size of $\K_1$ and $\M$.
  \end{proof}

\end{proof}

\subsection{Proof of Theorem~\ref{the:memb-uni-sol-npcomplete}}

\begin{proof}
  We show that the membership problem for universal solutions with extended
  ABoxes is \NP-complete by first proving the lower bound, and then the upper
  bound.

  \begin{lemma}\label{lem:memb-uni-sol-nphard}
    The membership problem for universal solutions with extended ABoxes is 
    \NP-hard.
  \end{lemma}
  \begin{proof}
    The proof is by reduction of 3-colorability of undirected graphs known to be
    \NP-hard. Suppose we are given an undirected graph $G = (V,E)$. Let
    $\Sigma_1 = \{\Edge\}$ and $\Sigma_2 = \{\Edge'\}$. Let $r,g,b \in N_a$, $V
    \subseteq N_l$ and
    \[
    \begin{array}{rcl}
      \A_1 &=& \{\Edge(r,g), \Edge(g,r), \Edge(r,b), \Edge(b,r), \Edge(g,b), \Edge(b,g)\},\\ 
      \T_1 &=& \{\}, \\
      \T_{12} &=& \{\Edge \ISA \Edge'\}, \\
      \A_2 &=& \{\Edge'(r,g), \Edge'(g,r), \Edge'(r,b), \Edge'(b,r), \Edge'(g,b), \Edge'(b,g)\} \cup {} \\
      & &\{\Edge'(x,y), \Edge'(y,x) \mid (x,y) \in E\}.
    \end{array}
    \]
    Note that the nodes in $G$ become labeled nulls in $\A_2$.

    We show that $G$ is 3-colorable if and only if $\A_2$ is a universal
    solution for $\K_1 = \tup{\T_1, \A_1}$ under $\M= (\Sigma_1, \Sigma_2,
    \T_{12})$.

    ($\Rightarrow$) Suppose $G$ is 3-colorable. Then it follows that there
    exists a function $h$ that assigns to each vertex from $V$ one of the colors
    $\{r,g,b\}$ such that if $(x,y) \in E$, then $h(x) \neq h(y)$, hence $h$ is
    a homomorphism from $G$ to the undirected graph $(\{r,g,b\}, \{(r,g), (g,b),
    (b,r)\})$.
    
    We prove that $\A_2$ is a universal solution for $\K_1$ under
    $\M$. Obviously, $\K_1$ and $\M$ are $\Sigma_2$-positive. Thus, it remains
    to verify that $\V_{\A_2}$ is $\Sigma_2$-homomorphically equivalent to
    $\Uni_{\tup{\T_1 \tup \T_{12},\A_1}}$. First, it is easy to see that
    $\Uni_{\tup{\T_1 \cup \T_{12},\A_1}}$ is $\Sigma_2$-homomorphically
    embeddable into $\V_{\A_2}$.  Second, $h$ is also a homomorphism from
    $\V_{\A_2}$ to $\Uni_{\tup{\T_1 \cup \T_{12},\A_1}}$, thus $\V_{\A_2}$ is
    homomorphically embeddable into $\Uni_{\tup{\T_1 \cup \T_{12},\A_1}}$.

    ($\Leftarrow$) Suppose now $\A_2$ is a universal solution for $\K_1$ under
    $\M$. Then by Lemma~\ref{lem:A2-uni-sol-dlliter-pos-iff-homo-equiv} it
    follows that $\V_{\A_2}$ is $\Sigma_2$-homomorphically equivalent to
    $\Uni_{\tup{\T_1 \cup \T_{12},\A_1}}$. Let $h$ be a homomorphism from
    $\V_{\A_2}$ to $\Uni_{\tup{\T_1 \cup \T_{12},\A_1}}$. Then $h$ assigns to
    each labeled null $x \in \dom[\A_2]$ some constant $a \in \dom[\A_1]$, and
    it is easy to see that $h$ is an assignment for the vertices in $V$ that is
    a 3-coloring of $G$.
  \end{proof}

  \begin{lemma}\label{lem:memb-uni-sol-np}
    The membership problem for universal solutions with extended ABoxes is in
    \NP.
  \end{lemma}
  \begin{proof}
    Assume we are given a mapping $\M=(\Sigma_1, \Sigma_2, \T_{12}$), a source KB
    $\K_1 =\tup{\T_1, \A_1}$, and a target ABox $\A_2$. We want to decide whether
    $\A_2$ is a universal solution with extended ABoxes for $\K_1$ under $\M$
    (ABoxes without inequalities).

    We need the following proposition that provides an upper bound for checking
    existence of homomorphism from $\V_{\A_2}$ to $\Uni_{\tup{\T_1 \cup \T_{12},
        \A_1}}$.
    \begin{proposition}
      Deciding whether $\V_{\A_2}$ is homomorphically embeddable into
      $\Uni_{\tup{\T_1 \cup \T_{12}, \A_1}}$ can be done in \NP in the size of
      $\K_1$, $\M$ and $\A_2$.
    \end{proposition}
    
    \begin{proof}
      First, if there exists a homomorphism $h$ from $\V_{\A_2}$ to
      $\Uni_{\tup{\T_1 \cup \T_{12}, \A_1}}$, then there exists a polynomial size
      witness $\A_3$ such that $\V_{\A_3} \subseteq \Uni_{\tup{\T_1 \cup \T_{12},
          \A_1}}$ and $h$ is a homomorphism from $\V_{\A_2}$ to $\V_{\A_3}$ (take
      $\V_{\A_3} = h(\V_{\A_2})$, then $\card{\A_3} \leq \card{\A_2}$). Therefore, to
      verify that such $h$ exists, it is sufficient to compute $\A_3$ and then to
      check whether $\V_{\A_2}$ can be homomorphically mapped into $\V_{\A_3}$.

      Second, there exists a witness $\A_3$ such that $\V_{\A_3} \subseteq
      \Uni_{\tup{\T_1 \cup \T_{12}, \A_1}}$ and every $x \in \dom[\A_3]$ is a path
      of polynomial length in the size of $\T_1 \cup \T_{12}$ and $\A_2$ (more
      precisely, of length smaller or equal $2m$, where $m$ is the size of $\T_1
      \cup \T_{12} \cup \A_2$). Proof: let h be a homomorphism from $\V_{\A_2}$ to
      $\Uni_{\tup{\T_1 \cup \T_{12}, \A_1}}$ and $\A_3$ an ABox such that
      $\V_{\A_3} = h(\V_{\A_2})$. Assume that $x \in \dom[\A_3]$ and the length of
      $x$ is more than $2m$. Then $x$ is not connected to $\Ind(\A_1)$ in $\A_3$,
      i.e., there exists no path $R_1(x_1,x_2), \dots, R_n(x_n, x_{n+1})$ with
      $x_1 =x$, $x_{n+1} =a \in \Ind(\A_1)$, $R_i(x_i,x_{i+1}) \in \A_3$
      (otherwise it contradicts $\V_{\A_3} = h(\V_{\A_2})$). Let $C$ be the
      maximal connected subset of $\A_3$ with $x \in \dom[C]$, i.e., $\dom[C] \cap
      \dom[\A_3 \setminus C] = \emptyset$ and for each $C' \subseteq C$, $\dom[C']
      \cap \dom[C \setminus C'] \neq \emptyset$, moreover $\dom[C] \cap \Ind(\A_1)
      = \emptyset$. Let $y$ be the path (in the sense of $\gpath(\tup{\T_1 \cup
        \T_{12}, \A_1})$) of minimal length in $C$, it exists and is unique since
      $\V_{\A_3} \subseteq \Uni_{\tup{\T_1 \cup \T_{12}, \A_1}}$ and there are no
      constants in $C$, and for each $x \in C$, $x = y \cdot w_{[R_1]} \dots
      w_{[R_n]}$ for some $n$. Further assume $\tail(y) = w_{[R]}$, then let $y'$
      be a path of the minimal length in $\dom[\Uni_{\tup{\T_1 \cup \T_{12},
          \A_1}}]$ with $\tail(y') = w_{[R]}$ (note that there is an infinite
      number of $y''$ with $\tail(y'') = w_{[R]}$). Then the length of $y'$ is
      bounded by the size of $\T_1 \cup\T_{12}$ and the length of each $y' \cdot
      w_{[R_1]} \dots w_{[R_n]}$, where $y \cdot w_{[R_1]} \dots w_{[R_n]} \in C$,
      is bounded by the size of $\T_1 \cup\T_{12} \cup \A_2$. Now, define a new
      function $h': \dom[\V_{\A_2}] \to \dom[\Uni_{\tup{\T_1 \cup \T_{12},
          \A_1}}]$ such that $h'(x) = h(x)$ if $h(x) \notin C$, $h'(x) = y' \cdot
      w_{[R_1]} \dots w_{[R_n]}$ if $h(x) = y \cdot w_{[R_1]} \dots w_{[R_n]}$. It
      is easy to see that $h'$ is a homomorphism from $\V_{\A_2}$ to
      $\Uni_{\tup{\T_1 \cup \T_{12}, \A_1}}$. We can continue this iteratively
      until we get that for every $x \in \dom[\A_3]$, $x$ is a path of length
      bounded by $2m$, where $\A_3$ is an ABox such that $\V_{\A_3} =
      h'(\V_{\A_2})$.

      Finally, our algorithm for checking existence of a homomorphism from $\V_{\A_2}$ to
      $\Uni_{\tup{\T_1 \cup \T_{12}, \A_1}}$ is as follows:
      \begin{enumerate}
      \item compute (guess) $\A_3$ (in \NP):
        \begin{itemize}
        \item for each $x \in \dom[\A_2]$ we guess 
          $y \in \dom[\Uni_{\tup{\T_1 \cup \T_{12}, \A_1}}]$ such that there
          exists a $\tup{\T_1 \cup \T_{12},\A_1}$-path from some $a \in
          \Ind(\A_1)$ to $y$ and $y$ is a path of polynomial length,
        \item Let $W$ be the set of all $y$ guessed above, then
          \[\begin{array}{rcl}
            \A_3& =& \{A(x) \mid x \in W, \tail(x) = w_{[R]}, \T_1 \cup \T_{12} \models
            \SOMET{R^-} \ISA A, A \in \Sigma_2\} \cup {}\\
            &&\{S(x',x) \mid x,x' \in W, x = x' \cdot w_{[R]}, \T_1 \cup \T_{12} \models
            R \ISA S, S \in \Sigma_2\},
          \end{array}\]
          $\V_{\A_3} \subseteq \Uni_{\tup{\T_1 \cup \T_{12}, \A_1}}$, $\dom[\A_3] =
          W$ and $\A_3$ is of polynomial size. 
        \end{itemize}

      \item check whether there exists a homomorphism from $\V_{\A_2}$ to
        $\V_{\A_3}$ (in \NP).
      \end{enumerate}

      We prove that the above described procedure is correct.

      Assume, we computed $\A_3$ and there exists a homomorphism $h$ from
      $V_{\A_2}$ to $V_{\A_3}$. Then since $\V_{\A_3} \subseteq \Uni_{\tup{\T_1
          \cup \T_{12}, \A_1}}$, it follows that $h$ is a homomorphism from
      $\V_{\A_2}$ to $\Uni_{\tup{\T_1 \cup \T_{12}, \A_1}}$.

      Now, assume that there exists no homomorphism from $V_{\A_2}$ to $V_{\A_3}$,
      and by contradiction there exists a homomorphism from $\V_{\A_2}$ to
      $\Uni_{\tup{\T_1 \cup \T_{12}, \A_1}}$. Then, we showed that there exists a
      homomorphism $h'$ from $\V_{\A_2}$ to $\Uni_{\tup{\T_1 \cup \T_{12}, \A_1}}$
      and an ABox $\A_3$ such that $\V_{\A_3} = h'(\V_{\A_2})$ and the length of
      every $x \dom[\A_3]$ is bounded by $2m$, where $m$ is the size of $\T_1 \cup
      \T_{12} \cup \A_2$. Contradiction with step~1.

    \end{proof}
    
    Then the membership check for universal solutions with extended ABoxes can be
    done as follows:
    \begin{enumerate}
    \item verify whether $\K_1$ and $\M$ are $\Sigma_2$-positive, if yes
    \item check whether $\T_2$ is equivalent to the empty TBox, if yes
    \item check whether $\A_2$ is a solution with extended
      ABoxes for $\K_1^{\pos}$ under $\M^{\pos}$, if yes
    \item check whether $\A_2$ is homomorphically embeddable into $\Uni_{\tup{\T_1
          \cup \T_{12},\A_1}}$. If yes, then $\K_2$ is a universal solution for
      $\K_1$ under $\M$, otherwise it is not.
    \end{enumerate}
    Steps~1 and 2 can be done in polynomial time. Step~3 can be done in \NP
    similarly to Theorem~\ref{the:non-emp-uni-sol-owlql-aboxes-np}: guess an
    interpretation $\I$ of $\Sigma_1$ of polynomial size, check whether $\I$ is a
    model of $\K_1^{\pos}$ and $(\I,\V_{\A_2}) \models \T_{12}^{\pos}$. If yes,
    then $\A_2$ is a solution: let $\J$ be a model of $\A_2$ and $h$ a
    homomorphism from $\V_{\A_2}$ to $\J$. Then, let $\I^\J$ be the image of $h$
    applied to $\I$, $\I^\J = h(\I)$. Then $\I^\J$ is a model of $\K_1^{\pos}$ and
    $(\I^\J, \J) \models \T_{12}^{\pos}$, hence indeed, $\A_2$ is a solution.
    Step~4 is feasible in \NP, therefore in overall the membership check can be
    done in \NP.
  \end{proof}

\end{proof}


\subsection{Proof of Theorem~\ref{the:memb-uni-UCQ-sol-pspace-hard}}
\begin{proof}
  \newcommand{\bb}[1]{\textcolor{blue}{#1}}
  \newcommand{\rr}[1]{\textcolor{red}{#1}}
  \renewcommand{\gg}[1]{\textcolor{green!60!black}{#1}}
  \newcommand{\fin}{\mathit{fin}}
  \renewcommand{\inf}{\mathit{inf}}
  The proof is by reduction of the satisfiability problem for quantified Boolean
  formulas, known to be $\PSPACE$-complete. Suppose we are given a QBF \[\phi =
  \QQ_1 X_1 \ldots \QQ_n X_n \bigwedge_{j=1}^m C_j\] where $\QQ_i \in
  \{\forall,\exists\}$ and $C_j$, $1 \leq j \leq m$, are clauses over the
  variables $X_i$, $1 \leq i \leq n$.

  Let $\Sigma_1 = \{A, Y_i^k, X_i^k, S_l, T_l, Q_i^k, P_i^k, R_j, R_j^l \mid 1
  \leq j \leq m, 1 \leq i \leq n, 0 \leq l \leq n, k \in \{0,1\}\}$ where $A,
  Y_i^k, X_i^k$ are concept names and the rest are role names. Let $\T_1$ be
  the following TBox over $\Sigma_1$ for $1 \leq j \leq m$, $1 \leq i \leq n$
  and $k \in \{0,1\}$:
  \[\begin{array}{r@{~~}c@{~~}l@{\qquad}r@{~~}c@{~~}l@{\qquad}r@{~~}c@{~~}l}
    A & \ISA & \SOMET{S_0^-} &
    \SOMET{S_{i-1}^-} & \ISA & \SOMET{Q_i^k} & 
    \multicolumn{3}{l}{\text{if } \QQ_i = \forall} \\
    &&&  
    \SOMET{S_{i-1}^-} & \ISA & \SOMET{S_i} & 
    \multicolumn{3}{l}{\text{if } \QQ_i = \exists} \\
    \SOMET{(Q_{i}^k)^-} & \ISA & Y_i^k &
    Q_i^k & \ISA & S_i&
    \SOMET{S_n^-} & \ISA & \SOMET{R_j}\\
    \SOMET{R_j^-} & \ISA & \SOMET{R_j}\\
    \\
    A & \ISA & \rr{\SOMET{T_0^-}} &
    \rr{\SOMET{T_{i-1}^-}} & \ISA & \rr{\SOMET{P_i^k}} &
    \rr{P_i^k} & \ISA & \rr{T_i}\\
    \rr{\SOMET{(P_{i}^k)^-}} & \ISA & \gg{X_i^k} &
    \gg{X_i^0} & \ISA & \rr{\SOMET{R_j^i}} & 
    \multicolumn{3}{l}{\text{if } \neg X_i \in C_j} \\
    &&& 
    \gg{X_i^1} & \ISA & \rr{\SOMET{R_j^i}} & 
    \multicolumn{3}{l}{\text{if } X_i \in C_j}\\
    \rr{\SOMET{(R_j^i)^-}} & \ISA & \rr{\SOMET{R_j^{i-1}}}\\
  \end{array}\]
  and $\A_1 = \{A(a)\}$.

  Further, let $\Sigma_2 = \{A', Z_i^0, Z_i^1, S', R'_j, {P'_i}^k, T'_l, {R'_j}^l\}$ where
  $A', Z_i^0, Z_i^1$ are concept names and $S', R_j', {P'_i}^k, T'_l, {R'_j}^l$ are role
  names, $\M=(\Sigma_1, \Sigma_2, \T_{12})$, and $\T_{12}$ the following set of
  inclusions:
  \[\begin{array}{r@{~~}c@{~~}l@{\qquad}r@{~~}c@{~~}l@{\qquad}r@{~~}c@{~~}l}
    A & \ISA & A' & 
    S_i & \ISA & \bb{S'} &
    R_j & \ISA & \bb{R_j'}\\
    &&&
    \rr{T_i} & \ISA & \bb{S'} &
    \rr{T_i} & \ISA & \bb{{R_j'}^-}\\
    &&&
    Y_i^k & \ISA & \bb{Z_i^k} &
    \rr{R_j^i} & \ISA & \bb{R_j'} \\
    &&&
    \gg{X_i^k} & \ISA & \bb{Z_i^k} &
    \rr{R_j^0} & \ISA & \bb{{R_j'}^-}\\
    \\
    \rr{P_i^k}& \ISA & \rr{{P'_i}^k}& 
    \rr{T_l} & \ISA & \rr{T_l'}& 
    \rr{R_j^l} & \ISA & \rr{{R'_j}^l}  
  \end{array}\]

  Finally, let $\A_2 = \{A'(a)\}$, and $\T_2$ the following target TBox for $1
  \leq j \leq m$, $1 \leq i \leq n$ and $k \in \{0,1\}$:
  \[\begin{array}{r@{~~}c@{~~}l@{\qquad}r@{~~}c@{~~}l@{\qquad}r@{~~}c@{~~}l}
    A' & \ISA & \rr{\SOMET{{T'_0}^-}} &
    \rr{\SOMET{{T'_{i-1}}^-}} & \ISA & \rr{\SOMET{{P'_i}^k}} &
    \rr{{P'_i}^k} & \ISA & \rr{T'_i}\\
    \rr{\SOMET{({P'_{i}}^k)^-}} & \ISA & \bb{Z_i^k} &
    \bb{Z_i^0} & \ISA & \SOMET{{R'_j}^i} & 
    \multicolumn{3}{l}{\text{if } \neg X_i \in C_j} \\
    &&& 
    \bb{Z_i^1} & \ISA & \SOMET{{R'_j}^i} & 
    \multicolumn{3}{l}{\text{if } X_i \in C_j}\\
    \rr{\SOMET{({R'_j}^i)^-}} & \ISA & \rr{\SOMET{{R_j'}^{i-1}}}\\
    &&&
    \rr{T'_i} & \ISA & \bb{S'} & 
    \rr{T'_i} & \ISA & \bb{{R'_j}^-} \\
    &&&
    &&&
    \rr{{R'_j}^i} & \ISA & \bb{R'_j} \\
    &&&
    &&&
    \rr{{R'_j}^0} & \ISA & \bb{{R'_j}^-} \\
  \end{array}\]

  We verify that $\models \phi$ if and only if $\tup{\T_2,\A_2}$ is a universal
  $\UCQ$-solution for $\K_1 = \tup{\T_1, \A_1}$ under $\M$. From
  Claim~\ref{th:query-entail-homo} it follows that $\tup{\T_2,\A_2}$ is a
  universal $\UCQ$-solution for $\K_1 = \tup{\T_1, \A_1}$ under $\M$ iff
  $\Uni_{\tup{\T_1 \cup \T_{12}, \A_1}}$ is finitely $\Sigma_2$-homomorphically
  equivalent to $\Uni_{\tup{\T_2, \A_2}}$. Therefore, we are going to show that
  $\models \phi$ if and only if $\Uni_{\tup{\T_1 \cup \T_{12}, \A_1}}$ is
  finitely $\Sigma_2$-homomorphically equivalent to $\Uni_{\tup{\T_2, \A_2}}$.

  The rest of the proof is similar to
  Lemma~\ref{lem:non-emp-uni-sol-pspace-hard}.

\end{proof}

\section{Membership Problem for $\UCQ$-representability}
Note that for the ease of notation, in all proofs and statements
concerning UCQ-representability we use $\Sigma$ instead of $\Sigma_1$
and $\Xi$ instead of $\Sigma_2$. At the same time, alternative syntax
for the disjointness assertions is used: we write $B \sqcap B'
\sqsubseteq \bot$ instead of $B \sqsubseteq \neg B'$, for basic
concepts $B$ and $B'$; analogiousy for roles.

We need several new definitions.  For a TBox $\T$, a pair of basic
concepts $B, B'$ (resp., pair of roles $R, R'$) is
\emph{$\T$-consistent} if $\tup{\T, \{B(o), B'(o)\}}$ (resp.,
$\tup{\T, \{R(o, o'), R'(o, o')\}}$) is a consistent KB. We say a
concept $B$ is \emph{$\T$-consistent} if the pair $B, B$ is
$\T$-consistent, and we define in a similar way $\T$-consistency of a
role $R$. Denote by $\consc[\T]$ ($\consr[\T]$) the set of all
$\T$-consistent concepts (roles).

\subsection{Basic Preliminary Results}

\begin{lemma}\label{lem:Unitype-KBtype-ABox}
  Let $\K=\tup{\T, \A}$ be a KB, $a, b \in N_a$, $\sigma \in
  \Delta^{\Uni_\K}$, and $\tail(\sigma) \leadsto_\K w_{[R]}$. Then,
  \begin{enumerate}[series=thmparts, label=\textbf{(\roman*)}, start=1]
  \item \label{tmpt:lem-Unitype-KBtype-ABox-1} $B \in
    \ttype{\Uni_\K}(a)$ iff $\A \models B'(a)$ and $\T \vdash B'
    \sqsubseteq B$;
  \item \label{tmpt:lem-Unitype-KBtype-ABox-2} $R \in
    \rtype{\Uni_\K}(a,b)$ iff $\A \models R'(a,b)$ and $\T \vdash R'
    \sqsubseteq R$;
  \item \label{tmpt:lem-Unitype-KBtype-ABox-3} $B \in
    \ttype{\Uni_\K}(\sigma w_{[R]})$ iff $\T \vdash \exists R^-
    \sqsubseteq B$;
  \item \label{tmpt:lem-Unitype-KBtype-ABox-4} $R \in
    \rtype{\Uni_\K}(\sigma,\sigma w_{[R']})$ iff $\T \vdash R'
    \sqsubseteq R$.
  \end{enumerate}
\end{lemma}
\begin{proof}
  For~\ref{tmpt:lem-Unitype-KBtype-ABox-1} assume, first, $B$ is a
  concept name, then the proof straightforwardly follows from the
  definition of $\Uni_\K$. Let $B= \exists R$ for a role $R$, we show
  the ``only if'' direction. By the definition of $\Uni_\K$ it follows
  either $a \leadsto_\K w_{[R']}$ for some role $R'$ such that $\T
  \vdash R' \sqsubseteq R$ or $\K \vdash R(a,b)$ for some $b \in
  N_a$. In the first case $\K \vdash \exists R'(a)$ and $\T \vdash
  \exists R' \ISA B$ by the definition of $\leadsto$. It is then
  immediate that $\A \models B'(a)$ and $\T \vdash B' \ISA B$ for some
  concept $B'$. In the second case, there is a role $R''$ such that
  $\A \models R''(a,b)$ and $\T \vdash R'' \sqsubseteq R$, so the
  result follows with $B'=\exists R''$. The ``if'' direction is
  similar using the definition of $\Uni_\K$ and $\leadsto$, which
  concludes the proof of~\ref{tmpt:lem-Unitype-KBtype-ABox-1}. The
  proof of~\ref{tmpt:lem-Unitype-KBtype-ABox-2} is analogious.

  For~\ref{tmpt:lem-Unitype-KBtype-ABox-3} assume, first, $B$ is a
  concept name, then the proof straightforwardly follows from the
  definition of $\Uni_\K$. Let $B= \exists S$ for a role $S$, we, first,
  show the ``only if'' direction. It follows there exists $\sigma' \in
  \Delta^{\Uni_{\K}}$ such that $(\sigma w_{[R]}, \sigma') \in
  S^{\Uni_\K}$. From the definition of $\Uni_\K$ it should be clear that
  either $\delta' = \delta$ and $\T \vdash R \sqsubseteq S^-$, or
  $\sigma'=\sigma w_{[R]} w_{[R']}$ for a role $R'$ such that $w_{[R]}
  \leadsto_\K w_{[R']}$ and $\T \vdash R' \sqsubseteq S$. Then, from
  $w_{[R]} \leadsto_\K w_{[R']}$ we can also conclude $\T \vdash
  \exists R^- \sqsubseteq \exists R'$. One can see that in the both
  cases above it follows $\T \vdash \exists R^- \sqsubseteq \exists
  S$, which concludes the proof of the ``only if'' direction. The
  ``if'' direction is similar using the definition of $\Uni_\K$ and
  $\leadsto$.
\end{proof}
\begin{lemma}\label{lem:types-chase-ABox}
  Let $\tup{\T, \A}$ and $\tup{\T', \A'}$ be the KBs, such that:
  \begin{enumerate}[thmparts, start=1]
  \item \label{tmpt:lem:types-chase-ABox-prereq-1} $\T \subseteq \T'$,
  \item \label{tmpt:lem:types-chase-ABox-prereq-2} $\A \models B(a)$
    implies $\A' \models B(a)$ and $\A \models R(a,b)$ implies $\A'
    \models R(a,b)$, for all $a,b \in N_a$, concepts $B$ and roles $R$.
  \end{enumerate}
  Then, for each $\sigma \in \Delta^{\Uni_{\tup{\T,\A}}}$ there exists
  $\delta \in \Delta^{\Uni_{\tup{\T',\A'}}}$ such that%
  \begin{enumerate}[thmparts]
  \item \label{tmpt:lem:types-chase-ABox-1}
    $\ttype{\Uni_{\tup{\T,\A}}}(\sigma) \subseteq
    \ttype{\Uni_{\tup{\T',\A}}}(\delta)$,
  \item \label{tmpt:lem:types-chase-ABox-2}
    $\rtype{\Uni_{\tup{\T,\A}}}(a,\sigma) \subseteq
    \rtype{\Uni_{\tup{\T',\A}}}(a, \delta)$ for all $a \in N_a$.
  \end{enumerate}
\end{lemma}
\begin{proof} Consider, first, the case $\sigma = b \in N_a$, then set
  $\delta=b$ and we show~\ref{tmpt:lem:types-chase-ABox-1}. Consider
  $B \in \ttype{\Uni_{\tup{\T,\A}}}(\sigma)$, it follows by
  Lemma~\ref{lem:Unitype-KBtype-ABox} $\A \models B'(b)$ and $\T \vdash
  B' \sqsubseteq B$, for some concept $B'$. Then,
  by~\ref{tmpt:lem:types-chase-ABox-prereq-1} it follows $\T' \vdash
  B' \sqsubseteq B$ and by ~\ref{tmpt:lem:types-chase-ABox-prereq-2}
  it follows $\A' \models B'(b)$, therefore, by
  Lemma~\ref{lem:Unitype-KBtype-ABox} we obtain $B \in
  \ttype{\Uni_{\tup{\T',\A}}}(\delta)$. The proof
  for~\ref{tmpt:lem:types-chase-ABox-2} is analogious.

  Now, assume the lemma holds for $\sigma' \in \dom[\Uni_{\tup{\T,
      \A}}]$; we show it also holds for $\sigma = \sigma'w_{[R]} \in
  \dom[\Uni_{\tup{\T, \A}}]$ for a role $R$.  By the definition of
  $\Uni_{\tup{\T, \A}}$ it follows $\tail(\sigma') \leadsto_{\tup{\T,
      \A}} w_{[R]}$ and so $\exists R \in
  \ttype{\Uni_{\tup{\T,\A}}}(\sigma')$. By
  Lemma~\ref{lem:Unitype-KBtype-ABox} it follows
  \begin{align}
    \label{eq:lem:types-chase-ABox-1} \T \vdash \exists R^-
    \sqsubseteq B \text{ for each } B
    \in \ttype{\Uni_{\tup{\T,\A}}}(\sigma)\\
    \label{eq:lem:types-chase-ABox-2} \T \vdash R \sqsubseteq Q \text{
      for each }Q \in \ttype{\Uni_{\tup{\T,\A}}}(a,\sigma)
  \end{align}
  On the other hand, observe by our induction hypothesis that there
  exists $\delta' \in \Delta^{\tup{\T', \A'}}$ such that
  $\ttype{\Uni_{\tup{\T,\A}}}(\sigma') \subseteq
  \ttype{\Uni_{\tup{\T',\A'}}}(\delta')$; therefore, $\exists R \in
  \ttype{\Uni_{\tup{\T',\A'}}}(\delta')$.  It follows there exists
  $\delta'' \in \dom[\Uni_{\tup{\T \cup \T',\A}}]$ such that
  $(\delta', \delta'') \in R^{\Uni_{\tup{\T',\A'}}}$. We select
  $\delta$ (for $\sigma$) equal to $\delta''$;
  using~\eqref{eq:lem:types-chase-ABox-1},~\ref{tmpt:lem:types-chase-ABox-prereq-1}
  and Lemma~\ref{lem:Unitype-KBtype-ABox} one can easily
  show~\ref{tmpt:lem:types-chase-ABox-1}, and
  using~\eqref{eq:lem:types-chase-ABox-2}
  and~\ref{tmpt:lem:types-chase-ABox-prereq-1} one can
  show~\ref{tmpt:lem:types-chase-ABox-2}.
\end{proof}
\begin{lemma}\label{lem:Unitype-KBtype}
  Let $\K = \tup{\T,\A}$ and assume $a \leadsto_\K w_{[R]}$ for some
  basic role $R$. Then there exists a basic concept $B$, such that $\A
  \models B(a)$, and:
  \begin{enumerate}[thmparts, start=1]
  \item \label{tmpt:lem-Unitype-KBtype-0} $o \leadsto_{\tup{\T,B(o)}}
    w_{[R]}$;
  \item \label{tmpt:lem-Unitype-KBtype-1} $\ttype{\Uni_\K}(a w_{[R]})
    \subseteq \ttype{\Uni_{\tup{\T,B(o)}}} (ow_{[R]})$;
  \item \label{tmpt:lem-Unitype-KBtype-2} $\rtype{\Uni_\K}(a,a
    w_{[R]}) \subseteq \rtype{\Uni_{\tup{\T,B(o)}}}(o,ow_{[R]})$.
  \end{enumerate}
\end{lemma}


\begin{proof}
  Consequence of Lemma~\ref{lem:Unitype-KBtype-ABox}.
\end{proof}
\begin{lemma}\label{lem:types-chase}
  Let $\A$ be an ABox, $\BB$ a set of basic concepts, and $\T,\T'$
  TBoxes. Let $\B=\tup{\T,\{B(o) \mid B \in \BB\}}$, and assume $y \in
  \dom[\Uni_{\B}]$. If $\sigma \in \dom[\Uni_{\tup{\T \cup \T',\A}}]$
  and $\BB \subseteq \ttype{\Uni_{\tup{\T \cup \T',\A}}}(\sigma)$,
  then there exists $\delta \in \dom[\Uni_{\tup{\T \cup \T',\A}}]$
  such that
  \begin{description}
  \item[(i)] $\ttype{\Uni_\B}(y) \subseteq \ttype{\Uni_{\tup{\T \cup
          \T',\A}}}(\delta)$
  \item[(ii)] $\rtype{\Uni_\B}(o,y) \subseteq \rtype{\Uni_{\tup{\T
          \cup \T',\A}}}(\sigma, \delta)$
  \end{description}
\end{lemma}
\begin{proof} Straightforward consequence of
  Lemma~\ref{lem:types-chase-ABox}.
\end{proof}

\begin{lemma}\label{lem:XiBprime-SigmaB}
  For each $\sigma \in \dom[\Uni_{\tup{\T_1 \cup \T_{12},\A}}]$ and
  $B' \in \ttype[\Xi]{\Uni_{\tup{\T_1 \cup \T_{12},\A}}}(\sigma)$ one
  of the following holds:
  \begin{enumerate}[thmparts, start=1]
  \item \label{tmpt:lem-XiBprime-SigmaB-1} there exists a concept $B$
    over $\Sigma$ such that $B \in \ttype{\Uni_{\tup{\T_1 \cup
          \T_{12},\A}}}(\sigma)$ and $\T_{12} \vdash B \ISA B'$;
  \item \label{tmpt:lem-XiBprime-SigmaB-2} $\ttype{\Uni_{\tup{\T_1
          \cup \T_{12},\A}}}(\sigma)=\{B'\}$.
  \end{enumerate}
\end{lemma}
\begin{proof} Using Lemma~\ref{lem:Unitype-KBtype-ABox}, and
  considering the structure of $\T_1 \cup \T_{12}$ with $\Sigma \cap
  \Xi = \emptyset$. The case~\ref{tmpt:lem-XiBprime-SigmaB-2} occurs,
  when $\tail(\sigma)=w_{[Q]}$ for a role $Q$ is over $\Xi$.
\end{proof}

\begin{lemma}\label{lem:kr-iff-b-cons}
  A \dlliter KB $\tup{\T,\A}$ is consistent iff
  \begin{enumerate}[thmparts, start=1]
  \item \label{tmpt:lem-kr-iff-b-cons-conc} $B,B'$ is $\T$-consistent
    for each pair of basic concepts $B, B'$ and each $a \in \Ind(\A)$
    such that $\A \models B(a)$ and $\A \models B'(a)$;
  \item \label{tmpt:lem-kr-iff-b-cons-role} $R,R'$ is $\T$-consistent
    for each pair of roles $R, R'$ and each $a,b \in \Ind(\A)$ such
    that $\A \models R(a,b)$ and $\A \models R'(a,b)$
  \end{enumerate}
\end{lemma}
\begin{proof}
  ($\Rightarrow$) Assume \ref{tmpt:lem-kr-iff-b-cons-conc} is
  violated, so there exist $B_1, B_2$ and $a \in \Ind(\A)$ such that
  $\A \models B_1(a)$, $\A \models B_2(a)$, and $\tup{\T, \{B_1(o),
    B_2(o)\}}$ is inconsistent. It follows that $\Uni_{\tup{\T,
      \{B_1(o), B_2(o)\}}}$ is not a model of $\tup{\T, \{B_1(o),
    B_2(o)\}}$, so there is $\delta \in \dom[\Uni_{\tup{\T, \{B_1(o),
      B_2(o)\}}}]$ and a disjointness assertion $B \AND C \ISA \bot
  \in \T$ (note that inclusion assertions $B \ISA C \in \T$ cannot
  cause inconsistency) such that $B, C \in \ttype{\Uni_{\tup{\T,
        \{B_1(o), B_2(o)\}}}}(\delta)$. Obviously, $\{B_1, B_2\}
  \subseteq \ttype{\Uni_{\tup{\T, \A}}}(a)$, then by
  Lemma~\ref{lem:types-chase} we obtain $\delta \in
  \dom[\Uni_{\tup{\T, \A}}]$ such that $B, C \in \ttype{\Uni_{\tup{\T,
        \A}}}(\delta)$. Hence, $\Uni_{\tup{\T,\A}}$ is not a model of
  $\tup{\T, \A}$, which contradicts Claim~\ref{th:uni-sat-kb} since
  $\tup{\T, \A}$ is consistent.

  It can be also the case that $\Uni_{\tup{\T, \{B_1(o), B_2(o)\}}}$
  is inconsistent due to the disjointness assertion $R \AND Q \ISA
  \bot \in \T$. Then the proof is similar using
  Lemmas~\ref{lem:types-chase} and Claim~\ref{th:uni-sat-kb}.

  Assume now \ref{tmpt:lem-kr-iff-b-cons-role} is violated, the proof
  is a straightforward modification of the proof above.

  ($\Leftarrow$) The proof is analogous to ($\Rightarrow$).
\end{proof}
\begin{lemma}\label{lem:kr-to-all-b-cons}
  If a KB $\tup{\T,\A}$ is consistent, then for all $\delta, \sigma
  \in \Delta^{\Uni_{\tup{\T,\A}}}$,
  \begin{enumerate}[thmparts, start=1]
  \item $B$ is $\T$-consistent for each $B \in
    \ttype{\Uni_{\tup{\T,\A}}}(\delta)$;
  \item $R$ is $\T$-consistent for each $R \in
    \rtype{\Uni_{\tup{\T,\A}}}(\delta, \sigma)$.
  \end{enumerate}
\end{lemma}
\begin{proof}
  Similar to Lemma \ref{lem:kr-iff-b-cons}.
\end{proof}

\subsection{Homomorphism Lemmas}

Here we present a series of important lemmas used in the proof the
main results in the following sections.

\begin{lemma}\label{lem:from-model-types-to-kb-types-ABox}
  Assume a mapping $\M = (\Sigma, \Xi, \T_{12})$, ABoxes $\A$ and
  $\A'$ over, respectively, $\Sigma$ and $\Xi$, and a $\Xi$-TBox over
  $\T_2$. If $\Uni_{\tup{\T_{12},\A}}$ is $\Xi$ homomorphically
  embeddable into $\Uni_{\A'}$, then $\Uni_{\tup{\T_2 \cup
      \T_{12},\A}}$ is $\Xi$ homomorphically embeddable into
  $\Uni_{\tup{\T_2,\A'}}$.
\end{lemma}
\begin{proof} Consider the $\Xi$ homomorphism $h :
  \Delta^{\Uni_{\tup{\T_{12},\A}}} \mapsto \Delta^{\Uni_{\A'}}$ from
  $\Uni_{\tup{\T_{12},\A}}$ to $\Uni_{\A'}$, we are going to construct
  the $\Sigma$ homomorphism $h' : \Uni_{\tup{\T_2 \cup \T_{12},\A}}
  \mapsto \Uni_{\tup{\T_2,\A'}}$ from $\Uni_{\tup{\T_2 \cup
      \T_{12},\A}}$ to $\Uni_{\tup{\T_2,\A'}}$. Initially, we define
  $h'(a)=a$, let us immediately verify that
  $\ttype[\Xi]{\Uni_{\tup{\T_2 \cup \T_{12},\A}}}(a) \subseteq
  \ttype[\Xi]{\Uni_{\tup{\T_2,\A'}}}(h'(a))$. Notice that by the
  definition of $h$ we have:
  \begin{align}
    \label{eq:hom-uni-t12-a-to-uni-a-prime} \ttype[\Xi]{\Uni_{\tup{\T_{12},\A}}}(a) &\subseteq \ttype[\Xi]{\Uni_{\A'}}(h(a)),\\
    \label{eq:hom-a-equal-hom-a-prime} h(a) & = h'(a).
  \end{align}
  Let $C \in \ttype[\Xi]{\Uni_{\tup{\T_2 \cup \T_{12},\A}}}(a)$, it
  follows by Lemma \ref{lem:Unitype-KBtype-ABox}
  \ref{tmpt:lem-Unitype-KBtype-ABox-1} there exists $B$ over $\Sigma$,
  such that $\A \models B(a)$ and $\T \cup \T' \vdash B \ISA C$. Taking
  into account the shape of $\T_2$ and $\T_{12}$, t follows also there
  exists $D$ over $\Xi$ such that $\T_{12} \vdash B \sqsubseteq D$ and
  $\T_2 \vdash D \sqsubseteq C$. Observe that $B \in
  \ttype{\Uni_{\tup{\T_{12},\A}}}(a) $, then by
  Lemma~\ref{lem:Unitype-KBtype-ABox}~\ref{tmpt:lem-Unitype-KBtype-ABox-1}
  and~\ref{tmpt:lem-Unitype-KBtype-ABox-3} it follows $D \in
  \ttype[\Xi]{\Uni_{\tup{\T_{12},\A}}}(a)$ and taking into account
  \eqref{eq:hom-uni-t12-a-to-uni-a-prime} and
  \eqref{eq:hom-a-equal-hom-a-prime} we conclude $D \in
  \ttype[\Xi]{\Uni_{\A'}}(h'(a))$. Finally, using again
  Lemma~\ref{lem:Unitype-KBtype-ABox}~\ref{tmpt:lem-Unitype-KBtype-ABox-1}
  and~\ref{tmpt:lem-Unitype-KBtype-ABox-3} we obtain $C \in
  \ttype[\Xi]{\Uni_{\tup{\T_2, \A'}}}(h'(a))$. The proof that
  $\rtype[\Xi]{\Uni_{\tup{\T_2 \cup \T_{12},\A}}}(a,b) \subseteq
  \rtype[\Xi]{\Uni_{\tup{\T_2,\A'}}}(h'(a),h'(b))$ for all constants $a$
  and $b$ is analogious.

  Now we show how to define $h'$ for $\sigma=a w_{[R]} \in
  \gpath(\tup{\T_2 \cup \T_{12},\A})$. It follows $a
  \leadsto_{\tup{\T_2 \cup \T_{12},\A}} w_{[R]}$, then two cases are
  possible:
  \begin{enumerate}[series=reqs, label=\textbf{(\Roman*)}]
  \item \label{req:lem:from-model-types-to-kb-types-ABox-gen-role-1}
    $R$ is over $\Sigma$;%
  \item \label{req:lem:from-model-types-to-kb-types-ABox-gen-role-2}
    $R$ is over $\Xi$.
  \end{enumerate}
  In case~\ref{req:lem:from-model-types-to-kb-types-ABox-gen-role-1}
  it follows $a \leadsto_{\tup{\T_{12}, \A}} w_{[R]}$ and by the
  condition of the current lemma it follows there is $\delta \in
  \Delta^{\Uni_{\A'}}$ such that:
  \begin{align}
    \label{eq:lem:from-model-types-to-kb-types-ABox-3}\ttype[\Xi]{\Uni_{\tup{\T_{12},
          \A}}}(a w_{[R]})&
    \subseteq \ttype[\Xi]{\Uni_{\A'}}(\delta),\\%
    \label{eq:lem:from-model-types-to-kb-types-ABox-4}
    \rtype[\Xi]{\Uni_{\tup{\T_{12}, \A}}}(a,a w_{[R]})&
    \subseteq \rtype[\Xi]{\Uni_{ \A'}}(a,\delta).
  \end{align}
  Then, using Lemma~\ref{lem:types-chase-ABox} (with $\A'=\A$,
  $\T=\emptyset$, $\T' = \T_2$) we obtain $\gamma \in
  \Delta^{\Uni_{\tup{\T_2, \A'}}}$ such that
  \begin{align}
    \label{eq:lem:from-model-types-to-kb-types-ABox-5}
    \ttype{\Uni_{\A'}}(\delta) & \subseteq \ttype{\Uni_{\tup{\T_2, \A'}}}(\gamma),\\
    \label{eq:lem:from-model-types-to-kb-types-ABox-6}
    \rtype{\Uni_{ \A'}}(a,\delta) & \subseteq
    \rtype{\Uni_{\tup{\T_2,\A'}}}(a, \gamma).
  \end{align}
  Now define $h'(\sigma) = \gamma$; we need to show
  \begin{align}
    \label{eq:lem:from-model-types-to-kb-types-ABox-7}
    \ttype[\Xi]{\Uni_{\tup{\T_2 \cup \T_{12},\A}}}(\sigma)& \subseteq
    \ttype[\Xi]{\Uni_{\tup{\T_2,\A'}}}(\gamma),\\
    \label{eq:lem:from-model-types-to-kb-types-ABox-8}
    \rtype[\Xi]{\Uni_{\tup{\T_2 \cup \T_{12},\A}}}(a,\sigma)
    &\subseteq \rtype[\Xi]{\Uni_{\tup{\T_2,\A'}}}(a,\gamma).
  \end{align}
  For \eqref{eq:lem:from-model-types-to-kb-types-ABox-7} consider the
  set $\vec{B} = \{ B \text{ over } \Xi \mid \T_{12} \vdash \exists
  R^- \sqsubseteq B\}$ and observe that $\vec{B} \subseteq
  \ttype[\Xi]{\Uni_{\tup{\T_{12}, \A}}}(a w_{[R]})$ and also by
  Lemma~\ref{lem:Unitype-KBtype-ABox} and the structure of $\T_2 \cup
  \T_{12}$, for each $B' \in \ttype[\Xi]{\Uni_{\tup{\T_{12}, \A}}}(a
  w_{[R]})$ there exists $B \in \vec{B}$ such that $\T_2 \vdash B
  \sqsubseteq B'$. By
  \eqref{eq:lem:from-model-types-to-kb-types-ABox-3} and
  \eqref{eq:lem:from-model-types-to-kb-types-ABox-5} we obtain
  $\vec{B} \subseteq \ttype[\Xi]{\Uni_{\tup{\T_2,\A'}}}(\gamma)$; then
  using Lemma~\ref{lem:Unitype-KBtype-ABox} it can be easily verified
  $B' \in \ttype[\Xi]{\Uni_{\tup{\T_2,\A'}}}(\gamma)$ for all $B'$ as
  above, which concludes the proof of
  \eqref{eq:lem:from-model-types-to-kb-types-ABox-7}. The
  \eqref{eq:lem:from-model-types-to-kb-types-ABox-8} is analogious
  using \eqref{eq:lem:from-model-types-to-kb-types-ABox-4},
  \eqref{eq:lem:from-model-types-to-kb-types-ABox-6}, and the set
  $\vec{S} = \{ S \text{ over } \Xi \mid \T_{12} \vdash R \sqsubseteq
  S\}$.

  Consider the
  case~\ref{req:lem:from-model-types-to-kb-types-ABox-gen-role-2},
  using Lemma~\ref{lem:Unitype-KBtype-ABox} and the structure of $\T_2
  \cup \T_{12}$ and $\A$, one can show:
  \begin{align}
    \label{eq:lem:from-model-types-to-kb-types-ABox-9} & \exists R \in \ttype[\Xi]{\Uni_{\tup{\T_2 \cup \T_{12}, \A}}}(a),\\
    \label{eq:lem:from-model-types-to-kb-types-ABox-10} \T_2 \vdash
    \exists R^- \sqsubseteq B& \text{ for all } B \in
    \ttype[\Xi]{\Uni_{\tup{\T_2 \cup \T_{12}, \A}}}(\sigma),\\
    \label{eq:lem:from-model-types-to-kb-types-ABox-11} \T_2 \vdash R
    \sqsubseteq S& \text{ for all } S \in \rtype[\Xi]{\Uni_{\tup{\T_2
          \cup \T_{12}, \A}}}(a,\sigma).
  \end{align}
  Provided that the homomorphism $h'$ is defined for $a$, it follows
  $\exists R \in \ttype{\Uni_{\tup{\T_2, \A'}}}(h'(a))$, therefore,
  there exists $\gamma \in \dom[\Uni_{\tup{\T_2, \A'}}]$ such that $R
  \in \rtype{\Uni_{\tup{\T_2, \A'}}}(h'(a), \gamma)$. Now define
  $h'(\sigma) = \gamma$; we need to show
  \begin{align}
    \label{eq:lem:from-model-types-to-kb-types-ABox-16}
    \ttype[\Xi]{\Uni_{\tup{\T_2 \cup \T_{12},\A}}}(\sigma)& \subseteq
    \ttype{\Uni_{\tup{\T_2,\A'}}}(\gamma),\\
    \label{eq:lem:from-model-types-to-kb-types-ABox-17}
    \rtype[\Xi]{\Uni_{\tup{\T_2 \cup \T_{12},\A}}}(a,\sigma)
    &\subseteq \rtype{\Uni_{\tup{\T_2,\A'}}}(a,\gamma).
  \end{align}
  For~\eqref{eq:lem:from-model-types-to-kb-types-ABox-16}
  consider~\eqref{eq:lem:from-model-types-to-kb-types-ABox-10} and
  Lemma~\ref{lem:Unitype-KBtype-ABox}; similarly,
  for~\eqref{eq:lem:from-model-types-to-kb-types-ABox-17}
  consider~\eqref{eq:lem:from-model-types-to-kb-types-ABox-11}.

  Assume now $\sigma=\sigma' w_{[R]}$ and the homomorphism from
  $\Uni_{\tup{\T_2 \cup \T_{12},\A}}$ to $\Uni_{\tup{\T_2,\A'}}$ is
  defined for $\sigma'$. The proof is done in the same way as for the
  case~\ref{req:lem:from-model-types-to-kb-types-ABox-gen-role-2}, all
  the statements are valid if one substitutes $a$ by $\sigma'$.
\end{proof}

Let $\M= (\Sigma, \Xi, \T_{12})$ be a mapping, and, $\T_1$ and $\T_2$,
respectively, $\Sigma$- and $\Xi$-TBoxes. Define KBs $\srkb=\tup{\T_1
  \cup \T_{12}, \{B(o)\}}$ and $\tgkb=\tup{\T_2 \cup \T_{12},
  \{B(o)\}}$ for a basic concept $B$ over $\Sigma$. We slightly abuse
the notation and write $\srkb[\A]$ to denote the KB $\tup{\T_1 \cup
  \T_{12}, \A}$ for a given ABox $\A$, analogously we use $\tgkb[\A]$
to denote $\tup{\T_2 \cup \T_{12}, \A}$.  We show
\begin{lemma}\label{lem:c2}
  Let $\A$ be an ABox over $\Sigma$ and assume for each concept $B$,
  role $R$, and all $\sigma, \delta \in \dom[\Uni_{\srkb[\A]}]$ such
  that
 \begin{enumerate}[thmparts, start=1]
  \item \label{req:c2-prereq-1} $B \in \ttype[\Sigma]{\srkb[\A]}(\sigma)$, 
  \item \label{req:c2-prereq-2} $R \in
    \rtype[\Sigma]{\srkb[\A]}(\sigma, \delta)$,
\end{enumerate}
the following conditions hold
  \begin{enumerate}[thmparts]
  \item \label{req:c2-o-subs} $\ttype[\Xi]{\tgkb}(o) \subseteq
    \ttype{\srkb}(o)$;
  \item \label{req:c2-role-subs} $\T_2 \cup \T_{12} \vdash R \ISA R'$
    implies $\T_1 \cup \T_{12} \vdash R \ISA R'$ for all roles $R'$
    over $\Xi$;
  \item \label{req:c2-gen} for each role $R$ such that $o
    \leadsto_{\tgkb} w_{[R]}$ there exists $y \in \dom[\Uni_{\srkb}]$
    such that
    \begin{enumerate}[label=\textbf{(\alph*)}]
    \item \label{req:c2-gen-conc-subs} $\ttype[\Xi]{\Uni_{\tgkb}}(ow_{[R]})
      \subseteq \ttype{\Uni_{\srkb}} (y)$,
    \item \label{req:c2-gen-role-subs}
      $\rtype[\Xi]{\Uni_{\tgkb}}(o,ow_{[R]}) \subseteq
      \rtype{\Uni_{\srkb}} (o,y)$.
    \end{enumerate}
  \end{enumerate}
  Then $\Uni_{\tgkb[\A]}$ is finitely homomorphically embeddable into
  $\Uni_{\srkb[\A]}$.
\end{lemma}
\begin{proof}
  Let $\A$ as above and assume the condition of the lemma are
  satisfied. We build a mapping $h$ from $\gpath(\tgkb[\A])$ to
  $\gpath(\srkb[\A])$ such that for any finite subinterpretation of
  $\Uni_{\tgkb[\A]}$ the restriction of $h$ to it is a homomorphism to
  $\Uni_{\srkb[\A]}$. Initially, we define $h(a)=a$, let us
  immediately verify that $\ttype{\Uni_{\tgkb[\A]}}(a) \subseteq
  \ttype{\Uni_{\srkb[\A]}}(a)$. Let $C \in
  \ttype{\Uni_{\tgkb[\A]}}(a)$, it follows by Lemma \ref{lem:Unitype-KBtype-ABox}
  \ref{tmpt:lem-Unitype-KBtype-ABox-1} there exists $B$ over $\Sigma$ such
  that $\A \models B(a)$ and $\T_2 \cup \T_{12} \vdash B \ISA
  C$. Observe that $B \in \ttype{\Uni_{\srkb[\A]}}(a)$; now if $C$ is
  over $\Sigma$ it follows $C=B$, so $C \in
  \ttype{\Uni_{\srkb[\A]}}(a)$ and the proof is done. Otherwise, $C
  \in \stype[\Xi]{\tgkb}(o)$, then by \ref{req:c2-o-subs} $C \in
  \stype{\srkb}(o)$, so $\T_1 \cup \T_{12} \vdash B \ISA C$. Finally,
  using Lemma \ref{lem:Unitype-KBtype-ABox}~\ref{tmpt:lem-Unitype-KBtype-ABox-1} obtain $C \in
  \ttype{\Uni_{\srkb[\A]}}(a)$. The proof of
  $\rtype{\Uni_{\tgkb[\A]}}(a,b) \subseteq
  \rtype{\Uni_{\srkb[\A]}}(a,b)$ is analogous using Lemma
  \ref{lem:Unitype-KBtype-ABox} \ref{tmpt:lem-Unitype-KBtype-ABox-2} and current
  \ref{req:c2-role-subs}.

  Now we show how to define $h$ for $\sigma=a  w_{[R]} \in
  \gpath(\tgkb[\A])$. It follows $a \leadsto_{\tgkb[\A]} w_{[R]}$,
  then by Lemma \ref{lem:Unitype-KBtype} (with $\K=\tgkb[\A]$) there
  exists $B$ over $\Sigma$ such that $\A \models B(a)$, $o
  \leadsto_{\tgkb} w_{[R]}$, and
  \begin{align}
    \label{eq:Unitargtype-KBtype1}\ttype{\Uni_{\tgkb[\A]}}(a w_{[R]})&
    \subseteq \ttype{\Uni_{\tgkb}}(ow_{[R]}) \\
    \label{eq:Unitargtype-KBtype2} \rtype{\Uni_{\tgkb[\A]}}(a,a
    w_{[R]})& \subseteq \rtype{\Uni_{\tgkb}}(o,ow_{[R]}).
  \end{align}
  We are going to show now there exists $y \in \Delta^{\Uni_{\srkb}}$
  such that
  \begin{align}
    \label{eq:KBtarg-KBsrc-conc} \ttype{\Uni_{\tgkb}}(ow_{[R]}) &\subseteq \ttype{\Uni_{\srkb}}(y) \text{ and }\\
    \label{eq:KBtarg-KBsrc-role} \rtype{\Uni_{\tgkb}}(o,ow_{[R]})&
    \subseteq \rtype{\Uni_{\srkb}}(o,y).
  \end{align}
  Assume, first, $\ttype[\Sigma]{\Uni_{\tgkb}}(ow_{[R]}) = \emptyset$,
  then also $\rtype[\Sigma]{\Uni_{\tgkb}}(o,ow_{[R]})=\emptyset$; it
  remains to observe that from $\A \models B(a)$ it
  follows~\ref{req:c2-prereq-1} is satisfied with $\sigma=a$, then by
  \ref{req:c2-gen} we obtain $y$ satisfying
  \eqref{eq:KBtarg-KBsrc-conc} and \eqref{eq:KBtarg-KBsrc-role}.

  Assume now $\ttype[\Sigma]{\Uni_{\tgkb}}(ow_{[R]}) \neq \emptyset$,
  it follows $B=\exists R$, $\ttype[\Sigma]{\Uni_{\tgkb}}(ow_{[R]}) =
  \{\exists R^-\}$, and $\rtype[\Sigma]{\Uni_{\tgkb}}(o,ow_{[R]}) =
  \{R\}$. Since $B=\exists R$, there must exists a role $Q$ such that
  $o \leadsto_{\srkb} w_{[Q]}$ and $\T_1 \cup \T_{12} \vdash Q \ISA
  R$, we choose $w_{[Q]}$ to be the required $y$; it is immediate to
  see $\ttype[\Sigma]{\Uni_{\tgkb}}(ow_{[R]}) \subseteq
  \ttype{\Uni_{\srkb}}(y)$, and
  $\rtype[\Sigma]{\Uni_{\tgkb}}(o,ow_{[R]}) \subseteq
  \rtype{\Uni_{\srkb}}(o,y)$. To prove also
  $\ttype[\Xi]{\Uni_{\tgkb}}(ow_{[R]}) \subseteq
  \ttype{\Uni_{\srkb}}(y)$ and $\rtype[\Xi]{\Uni_{\tgkb}}(o,ow_{[R]})
  \subseteq \rtype{\Uni_{\tgkb}}(o,y)$ we are going to use
  \ref{req:c2-o-subs} and \ref{req:c2-role-subs}, but we need $\exists
  R^- \in \ttype[\Sigma]{\Uni_{\srkb[\A]}}(\sigma)$ and $R \in
  \rtype[\Sigma]{\Uni_{\srkb[\A]}}(\sigma, \delta)$ for some $\sigma,
  \delta \in \Delta^{\Uni_{\srkb[\A]}}$. To get the latter two facts
  it is sufficient to notice $\exists R \in
  \ttype{\Uni_{\tgkb[\A]}}(a)$ (since $a \leadsto_{\tgkb} w_{[R]}$)
  and $\ttype{\Uni_{\tgkb[\A]}}(a) \subseteq
  \ttype{\Uni_{\srkb[\A]}}(a)$ proven above.

  The proof of $\ttype[\Xi]{\Uni_{\tgkb}}(ow_{[R]}) \subseteq
  \ttype{\Uni_{\srkb}}(y)$ is as follows: assume $B' \in
  \ttype[\Xi]{\Uni_{\tgkb}}(ow_{[R]})$, then since $R$ is over
  $\Sigma$ it follows $B' \in \ttype[\Xi]{\Uni_{\tgkb[\exists
      R^-]}}(o)$. By $\exists R^- \in
  \ttype[\Sigma]{\Uni_{\srkb[\A]}}(\sigma)$ and \ref{req:c2-o-subs}
  obtain $B' \in \ttype[\Xi]{\Uni_{\srkb[\exists R^-]}}(o)$, then
  since $\T_1 \cup \T_{12} \vdash Q \ISA R$ it follows
  $\ttype{\Uni_{\srkb[\exists R^-]}}(o) \subseteq
  \ttype{\Uni_{\srkb}}(ow_{[Q]})$ and we obtain $B' \in
  \ttype{\Uni_{\srkb}}(y)$. The proof of
  $\rtype[\Xi]{\Uni_{\tgkb}}(o,ow_{[R]}) \subseteq
  \rtype{\Uni_{\tgkb}}(o,y)$ is analogous using $R \in
  \rtype[\Sigma]{\Uni_{\srkb[\A]}}(\sigma, \delta)$,
  \ref{req:c2-role-subs}, and $\T_1 \cup \T_{12} \vdash Q \ISA R$. We
  finished showing there exists $y \in \dom[\Uni_{\srkb}]$, such that
  \eqref{eq:KBtarg-KBsrc-conc} and \eqref{eq:KBtarg-KBsrc-role}.

  To continue the proof consider $\{B\} \subseteq
  \ttype{\Uni_{\srkb[\A]}}(a)$ and Lemma~\ref{lem:types-chase} (with
  $\T=\T_1 \cup \T_{12}$ and $\T'= \emptyset$) there exists $\delta
  \in \dom[\Uni_{\srkb[\A]}]$ such that $\ttype{\Uni_{\srkb}}(y)
  \subseteq \ttype{\Uni_{S_\A}}(\delta)$ and $\rtype{\Uni_{\srkb}}
  (o,y) \subseteq \rtype{\Uni_{\srkb[\A]}}(a,\delta)$. It follows now
  using \eqref{eq:Unitargtype-KBtype1} and
  \eqref{eq:KBtarg-KBsrc-conc} that $\ttype{\Uni_{\tgkb[\A]}}(a
  w_{[R]}) \subseteq \ttype{\Uni_{\srkb[\A]}}(\delta)$. Analogously
  using \eqref{eq:Unitargtype-KBtype2} and
  \eqref{eq:KBtarg-KBsrc-role} one obtains
  $\rtype{\Uni_{\tgkb[\A]}}(a,a w_{[R]}) \subseteq
  \rtype{\Uni_{\srkb[\A]}}(a,\delta)$.

  We show how to define the homomorphism for $\sigma w_{[R]} \in
  \gpath(\tgkb[\A])$ with $\tail(\sigma)=w_{[R']}$ given that the
  homomorphism for $h(\sigma)$ is defined. It follows $w_{[R']}
  \leadsto_{\tgkb[\A]} w_{[R]}$ and by definition of $\leadsto$ and
  the structure of $\T_2 \cup \T_{12}$ we obtain $\T_2 \cup \T_{12}
  \vdash \exists R'^- \ISA \exists R$ and $R$ is a $\Xi$ role
  \emph{different from} $R^-$. By Lemma \ref{lem:Unitype-KBtype-ABox} it also
  follows $\{\exists R'^-,\SOMET{R}\} \subseteq
  \ttype{\Uni_{\tgkb[\A]}}(\sigma)$. Since $h$ is a homomorphism,
  $\{\exists R'^-,\SOMET{R}\} \subseteq
  \ttype{\Uni_{\srkb[\A]}}(\delta)$ for $\delta=h(\sigma) \in
  \dom[\Uni_{\srkb[\A]}]$. We use Lemma \ref{lem:XiBprime-SigmaB} to
  obtain $B$ over $\Sigma$ such that $B \in
  \ttype{\Uni_{\srkb[\A]}}(\delta)$ and $\T_{12} \vdash B \ISA
  \SOMET{R}$. Notice that such $B$ exists: since $\exists R'^-$ and
  $\exists R$ are different concepts, \ref{tmpt:lem-XiBprime-SigmaB-2}
  of Lemma \ref{lem:XiBprime-SigmaB} is excluded, so
  \ref{tmpt:lem-XiBprime-SigmaB-1} holds.

  Then in $\tgkb$ we have that $o \leadsto_{\tgkb} w_{[R]}$ for a
  $\Xi$ role $R$, and the proof continues analogously to the proof for
  the case $\sigma=a w_{[R]}$ above using the conditions
  \textbf{(ii)}, \textbf{(iii)} and Lemmas \ref{lem:types-chase} to
  obtain $\delta'$ in $\dom[\Uni_{\srkb[\A]}]$ such that
  $\ttype{\Uni_{\tgkb[\A]}}(\sigma w_{[R]}) \subseteq
  \ttype{\Uni_{\srkb[\A]}}(\delta')$ and
  $\rtype{\Uni_{\tgkb[\A]}}(\sigma, \sigma w_{[R]}) \subseteq
  \rtype{\Uni_{\srkb[\A]}}(\delta,\delta')$. We assign $h(\sigma
  w_{[R]})=\delta'$.
 
  Thus, we defined the mapping $h$ that is clearly a
  $\Xi$-homomorphism from each finite subinterpretation of
  $\Uni_{\tup{\T_1 \cup \T_{12}, \A}}$ into $\Uni_{\tup{\T_2 \cup
      \T_{12}, \A}}$.
\end{proof}
%

%
\begin{lemma}\label{lem:c1}
  Let $\A$ be an ABox over $\Sigma$ and assume for each concept $B$,
  role $R$, and all $\sigma, \delta \in \dom[\Uni_{\srkb[\A]}]$ such
  that
  \begin{enumerate}[thmparts, start=1]
  \item \label{req:c1-prereq-1} $B \in \ttype[\Sigma]{\srkb[\A]}(\sigma)$,
  \item \label{req:c1-prereq-2} $R \in
    \rtype[\Sigma]{\srkb[\A]}(\sigma, \delta)$
  \end{enumerate}
  the following conditions hold
  \begin{enumerate}[thmparts]
  \item \label{req:c1-o-subs} $\ttype[\Xi]{\Uni_{\srkb}}(o) \subseteq
    \ttype{\Uni_{\tgkb}}(o)$;
  \item \label{req:c1-role-subs} $\T_1 \cup \T_{12} \vdash R \ISA R'$
    implies $\T_2 \cup \T_{12} \vdash R \ISA R'$ for all roles $R'$
    over $\Xi$;
  \item \label{req:c1-gen} for each role $R$ such that $o
    \leadsto_{\Uni_{\srkb}} w_{[R]}$ there exists $y \in
    \dom[\Uni_{\tgkb}]$ such that
    \begin{enumerate}[label=\textbf{(\alph*)}]
    \item \label{req:c1-gen-conc-subs} $\ttype[\Xi]{\Uni_{\srkb}}(ow_{[R]})
      \subseteq \ttype{\Uni_{\tgkb}} (y)$,
    \item \label{req:c1-gen-role-subs} $\rtype[\Xi]{\srkb}(o,ow_{[R]})
      \subseteq \rtype{\Uni_{\tgkb}} (o,y)$.
    \end{enumerate}
  \end{enumerate}
  Then $\Uni_{\srkb[\A]}$ is finitely $\Xi$-homomorphically embeddable
  into $\Uni_{\tgkb[\A]}$.
\end{lemma}
%

%
%
\begin{proof}
  Assume the condition of the lemma is satisfied, and let $\A$ be an
  ABox over $\Sigma$. We build a mapping $h$ from $\gpath(\srkb[\A])$
  to $\gpath(\tgkb[\A])$ such that for any finite subinterpretation of
  $\Uni_{\srkb[\A]}$ the restriction of $h$ to it is a
  $\Xi$-homomorphism to $\Uni_{\T_\A}$. Initially, we define $h(a)=a$,
  let us immediately verify that $\ttype[\Xi]{\Uni_{\srkb[\A]}}(a)
  \subseteq \ttype[\Xi]{\Uni_{\tgkb[\A]}}(a)$. Let $B' \in
  \ttype[\Xi]{\Uni_{\srkb[\A]}}(a)$, it follows by Lemma
  \ref{lem:Unitype-KBtype-ABox} \ref{tmpt:lem-Unitype-KBtype-ABox-1}
  there exists $B$ over $\Sigma$ such that $\A \models B(a)$ and $\T_1
  \cup \T_{12} \vdash B \ISA B'$. Observe that $B \in
  \ttype{\Uni_{\srkb[\A]}}(a)$, then by \ref{req:c1-o-subs} $B' \in
  \ttype{\Uni_{\tgkb}}(o)$, so $\T_2 \cup \T_{12} \vdash B \ISA
  B'$. Finally, using Lemma
  \ref{lem:Unitype-KBtype-ABox}~\ref{tmpt:lem-Unitype-KBtype-ABox-1}
  obtain $B' \in \ttype{\Uni_{\srkb[\A]}}(a)$. The proof of
  $\rtype[\Xi]{\Uni_{\srkb[\A]}}(a,b) \subseteq
  \rtype[\Xi]{\Uni_{\tgkb[\A]}}(a,b)$ is analogious using Lemma
  \ref{lem:Unitype-KBtype-ABox} \ref{tmpt:lem-Unitype-KBtype-ABox-2}
  and current \ref{req:c1-role-subs}.

  Now we show how to define $h$ for $\sigma=aw_{[R]} \in
  \gpath(\srkb[\A])$. It follows $a \leadsto_{\srkb[\A]} w_{[R]}$ and
  by Lemma \ref{lem:Unitype-KBtype} (with $\K=\srkb[\A]$) we obtain
  $B$ over $\Sigma$ such that $\A \models B(a)$, $o \leadsto_{\srkb}
  w_{[R]}$, and
  \begin{align}
    \label{eq:Unisrctype-KBtype-conc}\ttype{\Uni_{\srkb[\A]}}(aw_{[R]})& \subseteq
    \ttype{\Uni_{\srkb}}(ow_{[R]}) \\%
    \label{eq:Unisrctype-KBtype-role}
    \rtype{\Uni_{\srkb[\A]}}(a,aw_{[R]})& \subseteq
    \rtype{\Uni_{\srkb}}(o,ow_{[R]}).
  \end{align}
  Notice that $B \in \stype[\Sigma]{\srkb}(a)$ (that
  is,~\ref{req:c1-prereq-1}), then by \ref{req:c1-gen} there exists
  $y \in \dom[\tgkb]$ such that
  \begin{align}
    \label{eq:KBsrc-KBtarg-conc} \ttype[\Xi]{\Uni_{\srkb}}(w_{[R]}) &\subseteq \ttype{\Uni_{\tgkb}}(y),\\
    \label{eq:KBsrc-KBtarg-role} \rtype[\Xi]{\Uni_{\srkb}}(o,w_{[R]})&
    \subseteq \rtype{\Uni_{\tgkb}}(o,y).
  \end{align}
  Since $\{B\} \subseteq \ttype{\Uni_{\tgkb[\A]}}(a)$, by
  Lemma~\ref{lem:types-chase} (with $\T=\T_2 \cup \T_{12}$ and $\T'=
  \emptyset$) there exists $\delta \in \dom[\Uni_{\tgkb[\A]}]$ such
  that $\ttype{\Uni_{\tgkb}}(y) \subseteq
  \ttype{\Uni_{\srkb[\A]}}(\delta)$ and $\rtype{\Uni_{\tgkb}} (o,y)
  \subseteq \rtype{\Uni_{\tgkb[\A]}}(a,\delta)$. It follows now using
  \eqref{eq:Unisrctype-KBtype-conc} and \eqref{eq:KBsrc-KBtarg-conc}
  that $\ttype[\Xi]{\Uni_{\srkb[\A]}}(a w_{[R]}) \subseteq
  \ttype[\Xi]{\Uni_{\tgkb[\A]}}(\delta)$. Analogously using
  \eqref{eq:Unisrctype-KBtype-role} and \eqref{eq:KBsrc-KBtarg-role}
  one obtains $\rtype[\Xi]{\Uni_{\srkb[\A]}}(a,a w_{[R]})
  \subseteq \rtype[\Xi]{\Uni_{\tgkb[\A]}}(a,\delta)$. We assign
  $h(\sigma)=\delta$.

  We show how to define the homomorphism for $\sigma w_{[R]} \in
  \gpath(\srkb[\A])$ with $\sigma=\sigma' w_{[R']}$ given that the
  homomorphism $h(\sigma)$ and $h(\sigma')$ is defined. It follows
  $w_{[R']} \leadsto_{\srkb[\A]} w_{[R]}$ and it that case $R'$ is
  over $\Sigma$ by the structure of $\T_1 \cup \T_{12}$. Analogously
  to the proof of Lemma \ref{lem:Unitype-KBtype} it can be verified $o
  \leadsto_{\srkb[(\exists R'^-)]} w_{[R]}$ and
  \begin{align}
    \label{eq:Unisrctype-KBtype1}\ttype{\Uni_{\srkb[\A]}}(\sigma
    w_{[R]})& \subseteq
    \ttype{\Uni_{\srkb[(\exists R'^-)]}}(o w_{[R]}) \text{ and} \\
    \label{eq:Unisrctype-KBtype2}
    \rtype{\Uni_{\srkb[\A]}}(\sigma,\sigma w_{[R]})& \subseteq
    \rtype{\Uni_{\srkb[(\exists R'^-)]}}(o,o w_{[R]}).
  \end{align}
  Observe that $\exists R'^- \in
  \ttype[\Sigma]{\Uni_{\srkb[\A]}}(\sigma)$ (that
  is,~\ref{req:c1-prereq-1}), then by \ref{req:c1-gen} there is $y \in
  \dom[\Uni_{\tgkb[(\exists R'^-)]}]$ satisfying
  \ref{req:c1-gen-conc-subs} and \ref{req:c1-gen-role-subs}. Given the
  structure of $\T_2 \cup \T_{12}$ two cases are possible:
  \begin{enumerate}[resume*=reqs]
  \item \label{req:lem-c1-indstep2-1} $y \in
    \dom[\Uni_{\tup{\T_2,\{B(o) \mid B \in \mathbf{B}\}}}]$ for the
    set $\mathbf{B}$ of all concepts $B$ over $\Xi$ such that $\T_{12}
    \vdash \exists R'^- \ISA B$, 
  \begin{align}
    \label{eq:KBtarg-typesubs1} \ttype[\Xi]{\Uni_{\srkb[(\exists R'^-)]}}
    (o w_{[R]}) &\subseteq
    \ttype{\Uni_{\tup{\T_2,\{B(o) \mid B \in \mathbf{B}\}}}}(y), \text{ and }\\
    \label{eq:KBtarg-typesubs2} \rtype[\Xi]{\Uni_{\srkb[(\exists
        R'^-)]}}(o,o w_{[R]}) &\subseteq \rtype{\Uni_{\tup{\T_2,\{B(o)
          \mid B \in \mathbf{B}\}}}}(o,y).
  \end{align}
  \item \label{req:lem-c1-indstep2-2} $o \leadsto_{\tgkb[(\exists
      R'^-)]} w_{[R'^-]}$, 
    \begin{align}
      \label{eq:KBtarg-typesubs3a} y \in \dom[\Uni_{\tup{\T_2,\{B(o)
          \mid B \in \mathbf{B}\}}}], \text{ for the set $\mathbf{B}$
        of} & \text{ all concepts $B$ over $\Xi$ such that $\T_{12} \vdash
        \exists R' \ISA B$},\\%
      \label{eq:KBtarg-typesubs1a} \ttype[\Xi]{\Uni_{\srkb[(\exists
          R'^-)]}} (o w_{[R]}) &\subseteq \ttype{\Uni_{\tup{\T_2,\{B(o)
            \mid B \in
            \mathbf{B}\}}}}(y), \text{ and }\\
    \label{eq:KBtarg-typesubs2a} \rtype[\Xi]{\Uni_{\srkb[(\exists
      R'^-)]}}(o,o w_{[R]}) &\subseteq \rtype{\Uni_{\tgkb[(\exists
      R'^-)]}}(o, o w_{[R'^-]}).
  \end{align}
  \end{enumerate}

  Consider \ref{req:lem-c1-indstep2-1}; then, $\mathbf{B} \subseteq
  \ttype[\Xi]{\Uni_{\tgkb[\A]}}(h(\sigma))$, since obviously
  $\mathbf{B} \subseteq \ttype[\Xi]{\Uni_{\srkb[\A]}}(\sigma)$ and $h$
  is a homomorphism on $\sigma$. By Lemma \ref{lem:types-chase} (with
  $\T=\T_2$ and $\T'=\T_{12}$) we obtain $\delta \in
  \dom[\Uni_{\tgkb[\A]}]$ such that $\ttype{\Uni_{\tup{\T_2,\{B(o)
        \mid B \in \mathbf{B}\}}}}(y) \subseteq
  \ttype{\Uni_{\tgkb[\A]}}(\delta)$ and $\rtype{\Uni_{\tup{\T_2,\{B(o)
        \mid B \in \mathbf{B}\}}}} (o,y) \subseteq
  \rtype{\Uni_{\tgkb[\A]}}(h(\sigma),\delta)$.  Note that using
  \eqref{eq:Unisrctype-KBtype1} and \eqref{eq:KBtarg-typesubs1} we
  obtain $\ttype{\Uni_{\srkb[\A]}}(\sigma w_{[R]}) \subseteq
  \ttype{\Uni_{\tgkb[\A]}}(\delta)$; also using
  \eqref{eq:Unisrctype-KBtype2} and \eqref{eq:KBtarg-typesubs2} we
  obtain $\rtype[\Xi]{\Uni_{\srkb[\A]}}(\sigma, \sigma w_{[R]})
  \subseteq \rtype[\Xi]{\Uni_{\tgkb[\A]}}(h(\sigma),\delta)$. We
  assign $h(\sigma w_{[R]})=\delta$ which concludes the proof.

  Consider \ref{req:lem-c1-indstep2-2}; at this point we need
  \begin{align}
    \label{eq:type-subs-KBsrc1} \mathbf{B} &\subseteq
    \ttype{\Uni_{\srkb[\A]}}(\sigma') \text{ and }\\
    \label{eq:type-subs-KRsrc2} \mathbf{R} &\subseteq
    \rtype{\Uni_{\srkb[\A]}}(\sigma,\sigma'),
  \end{align}
  for $\mathbf{R} = \{ R'' \mid \T_{12} \vdash R'^- \ISA
  R''\}$. Indeed, \eqref{eq:type-subs-KBsrc1} follows since $\exists
  R' \in \ttype{\Uni_{\srkb[\A]}}(\sigma')$, by the definition of
  $\mathbf{B}$, and Lemma \ref{lem:Unitype-KBtype-ABox}~\ref{tmpt:lem-Unitype-KBtype-ABox-1} and~\ref{tmpt:lem-Unitype-KBtype-ABox-3}. For
  \eqref{eq:type-subs-KRsrc2} let $R'' \in \mathbf{R}$ it follows
  $[R'^-] \leq_{\T_1 \cup \T_{12}} [R'']$ and so $[R'] \leq_{\T_1 \cup
    \T_{12}} [R''^-]$. Then by the definition of $\Uni_{\srkb[\A]}$
  obtain $R''^- \in \rtype{\Uni_{\srkb[\A]}}(\sigma', \sigma)$, so
  obviously $R'' \in \rtype{\Uni_{\srkb[\A]}}(\sigma, \sigma')$.

  Observe that, since $h$ is a $\Xi$-homomorphism on $\sigma'$ and
  \eqref{eq:type-subs-KBsrc1}, it follows
  \begin{align}
    \mathbf{B} \subseteq
    \ttype[\Xi]{\Uni_{\tgkb[\A]}}(h(\sigma'))\label{eq:lem-c1-B-subs-tuni}
  \end{align}
  and distinguish two subcases:
  \begin{enumerate}[reqs]
  \item \label{req:lem-c1-indstep2-2-1}
    $\rtype[\Xi]{\Uni_{\srkb[\A]}}(\sigma,\sigma w_{[R]}) =
    \emptyset$;
  \item \label{req:lem-c1-indstep2-2-2}
    $\rtype[\Xi]{\Uni_{\srkb[\A]}}(\sigma,\sigma w_{[R]}) \neq
    \emptyset$.
  \end{enumerate}
  In case \ref{req:lem-c1-indstep2-2-1} consider
  \eqref{eq:KBtarg-typesubs3a}, \eqref{eq:lem-c1-B-subs-tuni} and
  Lemma \ref{lem:types-chase} to obtain
  $\delta \in \dom[\Uni_{\tgkb[\A]}]$ such that
  $\ttype{\Uni_{\tup{\T_2,\{B(o) \mid B \in \mathbf{B}\}}}}(y) \subseteq
  \ttype{\Uni_{\tgkb[\A]}}(\delta)$.
  Then using \eqref{eq:Unisrctype-KBtype1} and
  \eqref{eq:KBtarg-typesubs1a} one obtains
  $\ttype[\Xi]{\Uni_{\srkb[\A]}}(\sigma w_{[R]}) \subseteq
  \ttype[\Xi]{\Uni_{\tgkb[\A]}}(\delta)$. Taking $\delta=h(\sigma
  w_{[R]})$ completes the proof of the first subcase.

  In the alternative case \ref{req:lem-c1-indstep2-2-2}, it follows by
  \eqref{eq:Unisrctype-KBtype2} that $\rtype[\Xi]{\Uni_{\srkb[(\exists
    R'^-)]}}(o,w_{[R]}) \neq \emptyset$ therefore $y=o$ (c.f.
  \eqref{eq:KBtarg-typesubs3a}). We assign $h(\delta  w_{[R]})=h(\sigma')$ and we prove
  $\ttype[\Xi]{\Uni_{\srkb[\A]}}(\sigma w_{[R]}) \subseteq
  \ttype{\Uni_{\tgkb[\A]}}(h(\sigma'))$, and
  $\rtype[\Xi]{\Uni_{\srkb[\A]}}(\sigma, \sigma w_{[R]})
  \subseteq \rtype{\Uni_{\tgkb[\A]}}(h(\sigma), h(\sigma'))$.

  Indeed, let $B \in \ttype[\Xi]{\Uni_{\srkb[\A]}}(\sigma w_{[R]})$,
  by \eqref{eq:Unisrctype-KBtype1} $B \in \stype[\Xi]{\srkb[(\exists
    R'^-)]}(o w_{[R]})$, then by \eqref{eq:KBtarg-typesubs1a} there
  exists $B' \in \mathbf{B}$ such that $\T_2 \vdash B' \ISA B$. Using
  \eqref{eq:lem-c1-B-subs-tuni} and Lemma
  \ref{lem:Unitype-KBtype-ABox}~\ref{tmpt:lem-Unitype-KBtype-ABox-3}
  obtain $B \in \ttype{\Uni_{\tgkb[\A]}}(h(\sigma'))$.

  Let now $Q \in \rtype[\Xi]{\Uni_{\srkb[\A]}}(\sigma, \sigma
  w_{[R]})$, by \eqref{eq:Unisrctype-KBtype2} it follows $Q \in
  \rtype[\Xi]{\Uni_{\srkb[(\exists R'^-)]}}(o,o w_{[R]})$, then by
  \eqref{eq:KBtarg-typesubs2a} there exists $R'' \in \mathbf{R}$ such
  that $\T_2 \vdash R'' \ISA Q$. Since $h$ is a homomorphism on
  $\sigma$, $\sigma'$ and \eqref{eq:type-subs-KRsrc2} obtain $R'' \in
  \ttype[\Xi]{\Uni_{\tgkb[\A]}}(h(\sigma), h(\sigma'))$. By the
  definition of $\Uni_{\tgkb[\A]}$ we conclude also $Q \in
  \rtype[\Xi]{\Uni_{\tgkb[\A]}}(h(\sigma), h(\sigma'))$. This
  concludes the proof of the second subcase and the whole case
  \ref{req:lem-c1-indstep2-2}. We have shown how to define $h$ for
  $\sigma w_{[R]} \in \gpath(\srkb[\A])$ so that $h$ is
  $\Xi$-homomorphism.
\end{proof}

\subsection{Proof of Proposition 6.1}

This proof can be obtained as an easy consequence of the following
\begin{lemma}
  Let $\M=(\Sigma, \Xi, \T_{12})$ be a mapping, and $\T_1$ and $\T_2$,
  respectively, $\Sigma$- and $\Xi$-TBoxes, $q(\vec{x})$ a
  $\Xi$-query, and $\A$ a $\Sigma$ ABox. Then
  \begin{align*}
    \bigcap_{\begin{subarray}{c}%
        \A' - \text{ ABox, s.t. it is}\\
        \text{UCQ-solution for } \A\\
        \text{ under }\T_{12}
      \end{subarray}} \cert(q, \tup{\T_2, \A'}) \subseteq
    \bigcap_{\begin{subarray}{c}%
        \A' - \text{ extended ABox, s.t.}\\
        \text{it is UCQ-solution for } \A\\
        \text{ under }\T_{12}
      \end{subarray}} \cert(q, \tup{\T_2, \A'}).
  \end{align*}
\end{lemma}
\begin{proof} Consider a tuple of constants $\vec{a}$ such that
  $\tup{\T_2, \A'} \models q[\vec{a}]$ for all $\Xi$-ABoxes $\A'$,
  such that $\A'$ is a $\UCQ$-solution for $\A$ under
  $\T_{12}$. Assume an extended ABox $\A'$, such that it is a
  $\UCQ$-solution for $\A$; we are going to show $\tup{\T_2, \A'}
  \models q[\vec{a}]$. If $\tup{\T_2, \A'}$ is inconsistent, the proof
  is done; otherwise, take an interpretation $\I \models \tup{\T_2,
    \A'}$. It follows there exists a substitution over $\I$, such that
  $h(u) \in B^\I$ for every $B(u) \in \A$, and $(h(u), h(v)) \in R^\I$
  for all $R(u,v) \in \A$. We associate with every null $n$ in $\A'$ a
  \emph{fresh} (w.r.t. constants in $\A'$, $\vec{a}$, and
  $q(\vec{x})$) constant $a_n \in N_a$; then take $\A^*$ the result of
  the substitution of each $n$ by $a_n$ in $\A'$. Consider an
  interpretation $\I^*$, such that it is equal to $\I$, except for
  $a_n$, such that $n$ is a null in $\A'$, we set $a_n^{\I^*} =
  h(n)$. It should be clear that $\I^* \models \tup{\T_2, \A^*}$, then
  we obtain $\I^* \models q[\vec{a}]$. It remains to show $\I \models
  q[\vec{a}]$; for that assume $\vec{x}= (x_1, \dots x_n)$, $\vec{a} =
  (a_1, \dots, a_n)$, and
$$q(\vec{x})=
\exists y_1, \dots, y_m \varphi(\vec{x}, y_1, \dots, y_m, b_1, \dots,
b_k),$$
where $b_i$ are constants and $\varphi$ a quantifier-free
formula. It follows, there exist $d_1, \dots, d_n, e_1, \dots e_m,
f_1, \dots, f_k \in \Delta^{\I^*}$, such that $d_i = a_i^{\I^*}$, $f_i
= b_i^{\I^*}$, and
$$\I^* \models \varphi(d_1, \dots d_n, e_1, \dots,
e_m, f_1, \dots, f_k).$$
It remains to observe that all of $d_i, e_i, f_i$ belong to the
interpretation of the same concepts/roles in $\I$ as in $\I'$, and
$a_i^\I=d_i$, $b_i^\I=f_i$. Therefore, $\I \models \varphi(d_1, \dots
d_n, e_1, \dots, e_m, f_1, \dots, f_k)$, and, finally, $\I \models
q[\vec{a}]$.
\end{proof}

\subsection{Proof of Proposition 6.2}

The result is proved in Theorem~\ref{th:memb-repres-nlogspace-compl}, which is
based on a series of lemmas.

\begin{lemma}\label{lem:srkb-query-equiv-tgkb}
  Let $\M=(\Sigma, \Xi, \T_{12})$ be a mapping, and $\T_1$ and $\T_2$,
  respectively, $\Sigma$- and $\Xi$-TBoxes. Then $\T_2$ is a
  \emph{UCQ-representation of $\T_1$ under $\T_{12}$} if and only if
  $\tup{\T_1 \cup \T_{12}, \A}$ is $\Xi$-query equivalent to
  $\tup{\T_2 \cup \T_{12}, \A}$ for every ABox $\A$ over $\Sigma$ such
  that $\tup{\T_1, \A}$ is consistent.
\end{lemma}
\begin{proof}
  We first prove the following:
  \begin{proposition}\label{prop:srkb-ans-iff-trgtkb-ans-all-sol}
    Let $\M=(\Sigma, \Xi, \T_{12})$ be a mapping, and $\T_1$ and
    $\T_2$, respectively, $\Sigma$- and $\Xi$-TBoxes, $\A$ a
    $\Sigma$-ABox, such that $\tup{\T_1, \A}$ is consistent,
    $q(\vec{x})$ a $\Xi$ query and $\vec{a}$ a tuple of
    constants. Then $\tup{\T_2 \cup \T_{12}, \A} \models q[\vec{a}]$
    iff $\tup{\T_2, \A'} \models q[\vec{a}]$ for all $\Xi$-ABoxes
    $\A'$ such that $\A'$ is a UCQ-solution for $\A$ under $\M$.
  \end{proposition}
  \begin{proof}
    ($\Rightarrow$) Let $\A$ and $\A'$ as above; we show $\tup{\T_2,
      \A'}$ $\Xi$-query entails $\tup{\T_2 \cup \T_{12}, \A}$. Notice
    that since $\A'$ is a UCQ-solution, it follows $\tup{\T_{12}, \A}$
    $\Xi$-query entails $\A'$; and since $\A$ is consistent,
    $\tup{\T_{12}, \A}$ is consistent as well. Using
    Claim~\ref{th:query-entail-homo}, we obtain that
    $\Uni_{\tup{\T_{12}, \A}}$ is $\Xi$ homomorphically embeddable
    into $\Uni_{\A'}$. By
    Lemma~\ref{lem:from-model-types-to-kb-types-ABox} it follows
    $\Uni_{\tup{\T_2 \cup \T_{12}, \A}}$ is $\Xi$ homomorphically
    embeddable into $\Uni_{\tup{\T_2, \A'}}$. Now, if $\tup{\T_2 \cup
      \T_{12}, \A}$ is inconsistent, it can be shown in the way
    similar to the proof of Lemma~\ref{lem:kr-iff-b-cons} that
    $\tup{\T_2, \A'}$ is inconsistent, then the proof is
    done. Otherwise, we use Claim~\ref{th:query-entail-homo} to
    conclude $\tup{\T_2, \A'}$ $\Xi$-query entails $\tup{\T_2 \cup
      \T_{12}, \A}$.

    ($\Leftarrow$) Let $\A$, $q(\vec{x})$, and $\vec{a}$ as above;
    assume $\tup{\T_2, \A'} \models q[\vec{a}]$ for all solutions
    $\A'$ for $\A$ under $\T_{12}$. 
    We are going to show $\tup{\T_2 \cup \T_{12},\A} \models
    q[\vec{a}]$. If $\tup{\T_2 \cup \T_{12},\A}$ is inconsistent, the
    proof is done; assume the opposite, then we will show
    $\Uni_{\tup{\T_2 \cup \T_{12},\A}} \models q[\vec{a}]$, using
    Theorem \ref{th:query-kb-vs-uni} the proof will be done. Consider
    $\Uni_{\tup{\T_{12},\A}}$ and define the set
$$\Theta_{\T_{12},
  \A} = \Ind(\A) \cup \{a w_{[R]} \in \Delta^{\Uni_{\tup{\T_{12},\A}}}
\mid a \in \Ind(\A) \text{ and } R \text{ over } \Sigma\}.
$$
For each $\sigma \in \Theta_{\T_{12}, \A}$ define $t_\sigma = a$ if
$\sigma=a \in \Ind(\A)$, and $t_\sigma = a_\sigma$ for a \emph{fresh
  w.r.t.}  $\A$, $\vec{a}$, and $q(\vec{x})$ constant $a_\sigma$,
otherwise. Now, define
\begin{align*}
  \A'=&\{ B(t_\sigma) \mid B \text{ basic conc. over } \Xi, \sigma \in
  B^{\Uni_{\tup{\T_{12},\A}}} \cap \Theta_{\T_{12},
    \A}\} \cup \\
  & \{ P(t_\sigma, t_{\sigma'}) \mid P \text{ role name over } \Xi,
  (\sigma, \sigma') \in P^{\Uni_{\tup{\T_{12},\A}}} \cap
  (\Theta_{\T_{12}, \A} \times \Theta_{\T_{12}, \A})\}.
\end{align*}
It is straightforward to build a $\Xi$-homomorphism from
$\tup{\T_{12}, \A}$ to $\A'$ and use Claim~\ref{th:query-entail-homo}
to show $\A'$ is a UCQ-solution for $\A$ under $\T_{12}$. Consider now
$\Uni_{\tup{\T_2, \A'}}$ and a mapping $g : \dom[\Uni_{\tup{\T_2,
    \A'}}] \mapsto \dom[\Uni_{\tup{\T_2 \cup \T_{12},\A}}]$ defined in
the following way:
\begin{align*}
  g(a w_{[R_1]} \dots w_{[R_n]}) = %
  \begin{cases}
    \sigma w_{[R_1]} \dots w_{[R_n]},& \text{ if } a=t_\sigma \text{
      and } \sigma \in \Theta_{\T_{12}, \A},\\
    a,& \text{ otherwise},
  \end{cases}
\end{align*}
where $n \geq 0$.  Notice that $g$ is not a homomorphism, however,
using the definitions of $\Uni_{\tup{\T_2, \A'}}$ and $\Uni_{\tup{\T_2
    \cup \T_{12},\A}}$ one can straightforwardly verify
\begin{align}
  \label{eq:prop:srkb-ans-iff-trgtkb-ans-all-sol-1} \ttype{\tup{\T_2,
      \A'}}(\delta) &\subseteq \ttype{\tup{\T_2 \cup
      \T_{12}, \A}}(g(\delta)),\\
  \label{eq:prop:srkb-ans-iff-trgtkb-ans-all-sol-2} \rtype{\tup{\T_2,
      \A'}}(\delta, \delta') &\subseteq \rtype{\tup{\T_2 \cup \T_{12},
      \A}}(g(\delta), g(\delta')),
\end{align}
for all $\delta, \delta' \in \dom[\Uni_{\tup{\T_2, \A'}}]$. This is
sufficint to prove in the way analogious to the proof of
Lemma~\ref{lem:kr-iff-b-cons}, that $\Uni_{\tup{\T_2, \A'}}$ is
consistent. Using Claim~\ref{th:query-kb-vs-uni} one can obtain
$\Uni_{\tup{\T_2, \A'}} \models q[\vec{a}]$. Finally, observe
\begin{align}
  \label{eq:prop:srkb-ans-iff-trgtkb-ans-all-sol-3} g(a) = a,
\end{align}
for all $a$ in $\Ind(\A)$, $\vec{a}$, or $q(\vec{x})$; then
using~\eqref{eq:prop:srkb-ans-iff-trgtkb-ans-all-sol-1},~\eqref{eq:prop:srkb-ans-iff-trgtkb-ans-all-sol-2},~\eqref{eq:prop:srkb-ans-iff-trgtkb-ans-all-sol-3}
in the same way as the proof of Claim~\ref{th:query-entail-homo} one
can show $\Uni_{\tup{\T_2 \cup \T_{12},\A}} \models q[\vec{a}]$, which
concludes the proof.

\end{proof}

Now, given a $\Sigma$ ABox $\A$ such that $\tup{\T_1, \A}$ is
consistent, we show that $\tup{\T_1 \cup \T_{12}, \A}$ is $\Xi$-query
equivalent to $\tup{\T_2 \cup \T_{12}, \A}$ \emph{if and only if } for
every $\Xi$ query $q(\vec{x})$ it holds
\begin{align}
  \cert(q, \tup{\T_1 \cup \T_{12}, \A})= \bigcap_{\A' - \text{
      solution for } \A \text{ under }\T_{12}} \cert(q, \tup{\T_2,
    \A'}). \label{eq:def-representab}
\end{align}
($\Rightarrow$) Let $q(\vec{x})$ be a $\Xi$ query, it follows
$\cert(q, \tup{\T_1 \cup \T_{12}, \A})=\cert(q, \tup{\T_2 \cup
  \T_{12}, \A})$, and we easily obtain \eqref{eq:def-representab}
using Proposition
\ref{prop:srkb-ans-iff-trgtkb-ans-all-sol}. ($\Leftarrow$) Let
$q(\vec{x})$ be a $\Xi$ query, we need to show $\cert(q, \tup{\T_1
  \cup \T_{12}, \A})=\cert(q, \tup{\T_2 \cup \T_{12}, \A})$, which is
easily concluded using Proposition
\ref{prop:srkb-ans-iff-trgtkb-ans-all-sol} and
\eqref{eq:def-representab}.
\end{proof}
\begin{lemma}\label{lem:t2-is-repr-iff-conds}
  The $\Xi$-TBox $\T_2$ is a $\UCQ$-representation of $\Sigma$-TBox
  $\T_1$ under the mapping $\M=(\Sigma, \Xi, \T_{12})$ \emph{if and
    only if} following conditions hold:
  \begin{enumerate}[thmparts, start=1]
  \item \label{req:t2-is-repr-iff-conds-conc-cons} for each pair of
    $\T_1$-consistent concepts $B, B'$ over $\Sigma$, $B,B'$ is $\T_1
    \cup \T_{12}$-consistent iff $B,B'$ is $\T_2 \cup
    \T_{12}$-consistent;
  \item \label{req:t2-is-repr-iff-conds-roles-cons} for each pair of
    $\T_1$-consistent roles $R, R'$ over $\Sigma$, $R,R'$ is $\T_1
    \cup \T_{12}$-consistent iff $R,R'$ is $\T_2 \cup
    \T_{12}$-consistent;
  \item \label{req:t2-is-repr-iff-conds-conc-equal} for each $\T_1
    \cup \T_{12}$-consistent concept $B$ over $\Sigma$ and each $B'$
    over $\Sigma_2$, $\T_1 \cup \T_{12} \vdash B \ISA B'$ iff $\T_2
    \cup \T_{12} \vdash B \ISA B'$;
  \item \label{req:t2-is-repr-iff-conds-roles-equal} for each $\T_1
    \cup \T_{12}$-consistent role $R$ over $\Sigma$ and each $R'$
    over $\Sigma_2$, $\T_1 \cup \T_{12} \vdash R \ISA R'$ iff $\T_2
    \cup \T_{12} \vdash R \ISA R'$;%
  \item \label{req:t2-is-repr-iff-conds-roles-gen} for each $B \in
    \consc$ over $\Sigma$ and each role $R$ such
    that $o \leadsto_{\srkb} w_{[R]}$ there exists $y \in
    \dom[\Uni_{\tgkb}]$ such that
    \begin{enumerate}[label=\textbf{(\alph*)}]
    \item $\ttype[\Xi]{\Uni_{\srkb}}(o w_{[R]}) \subseteq
      \ttype{\Uni_{\tgkb}} (y)$,
    \item $\rtype[\Xi]{\Uni_{\srkb}}(o, o w_{[R]}) \subseteq
      \rtype{\Uni_{\tgkb}} (o,y)$;
    \end{enumerate}
  \item \label{req:t2-is-repr-iff-conds-roles-gen-inv} for each $B \in
    \consc$ over $\Sigma$ and each role $R$ such that $o
    \leadsto_{\tgkb} w_{[R]}$ there exists $y$ such that $y \in
    \dom[\Uni_{\srkb}]$ and
    \begin{enumerate}[label=\textbf{(\alph*)}]
    \item $\ttype[\Xi]{\Uni_{\tgkb}}(o w_{[R]}) \subseteq \ttype{\Uni_{\srkb}}
      (y)$,
    \item $\rtype[\Xi]{\Uni_{\tgkb}}(o, o w_{[R]}) \subseteq
      \rtype{\Uni_{\srkb}} (o,y)$
    \end{enumerate}
  \end{enumerate}

\end{lemma}
\begin{proof}
  ($\Leftarrow$) Let the conditions above hold for $\T_1$, $\T_2$ and
  $\T_{12}$. Let $\A$ be an ABox over $\Sigma$ such that $\tup{\T_1,
    \A}$ is consistent, we show $\srkb[\A]$ is $\Xi$-query equivalent
  to $\tgkb[\A]$.

  Observe that $\srkb[\A]$ is consistent iff $\tgkb[\A]$ is
  consistent.  Indeed, if $\srkb[\A]$ is inconsistent then by Lemma
  \ref{lem:kr-iff-b-cons} one of the following holds:
  \begin{enumerate}[reqs]
  \item \label{req:lem-t2-is-repr-iff-conds-incons-conc} $B_1,B_2$ is
    $\T_1 \cup \T_{12}$-inconsistent for some basic concepts $B_1,B_2$
    and $a \in \Ind(\A)$ such that $\A \models B_1(a)$, $\A \models
    B_2(a)$;
  \item \label{req:lem-t2-is-repr-iff-conds-incons-role} $R_1,R_2$ is
    $\T_1 \cup \T_{12}$-inconsistent for some roles $R_1,R_2$ and $a,b
    \in \Ind(\A)$ such that $\A \models R_1(a,b)$, $\A \models
    R_2(a,b)$
  \end{enumerate}
  Consider \ref{req:lem-t2-is-repr-iff-conds-incons-conc} and observe
  that by Lemma \ref{lem:kr-iff-b-cons} $B_1,B_2$ are $\T_1$
  consistent. Then by \ref{req:t2-is-repr-iff-conds-conc-cons}
  $B_1,B_2$ are $\T_2 \cup \T_{12}$-inconsistent and again by Lemma
  \ref{lem:kr-iff-b-cons} $\tgkb[\A]$ is inconsistent. The proof for
  the case of \ref{req:lem-t2-is-repr-iff-conds-incons-role} is
  similar using \ref{req:t2-is-repr-iff-conds-roles-cons}. The proof
  can be inverted to show $\tgkb[\A]$ is inconsistent implies
  $\srkb[\A]$ is inconsistent.

  First, assume $\srkb[\A]$ is inconsistent, it follows $\srkb[\A]
  \models q[\vec{a}]$ for all $\vec{a} \subseteq \Ind(\A)$ and
  $\Xi$-queries $q$. By the paragraph above, $\tgkb[\A]$ is
  inconsistent, so $\tgkb[\A] \models q[\vec{a}]$ for all $\vec{a}
  \subseteq \Ind(\A)$ and $\Xi$-queries $q$, and so $\srkb[\A]$ is
  $\Xi$-query equivalent to $\tgkb[\A]$.
  
  Now assume $\srkb[\A]$ is consistent, by Lemma
  \ref{lem:kr-to-all-b-cons} each $B$ is $\T_1 \cup
  \T_{12}$-consistent for all $\delta, \sigma \in
  \Delta^{\Uni_{\srkb[\A]}}$, each $B$ such that $B \in
  \ttype{\Uni_{\srkb[\A]}}(\delta)$, and each $R$ such that $R \in
  \rtype{\Uni_{\srkb[\A]}}(\delta, \sigma)$. It follows from
  \ref{req:t2-is-repr-iff-conds-conc-equal},
  \ref{req:t2-is-repr-iff-conds-roles-equal} and
  \ref{req:t2-is-repr-iff-conds-roles-gen}, that all the conditions of
  Lemma \ref{lem:c1} are satisfied, therefore we conclude
  $\Uni_{\srkb[\A]}$ is finitely $\Xi$-homomorphically embeddable into
  $\Uni_{\tgkb[\A]}$. Since $\tgkb[\A]$ is consistent, then we can
  apply Theorem \ref{th:query-entail-homo} to obtain $\tgkb[\A]$
  $\Xi$-query entails $\srkb[\A]$. On the other hand,
  \ref{req:t2-is-repr-iff-conds-conc-equal},
  \ref{req:t2-is-repr-iff-conds-roles-equal} and
  \ref{req:t2-is-repr-iff-conds-roles-gen-inv} imply that all the
  conditions of Lemma \ref{lem:c2} are satisfied, therefore we
  conclude $\Uni_{\tgkb[\A]}$ is finitely $\Xi$-homomorphically
  embeddable into $\Uni_{\srkb[\A]}$ and $\srkb[\A]$ $\Xi$-query
  entails $\tgkb[\A]$ by Theorem \ref{th:query-entail-homo}. We again
  obtain $\srkb[\A]$ is $\Xi$-query equivalent to $\tgkb[\A]$.

  ($\Rightarrow$) Assume, by contraction, one of the conditions
  \ref{req:t2-is-repr-iff-conds-conc-cons} --
  \ref{req:t2-is-repr-iff-conds-roles-gen-inv} is not satisfied. We
  produce a $\T_1$-consistent ABox $\A$ over $\Sigma$ and a
  \emph{instance} $\Xi$-query $q[]$ such that it is not the case that
  $\srkb[\A] \models q$ iff $\tgkb[\A] \models q$.

  Assume, first, the condition
  \ref{req:t2-is-repr-iff-conds-conc-cons} is violated, then we take
  $\A=\{B_1(o), B_2(o)\}$ violating it and $q = B_1 (a)$ for some
  constant $a \neq o$. If $B_1,B_2$ are $\T_1 \cup
  \T_{12}$-consistent, but $\T_2 \cup \T_{12}$ inconsistent, it
  follows $\srkb[\A] \not \models q$ and $\tgkb[\A] \models q$, and
  the opposite holds if $B_1,B_2$ are $\T_2 \cup \T_{12}$-consistent,
  but $\T_1 \cup \T_{12}$-inconsistent. If
  \ref{req:t2-is-repr-iff-conds-roles-cons} is violated, the proof is
  analogous.

  Let now the condition \ref{req:t2-is-repr-iff-conds-conc-equal} be
  violated for $B \in \consc$ over $\Sigma$. Assume, first, there is
  $B' \in \ttype[\Xi]{\Uni_{\srkb}}(o) \setminus
  \ttype[\Xi]{\Uni_{\tgkb}}(o)$, then we take $q= B'(o)$. By
  definition of $\Uni_{\srkb}$, $\Uni_{\tgkb}$ and Lemma
  \ref{lem:Unitype-KBtype-ABox} it follows $\Uni_{\srkb} \models q$
  and $\Uni_{\tgkb} \not \models q$; then by
  Claim~\ref{th:query-kb-vs-uni} it follows $\srkb \models q$ and
  $\tgkb \not \models q$. The opposite follows if there exists $B' \in
  \ttype[\Xi]{\Uni_{\tgkb}}(o) \setminus
  \ttype[\Xi]{\Uni_{\srkb}}(o)$, which completes the proof for this
  case. If \ref{req:t2-is-repr-iff-conds-roles-equal} is violated, the
  proof is analogios.

  To prove the case when \ref{req:t2-is-repr-iff-conds-roles-gen} is
  violated, we need an additional lemma below. Before we present it,
  notice that, w.l.o.g., one can consider $\UCQ$'s with atoms over
  \emph{basic} concepts $B(t)$; one can convert such a $\UCQ$ into the
  one over the standard syntax by using fresh existentially
  quantified variables.
  \begin{lemma}\label{lem:query-vs-gen}
    Let $\T$ TBox, $B$ a concept, $\vec{B}$ and $\vec{R}$ the sets of
    concepts and roles, respectively, and the instance query
    \begin{align*}
      q_{\vec{B}, \vec{R}} =\exists x \bigl(\bigwedge_{B' \in \vec{B}}
      B'(x) \land \bigwedge_{R' \in \vec{R}} R'(o,x)\bigr).
    \end{align*}
    Then $\Uni_{\tup{\T, \{B(o)\}}} \models q_{\vec{B}, \vec{R}}$ iff
    there exists $y \in \dom[\Uni_{\tup{\T,\{B(o)\}}}]$, such that
    \begin{enumerate}[thmparts, start=1]
    \item \label{tmpt:lem-query-vs-gen-conc} $\vec{B} \subseteq
      \ttype{\Uni_{\tup{\T,\{B(o)\}}}}(y)$,
    \item \label{tmpt:lem-query-vs-gen-role} $\vec{R} \subseteq
      \rtype{\Uni_{\tup{\T,\{B(o)\}}}}(o, y)$.
    \end{enumerate}
  \end{lemma}
  \begin{proof}
    Straightforward using Lemma
    \ref{lem:types-chase} and the definition of $\Uni_{\tup{\T, \{B(o)\}}}$.
  \end{proof}

  Now, assume \ref{req:t2-is-repr-iff-conds-roles-gen} is violated, so
  there exists $B \in \consc$ over $\Sigma$ and a role $R$ such that
  $o \leadsto_{\srkb} w_{[R]}$ and for all $y \in \dom[\Uni_{\tgkb}]$
  either $\ttype[\Xi]{\Uni_{\srkb}}(o w_{[R]}) \not \subseteq
  \ttype{\Uni_{\tgkb}} (y)$ or $\rtype[\Xi]{\Uni_{\srkb}}(o,w_{[R]})
  \not \subseteq \rtype{\Uni_{\tgkb}} (o,y)$. Then, by Lemma
  \ref{lem:query-vs-gen} with $\vec{B}=\ttype[\Xi]{\Uni_{\srkb}}(o
  w_{[R]})$, $\vec{R}=\rtype[\Xi]{\Uni_{\srkb}}(o,o w_{[R]})$ and
  $\T=\T_1 \cup \T_{12}$ it follows $\Uni_{\srkb} \models q_{\vec{B},
    \vec{R}}$. On the other hand, by Lemma \ref{lem:query-vs-gen} with
  $\T=\T_2 \cup \T_{12}$ it follows $\Uni_{\tgkb} \not \models
  q_{\vec{B}, \vec{R}}$. Using Claim~\ref{th:query-kb-vs-uni} we then
  obtain $\srkb \models q_{\vec{B}, \vec{R}}$ and $\tgkb \not \models
  q_{\vec{B}, \vec{R}}$.

  The case when \ref{req:t2-is-repr-iff-conds-roles-gen-inv} is
  violated is analogous to the case above. The proof is complete.

\end{proof}
\begin{theorem}\label{th:memb-repres-nlogspace-compl}
  The membership problem for $\UCQ$-representability is
  \NLOGSPACE-complete.
\end{theorem}
\begin{proof}
  The lower bound can be obtained by the reduction from the directed
  graph reachability problem, which is known to be NLogSpace-hard:
  given a graph $\mathcal{G} = (\mathcal{V}, \mathcal{E})$ and a pair
  of vertices $v_k, v_m \in \mathcal{V}$, decide if there is a
  directed path from $v_k$ to $v_m$. To encode the problem, we need a
  set of $\Sigma$ concept names $\{V_i \mid v_i \in \mathcal{V}\}$ and
  a set of $\Xi$ concept names $\{V_i' \mid v_i \in
  \mathcal{V}\}$. Consider $\T_1= \{ V_k \sqsubseteq V_m\} \cup \{ V_i
  \sqsubseteq V_j \mid (v_i, v_j) \in \mathcal{E}\}$, $\T_{12} = \{
  V_i \sqsubseteq V_i' \mid v_i \in \mathcal{V}\}$, and $\T_2 = \{V_i'
  \sqsubseteq V_j' \mid (v_i, v_j) \in \mathcal{E}\}$. One can easily
  verify that the condition~\ref{req:t2-is-repr-iff-conds-conc-equal}
  of Lemma~\ref{lem:t2-is-repr-iff-conds} is satisfied iff there is a
  directed path from $v_k$ to $v_m$ in $\mathcal{G}$, whereas the
  other conditions of Lemma~\ref{lem:t2-is-repr-iff-conds} are
  satisfied trivially. Therefore,
  \begin{proposition}
    There is a directed path from $v_k$ to $v_m$ in $\mathcal{G}$ iff
    $\T_2$ is a representation for $\T_1$ under $\M=(\Sigma, \Xi,
    \T_{12})$.
  \end{proposition}
  This concludes the proof of the lower bound. For the upper bound, we
  show that the
  conditions~\ref{req:t2-is-repr-iff-conds-conc-cons}--\ref{req:t2-is-repr-iff-conds-roles-gen-inv}
  of Lemma~\ref{lem:t2-is-repr-iff-conds} can be verified in
  \NLOGSPACE. It is well known (see, e.g., \cite{ACKZ09}), that given
  a pair of $\dlliter$ concepts $B, B'$, and a TBox $\T$, it can be
  verified in \NLOGSPACE, if $B, B'$ is $\T$ consistent (using an
  algorithm, based on directed graph reachability solving procedure);
  the same holds for a pair of $\dlliter$ roles $R$, $R'$. The same
  algorithm can be straightforwardly adopted to check, if $\T \vdash B
  \sqsubseteq B'$ or $\T \vdash R \sqsubseteq R'$. Therefore, clearly,
  the
  conditions~\ref{req:t2-is-repr-iff-conds-conc-cons}--\ref{req:t2-is-repr-iff-conds-roles-equal}
  can be verified in \NLOGSPACE.

  The conditions \ref{req:t2-is-repr-iff-conds-roles-gen} and
  \ref{req:t2-is-repr-iff-conds-roles-gen-inv} are slightly more
  involved; first of all, observe that, given a concept $B$ and a role
  $R$, it can be checked in \NLOGSPACE, whether $o
  \leadsto_{\tup{\T,\{B(o)\}}} w_{[R]}$, using an algorithm based on
  the directed graph reachability solving procedure. At the same time,
  given $z \in \{o\} \cup \{w_{[R]} \mid R - \text{ role}\}$, we can
  verify, if there exists $y \in \dom[\Uni_{\tup{\T,\{B(o)\}}}]$ with
  $z = \tail(y)$: we ``follow'' the sequence of roles $R_1, \dots, R_n
  = R$ (with $n\geq 0$) in the way that when we ``guess'' $R_{i+1}$,
  we check $w_{[R_i]} \leadsto_{\tup{\T,\{B(o)\}}} w_{[R_{i+1}]}$ (by
  the algorithm, similar to the one for checking $o
  \leadsto_{\tup{\T,\{B(o)\}}} w_{[R]}$), and ``forget'' $R_i$.

  Furthermore, in a similar way, as testing $\T \vdash B \sqsubseteq
  B'$, one can, check for a concept $B'$, if $B' \in
  \ttype[\Xi]{\Uni_{\tup{\T,\{B(o)\}}}}(o w_{[R]})$ in \NLOGSPACE; the
  same holds for checking if a role $R' \in
  \rtype[\Xi]{\Uni_{\srkb}}(o, o w_{[R]})$, and, then, for checking
  $B' \in \ttype[\Xi]{\Uni_{\tup{\T,\{B(o)\}}}}(y)$, for $y$ as
  above. By combining the algorithms outlined above, one can produce a
  procedure that checks the conditions
  \ref{req:t2-is-repr-iff-conds-roles-gen} and
  \ref{req:t2-is-repr-iff-conds-roles-gen-inv} in \NLOGSPACE.
\end{proof}
%
 


\section{Non-emptyness Problem for $\UCQ$-representability}

The definitions that follow are needed for the non-emptyness problem
of $\UCQ$-representability. Let the mapping $\M=(\Sigma, \Xi,
\T_{12})$, $\T_1$ and $\T_2$ TBoxes over, respectively, $\Sigma$ and
$\Xi$. For a pair of concepts $B'$, $C'$ be over $\Xi$, we say that
$\T_1 \cup \T_{12}$ is \emph{closed under inclusion between $B'$ and
  $C'$} if the following is satisfied for each $\T_1$-consistent
concept $B$ over $\Sigma$:
\begin{enumerate}[resume*=reqs]
\item \label{req:def-incl-clo-conc-1} $\T_1 \cup \T_{12} \vdash B \ISA
  B'$ implies $\T_1 \cup \T_{12} \vdash B \ISA C'$;
\item \label{req:def-incl-clo-conc-2} if $B'=\exists Q'$, then
  $\exists Q'^- \in \ttype[\Xi]{\Uni_{\srkb}}(o)$ implies $o
  \leadsto_{\srkb} w_{[Q]}$ for some role $Q$ such that $Q'^- \in
  \rtype[\Xi]{\Uni_{\srkb}}(o, o w_{[Q]})$ and $C' \in
  \ttype[\Xi]{\Uni_{\srkb}}(o w_{[Q]})$.
\end{enumerate}
Then, for a pair $R', Q'$ of roles over $\Xi$, we say $\T_1 \cup
\T_{12}$ is \emph{closed under inclusion between} $R'$ \emph{and} $Q'$
if the following is satisfied:
\begin{enumerate}[resume*=reqs]
\item \label{req:def-incl-clo-roles-1} $\T_1 \cup \T_{12} \vdash R
  \ISA R'$ implies $\T_1 \cup \T_{12} \vdash R \ISA Q'$ for each
  $\T_1$-consistent role $R$ over $\Sigma$;
%
\item \label{req:def-incl-clo-roles-2} $\T_1 \cup \T_{12}$ is closed
  under inclusion between $\exists R'$ and $\exists Q'$;
  %
\item \label{req:def-incl-clo-roles-2-inv} $\T_1 \cup \T_{12}$ is
  closed under inclusion between $\exists R'^-$ and $\exists Q'^-$.
\end{enumerate}

Next, we say $\T_1 \cup \T_{12}$ is \emph{closed under disjointness
  between} $B'$ and $C'$ if the following is satisfied:
\begin{enumerate}[reqs]
\item \label{req:def-disj-clo-conc-1} for each $\T_1 \cup
  \T_{12}$-consistent pair of concepts $B,C$ over $\Sigma$ it is
  \emph{not the case} $\T_1 \cup \T_{12} \vdash B \ISA B'$ and $\T_1
  \cup \T_{12} \vdash C \ISA C'$; %
\item \label{req:def-disj-clo-conc-2} for each $\T_1 \cup
  \T_{12}$-consistent concept $B$ over $\Sigma_1$ and each role $R$
  such that $o \leadsto_{\srkb} w_{[R]}$ it is \emph{not the case}
  $B', C' \in \ttype[\Xi]{\Uni_{\srkb}}(o w_{[R]})$.
\end{enumerate}
Then, $\T_1 \cup \T_{12}$ is \emph{closed under disjointness between}
$R'$ \emph{and} $Q'$ if the following is satisfied:
\begin{enumerate}[reqs]
\item \label{req:def-disj-clo-roles-1} for each $\T_1 \cup
  \T_{12}$-consistent pair of roles $R,Q$ over $\Sigma$ it is
  \emph{not the case} $\T_1 \cup \T_{12} \vdash R \ISA R'$ and $\T_1
  \cup \T_{12} \vdash Q \ISA Q'$;
%
\item \label{req:def-disj-clo-roles-2} for each $\T_1 \cup
  \T_{12}$-consistent concept $B$ over $\Sigma_1$ and each role $R$
  such that $o \leadsto_{\srkb} w_{[R]}$ it is \emph{neither the case}
  $R', Q' \in \rtype[\Xi]{\Uni_{\srkb}}(o, o w_{[R]})$ \emph{nor} $R'^-, Q'^-
  \in \rtype[\Xi]{\Uni_{\srkb}}(o, o w_{[R]})$
\end{enumerate}

Define a \emph{generating pass for a concept } $B$ over $\Sigma$ as a
pair $\pi=(\langle C_0, C_1, \dots C_n \rangle, L)$, where $\langle
C_0, C_1, \dots C_n \rangle$ a is tuple of concepts of the length
greater or equal $1$, $C_0=B$, and for each $1 \leq i \leq n$ it holds
$C_i=\SOMET{Q_i^-}$ for some role $Q_i$; then $L$ is a labeling
function
$$L: C_i \cup C_i \times C_j \mapsto 2^{\Xi \text{ - concepts}}
\cup 2^{\Xi \text{ -roles}}$$
such that $L(C_i, C_j)= \emptyset$ for $j \neq i+1$. It is said that a
generating pass $\pi$ for $B$ is \emph{conform with} $\T_1 \cup
\T_{12}$ if the following is satisfied:
\begin{enumerate}[reqs]
\item \label{req:xi-pass-def-chain}$\exists Q \in L(C_i)$ or $\exists
  Q=C_i$ for all $0 \leq i < n$ and roles $Q$ such that $C_{i+1} =
  \exists Q^-$;
\item \label{req:xi-pass-def-conc} For each $0 \leq i \leq n$ and $B'
  \in L(C_i)$ there exists $C'$ over $\Xi$ such that $\T_{12} \vdash
  C_i \ISA C'$ and $\T_1 \cup \T_{12}$ is closed under inclusion
  between $C'$ and $B'$.
\item \label{req:xi-pass-def-role} For each $0 \leq i < n$, role $Q$
  such that $C_{i+1} = \exists Q^-$ and $R' \in L(C_i, C_{i+1})$ there
  exists $Q'$ over $\Xi$ such that $\T_{12} \vdash Q \ISA Q'$ and
  $\T_1 \cup \T_{12}$ is closed under inclusion between $Q'$ and $R'$.
\end{enumerate}

\subsection{Basic Preliminary Results}

  \begin{lemma}\label{lem:t2-repr-then-closed}
    Let $\M=(\Sigma, \Xi, \T_{12})$ be a mapping, and a $\Sigma$-TBox
    $\T_2$ be is a representation for a $\Xi$-TBox $\T_1$ under
    $\T_{12}$. Then $\T_1 \cup \T_{12}$ is closed under:
    \begin{enumerate}[thmparts, start=1]
    \item \label{tmpt:lem-t2-repr-then-closed-incl} inclusion between
      concepts $B'$ and $C'$ (roles $R'$ and $Q'$) for all $B',C'$
      over $\Xi$ ($R',Q'$ over $\Xi$) such that $\T_2 \vdash B' \ISA
      C'$ ($\T_2 \vdash R' \ISA Q'$);
    \item \label{tmpt:lem-t2-repr-then-closed-disj} disjointness
      between concepts $B'$ and $C'$ (roles $R'$ and $Q'$) for all
      $B',C'$ over $\Xi$ ($R',Q'$ over $\Xi$) such that $\T_2 \vdash
      B' \ISA D'$, $\T_2 \vdash C' \ISA E'$, and $(D' \AND E' \ISA
      \bot) \in \T_2$ for some concepts $D', E'$ over $\Xi$ ($\T_2
      \vdash R' \ISA S'$, $\T_2 \vdash Q' \ISA T'$, and $(S' \AND T'
      \ISA \bot) \in \T_2$ for some roles $S', T'$ over $\Xi$);
    \item \label{tmpt:lem-t2-repr-then-closed-incons} disjointness
      between $B'$ and $B'$ ($R'$ and $R'$) for all $\T_2$
      inconsistent concepts $B'$ (roles $R'$).
    \end{enumerate}
  \end{lemma}
  \begin{proof}
    We assume that $\T_2$ is a representation, but
    \ref{tmpt:lem-t2-repr-then-closed-incl},
    \ref{tmpt:lem-t2-repr-then-closed-disj} or
    \ref{tmpt:lem-t2-repr-then-closed-incons} is violated, and derive
    a contradiction. Let, first,
    \ref{tmpt:lem-t2-repr-then-closed-incl} be violated for concepts,
    i.e., there are $B',C'$ over $\Xi$ such that $\T_2 \vdash B' \ISA
    C'$ and $\T_1 \cup \T_{12}$ is \emph{not} closed under inclusion
    between $B'$ and $C'$. Then, \ref{req:def-incl-clo-conc-1} or
    \ref{req:def-incl-clo-conc-2} must be violated for some $B \in
    \consc$ over $\Sigma$. Assume \ref{req:def-incl-clo-conc-1} is
    violated, i.e., $B' \in \ttype[\Xi]{\Uni_{\srkb}}(o)$ and $C' \not
    \in \ttype[\Xi]{\Uni_{\srkb}}(o)$. By Lemma
    \ref{lem:t2-is-repr-iff-conds}
    \ref{req:t2-is-repr-iff-conds-conc-equal} we get the
    contradiction. If \ref{req:def-incl-clo-conc-2} is violated, i.e.,
    $B'= \exists Q'$, $\exists Q'^- \in \ttype[\Xi]{\Uni_{\srkb}}(o)$,
    and for all roles $Q$ such that $o \leadsto_{\srkb} w_{[Q]}$ and
    $Q'^- \in \rtype[\Xi]{\Uni_{\srkb}}(o, o w_{[Q]})$ it holds $C' \not \in
    \ttype[\Xi]{\Uni_{\srkb}}(o)$. By Lemma \ref{lem:t2-is-repr-iff-conds}
    \ref{req:t2-is-repr-iff-conds-conc-equal} obtain $\exists Q'^- \in
    \ttype[\Xi]{\Uni_{\tgkb}}(o)$ and since $\T_2 \vdash B' \ISA C'$ it
    follows there exists a role $Q$ such that $o \leadsto_{\tgkb}
    w_{[Q]}$, $Q'^- \in \rtype[\Xi]{\Uni_{\tgkb}}(o, o w_{[Q]})$, and $C' \in
    \ttype[\Xi]{\Uni_{\tgkb}}(o w_{[Q]})$. Using
    \ref{req:t2-is-repr-iff-conds-roles-gen-inv} we obtain a
    contraction.

    Suppose there are roles $R',Q'$ over $\Xi$ such that $\T_2 \vdash
    R' \ISA Q'$ and $\T_1 \cup \T_{12}$ is \emph{not} closed under
    inclusion between $R'$ and $Q'$, then one of
    \ref{req:def-incl-clo-roles-1}, \ref{req:def-incl-clo-roles-2},
    \ref{req:def-incl-clo-roles-2-inv} is violated. Assume it is
    \ref{req:def-incl-clo-roles-1}, then $\T_1 \cup \T_{12} \vdash R
    \ISA R'$ and it is not the case $\T_1 \cup \T_{12} \vdash R \ISA
    Q'$ for some $R \in \consr$ over $\Sigma$. Using Lemma
    \ref{lem:t2-is-repr-iff-conds}
    \ref{req:t2-is-repr-iff-conds-roles-equal} we get the
    contradiction. Assume \ref{req:def-incl-clo-roles-2} is violated,
    then there is $B \in \consc$ over $\Sigma$ such that $\exists R'^-
    \in \ttype[\Xi]{\Uni_{\srkb}}(o)$ and for all roles $Q$ such that
    $o \leadsto_{\srkb} w_{[Q]}$ and $R'^- \in
    \rtype[\Xi]{\Uni_{\srkb}}(o, o w_{[Q]})$ it holds $\exists Q' \not
    \in \ttype[\Xi]{\Uni_{\srkb}}(o)$. By Lemma
    \ref{lem:t2-is-repr-iff-conds}
    \ref{req:t2-is-repr-iff-conds-conc-equal} obtain $\exists R'^- \in
    \stype[\Xi]{\Uni_{\tgkb}}(o)$ and since $\T_2 \vdash R' \ISA Q'$
    it follows there exists a role $Q$ such that $o \leadsto_{\tgkb}
    w_{[Q]}$, $R'^- \in \rtype[\Xi]{\Uni_{\tgkb}}(o, o w_{[Q]})$, and
    $\exists Q' \in \ttype[\Xi]{\Uni_{\tgkb}} (o w_{[Q]})$. Using
    \ref{req:t2-is-repr-iff-conds-roles-gen-inv} we obtain a
    contraction. The proof when \ref{req:def-incl-clo-roles-2-inv} is
    violated is analogous.

    Let now \ref{tmpt:lem-t2-repr-then-closed-disj} be violated for
    concepts, i.e., there are $B',C', D', E'$ over $\Xi$ such that
    $\T_2 \vdash B' \ISA D'$, $\T_2 \vdash C' \ISA E'$, $(D' \AND E'
    \ISA \bot) \in \T_2$ and $\T_1 \cup \T_{12}$ is \emph{not} closed
    under disjointness between $B'$ and $C'$.  Then,
    \ref{req:def-disj-clo-conc-1} or \ref{req:def-disj-clo-conc-2}
    must be violated. If it is \ref{req:def-disj-clo-conc-1} there is
    a $\T_1 \cup \T_{12}$-consistent pair of concepts $B,C$ such that
    $B' \in \ttype[\Xi]{\Uni_{\srkb}}(o)$ and $C' \in
    \ttype[\Xi]{\Uni_{\srkb[C]}}(o)$. By Lemma
    \ref{lem:t2-is-repr-iff-conds}
    \ref{req:t2-is-repr-iff-conds-conc-equal} obtain $B' \in
    \ttype[\Xi]{\Uni_{\tgkb}}(o)$ and $C' \in
    \ttype[\Xi]{\Uni_{\tgkb[C]}}(o)$, then by $\T_2 \vdash B' \ISA
    D'$, $\T_2 \vdash C' \ISA E'$ and $(D' \AND E' \ISA \bot) \in
    \T_2$ it follows the pair $B,C$ is $\T_2 \cup \T_{12}$
    inconsistent. We obtained a contradiction to Lemma
    \ref{lem:t2-is-repr-iff-conds}
    \ref{req:t2-is-repr-iff-conds-conc-cons}. If
    \ref{req:def-disj-clo-conc-2} is violated there is $B \in \consc$
    over $\Sigma$ and a role $R$ such that $o \leadsto_{\srkb}
    w_{[R]}$ and $B',C' \in \ttype[\Xi]{\Uni_{\srkb}}(o w_{[R]})$. By
    Lemma \ref{lem:t2-is-repr-iff-conds}
    \ref{req:t2-is-repr-iff-conds-roles-gen} there is $y \in
    \dom[\Uni_{\tgkb}]$ such that $B', C' \in
    \ttype[\Xi]{\Uni_{\tgkb}}(y)$. Using Lemmas~\ref{lem:types-chase},
    \ref{lem:kr-to-all-b-cons}, $\T_2 \vdash B' \ISA D'$, $\T_2 \vdash
    C' \ISA E'$ and $(D' \AND E' \ISA \bot) \in \T_2$ it follows $B$
    is $\T_2 \cup \T_{12}$ inconsistent, which is a contradiction to
    $B \in \consc$ over $\Sigma$ by Lemma
    \ref{lem:t2-is-repr-iff-conds}
    \ref{req:t2-is-repr-iff-conds-conc-cons}. The proof of the case
    when \ref{tmpt:lem-t2-repr-then-closed-disj} is violated for roles
    is analogous to the case of concepts above.

    Finally, assume \ref{tmpt:lem-t2-repr-then-closed-incons} is
    violated for concepts, i.e., there is $\T_2$ inconsistent $B'$
    such that $\T_2 \cup \T_{12}$ is not closed under inclusion
    between $B'$ and $B'$. It follows \ref{req:def-disj-clo-conc-1} or
    \ref{req:def-disj-clo-conc-2} must be violated. If it is
    \ref{req:def-disj-clo-conc-1} there is a $\T_1 \cup
    \T_{12}$-consistent pair of concepts $B,C$ such that $B' \in
    \ttype[\Xi]{\Uni_{\srkb}}(o)$ and $B' \in
    \ttype[\Xi]{\Uni_{\srkb[C]}}(o)$. By Lemma
    \ref{lem:t2-is-repr-iff-conds}
    \ref{req:t2-is-repr-iff-conds-conc-equal} obtain $B' \in
    \ttype[\Xi]{\Uni_{\tgkb}}(o)$ and $B' \in
    \ttype[\Xi]{\Uni_{\tgkb[C]}}(o)$, then by Lemmas
    \ref{lem:types-chase} and \ref{lem:kr-to-all-b-cons} we obtain
    that the pair $B,C$ is $\T_2 \cup \T_{12}$-inconsistent. We
    obtained a contradiction to Lemma \ref{lem:t2-is-repr-iff-conds}
    \ref{req:t2-is-repr-iff-conds-conc-cons}. If
    \ref{req:def-disj-clo-conc-2} is violated there is $B \in \consc$
    over $\Sigma$ and a role $R$ such that $o \leadsto_{\srkb}
    w_{[R]}$ and $B' \in \ttype[\Xi]{\Uni_{\srkb}}(o w_{[R]})$. By
    Lemma \ref{lem:t2-is-repr-iff-conds}
    \ref{req:t2-is-repr-iff-conds-roles-gen} there is $y \in
    \dom[\Uni_{\tgkb}]$ such that $B' \in
    \ttype[\Xi]{\Uni_{\tgkb}}(y)$. Using Lemmas \ref{lem:types-chase}
    and \ref{lem:kr-to-all-b-cons} it follows $B$ is $\T_2 \cup
    \T_{12}$-inconsistent, which is a contradiction to $B \in \consc$
    over $\Sigma$ by Lemma \ref{lem:t2-is-repr-iff-conds}
    \ref{req:t2-is-repr-iff-conds-conc-cons}. The proof of the case
    when \ref{tmpt:lem-t2-repr-then-closed-incons} is violated for
    roles is analogous to the case of concepts above.
  \end{proof}

\subsection{Proof of Proposition 6.3}

The result is shown in
Theorem~\ref{th:non-empt-repres-nlogspace-compl}; we need a series of
lemmas before we present the proof.

\begin{lemma}\label{lem:representab-conds}
  Given a mapping $\M=(\Sigma, \Xi, \T_{12})$ and a $\Sigma$-TBox
  $\T_1$, there exists $\Xi$-TBox $\T_2$, such that it is a
  $\UCQ$-representation of $\T_1$ under $\M$, if and only if the
  following conditions are satisfied:
  \begin{enumerate}[thmparts, start=1]
  \item \label{req:representab-conds-equal-conc} For each $\T_1 \cup
    \T_{12}$-consistent concept $B$ over $\Sigma$ and each $B'$ over
    $\Xi$ such that $\T_1 \cup \T_{12} \vdash B \ISA B'$ there exists
    $C'$ over $\Xi$ such that $\T_{12} \vdash B \ISA C'$ and $\T_1
    \cup \T_{12}$ is closed under the inclusion between $C'$ and $B'$.
  \item \label{req:representab-conds-equal-role} For each $\T_1 \cup
    \T_{12}$-consistent role $R$ over $\Sigma$ and each $R'$ over
    $\Xi$ such that $\T_1 \cup \T_{12} \vdash R \ISA R'$ there exists
    $Q'$ over $\Xi$ such that $\T_{12} \vdash R \ISA Q'$ and $\T_1
    \cup \T_{12}$ is closed under inclusion between $Q'$ and $R'$.
  \item \label{req:representab-conds-gen-role} For each $\T_1 \cup
    \T_{12}$-consistent concept $B$ over $\Sigma$ and each role $R$
    such that $o \leadsto_{\srkb} w_{[R]}$ there exists a generating
    pass $\pi = (\langle C_0, \dots C_n \rangle, L)$ for $B$ conform
    with $\T_1 \cup \T_{12}$, such that:
    \begin{enumerate}[label=\textbf{(\alph*)}]
    \item $\ttype[\Xi]{\Uni_{\srkb}}(o w_{[R]}) \subseteq L(C_n)$.
    \item $\rtype[\Xi]{\Uni_{\srkb}}(o, o w_{[R]}) \subseteq L(C_0,
      C_n)$;
    \end{enumerate}
  \item \label{req:representab-conds-conc-cons} For each
    $\T_1$-consistent pair of concepts $B_1, B_2$ over $\Sigma$, such
    that $B_1, B_2$ is $\T_1 \cup \T_{12}$-inconsistent, there are
    concepts $B,C$ such that one of the following holds:
    \begin{enumerate}[label=\textbf{(\alph*)}]
    \item \label{req:representab-conds-conc-cons-1} $B,C \in \{B_1,
      B_2\}$ and one of the following holds:
      \begin{enumerate}[label=(\arabic*)]
      \item \label{req:representab-conds-conc-cons-1-1} $\T_{12}
        \vdash B \ISA B'$, $\T_{12} \vdash C \ISA C'$, and $\T_1 \cup
        \T_{12}$ is closed under the disjointness between $B'$ and
        $C'$;
      \item \label{req:representab-conds-conc-cons-1-2} $\T_{12}
        \vdash B \ISA B'$, $\T_{12} \ni (C \AND C' \ISA \bot)$, and
        $\T_1 \cup \T_{12}$ is closed under inclusion between $B'$ and
        $C'$.
      \end{enumerate}
    \item \label{req:representab-conds-conc-cons-2} $\exists R \in
      \{B_1, B_2\}$ and one of the following holds:
      \begin{enumerate}[label=(\arabic*)]
      \item \label{req:representab-conds-conc-cons-2-1} $\T_{12}
        \vdash \exists R^- \ISA B'$, $\T_{12} \vdash \exists R^- \ISA
        C'$, and $\T_1 \cup \T_{12}$ is closed under disjointness
        between $B'$ and $C'$;
      \item \label{req:representab-conds-conc-cons-2-2} $\T_{12}
        \vdash \exists R^- \ISA B'$, $\T_{12} \ni (\exists R^- \AND C'
        \ISA \bot)$, and $\T_1 \cup \T_{12}$ is closed under inclusion
        between $B'$ and $C'$;
      \item \label{req:representab-conds-conc-cons-2-3} $\T_{12}
        \vdash R \ISA R'$, $\T_{12} \vdash R \ISA Q'$, and $\T_1 \cup
        \T_{12}$ is closed under the disjointness between $R'$ and
        $Q'$;
      \item \label{req:representab-conds-conc-cons-2-4} $\T_{12}
        \vdash R \ISA R'$, $\T_{12} \ni (R \AND Q' \ISA \bot)$, and
        $\T_1 \cup \T_{12}$ is closed under inclusion between $R'$ and
        $Q'$.
      \end{enumerate}
    \end{enumerate}
  \item \label{req:representab-conds-role-cons} For all
    $\T_1$-consistent pairs of roles $R_1, R_2$, such that $R_1, R_2$
    is $\T_1 \cup \T_{12}$-inconsistent one of the following holds:
    \begin{enumerate}[label=\textbf{(\alph*)}]
    \item \label{req:representab-conds-role-cons-1} there are roles
      $R,Q \in \{R_1, R_2\}$ and $R', Q'$ over $\Xi$ such that one of
      the following holds:
      \begin{enumerate}[label=(\arabic*)]
      \item \label{req:representab-conds-role-cons-1-1}$\T_{12} \vdash
        R \ISA R'$, $\T_{12} \vdash Q \ISA Q'$, and $\T_1 \cup
        \T_{12}$ is closed under disjointness between $R'$ and $Q'$;
      \item \label{req:representab-conds-role-cons-1-2} $\T_{12}
        \vdash R \ISA R'$, $\T_{12} \ni (Q \AND Q' \ISA \bot)$, and
        $\T_1 \cup \T_{12}$ is closed under inclusion between $R'$ and
        $Q'$;
      \end{enumerate}
    \item \label{req:representab-conds-role-cons-2} there exist $B, C
      \in \{\exists R_1, \exists R_2\}$ or $B, C \in \{\exists R_1^-,
      \exists R_2^-\}$ such that one of the following holds:
      \begin{enumerate}[label=(\arabic*)]
      \item \label{req:representab-conds-role-cons-2-1} $\T_{12}
        \vdash B \ISA B'$, $\T_{12} \vdash C \ISA C'$, and $\T_1 \cup
        \T_{12}$ is closed under disjointness between $B'$ and $C'$;
      \item \label{req:representab-conds-role-cons-2-2} $\T_{12}
        \vdash B \ISA B'$, $\T_{12} \ni (C \AND C' \ISA \bot)$, and
        $\T_1 \cup \T_{12}$ is closed under inclusion between $B'$ and
        $C'$.
      \end{enumerate}
    \end{enumerate}
  \end{enumerate}
\end{lemma}

\begin{proof}
  ($\Leftarrow$) Assume the conditions
  \ref{req:representab-conds-equal-conc} --
  \ref{req:representab-conds-role-cons} are satisfied, we construct a
  TBox $\T_2$ and prove it is a UCQ-representation for $\T_1$ under
  $\M$. The required $\T_2$ will be given as the union of the fives
  sets of axioms presented below. First, take $B \in \consc$ over
  $\Sigma$, $B' \in \ttype[\Xi]{\Uni_{\srkb}}(o)$, then let $ax_1(B,
  B') = \{C' \ISA B'\}$ for $C'$ given by the condition
  \ref{req:representab-conds-equal-conc}. For $R \in \consr$ over
  $\Sigma$ and $R'$ over $\Xi$, such that $\T_1 \cup \T_{12} \vdash R
  \ISA R'$, define $ax_2(R, R') = \{Q' \ISA R'\}$ for $Q'$ given by
  the condition \ref{req:representab-conds-equal-role}. For each $B
  \in \consc$ over $\Sigma$ and each role $R$ such that $o
  \leadsto_{\srkb} w_{[R]}$ define the set $ax_3(B,R)$ from the
  generating pass $\tup{C_0, \dots, C_n}$ for $B$ conform with $\T_1
  \cup \T_{12}$ that satisfies
  \ref{req:representab-conds-gen-role}. Take $ax_3(B,R)$ equal to the
  set of all axioms $C' \ISA B'$ satisfying \ref{req:xi-pass-def-conc}
  and all axioms $Q' \ISA R'$ satisfying
  \ref{req:xi-pass-def-role}. Now let $B_, B_2$ be a $\T_1$-consistent
  and $\T_1 \cup \T_{12}$-inconsistent pair of $\Sigma$ concepts, then
  define a set $ax_4(B_1, B_2)$ to be equal to $\{ B' \AND C' \ISA
  \bot\}$ for the corresponding $B'$ and $C'$, if
  \ref{req:representab-conds-conc-cons}\ref{req:representab-conds-conc-cons-1}\ref{req:representab-conds-conc-cons-1-1}
  or
  \ref{req:representab-conds-conc-cons}\ref{req:representab-conds-conc-cons-2}\ref{req:representab-conds-conc-cons-2-1}
  is satisfied; and $\{R' \AND Q' \ISA \bot\}$ for the corresponding
  $R'$ and $Q'$, if
  \ref{req:representab-conds-conc-cons}\ref{req:representab-conds-conc-cons-2}\ref{req:representab-conds-conc-cons-2-3}
  is satisfied. On the other hand, define $ax_4(B_1, B_2)$ to be equal
  to $\{ B' \ISA C'\}$ for the corresponding $B'$ and $C'$, if
  \ref{req:representab-conds-conc-cons}\ref{req:representab-conds-conc-cons-1}\ref{req:representab-conds-conc-cons-1-2}
  or
  \ref{req:representab-conds-conc-cons}\ref{req:representab-conds-conc-cons-2}\ref{req:representab-conds-conc-cons-2-2}
  is satisfied; and $\{R' \ISA Q'\}$ for the corresponding $R'$ and
  $Q'$, if
  \ref{req:representab-conds-conc-cons}\ref{req:representab-conds-conc-cons-2}\ref{req:representab-conds-conc-cons-2-4}
  is satisfied. Finally, we define $ax_5(R_1, R_2)$ for
  $\T_1$-consistent and $\T_1 \cup \T_{12}$-inconsistent pair of
  $\Sigma$ roles $R_1, R_2$ analogously to $ax_4(B_1, B_2)$ using the
  conditions
  \ref{req:representab-conds-role-cons}\ref{req:representab-conds-role-cons-2}
  and
  \ref{req:representab-conds-role-cons}\ref{req:representab-conds-role-cons-2}.
  Finally we have:
  \begin{align*}
    \T_2 = \hspace{-.4cm}&\bigcup_{%
      \begin{subarray}{c}
        B \in \consc \text{ over } \Sigma,\\
        B' \in \ttype[\Xi]{\Uni_{\srkb}}(o)
      \end{subarray}%
    } \hspace{-.3cm} ax_3(B,B') \cup \bigcup_{%
      \begin{subarray}{c}
        R \in \consr \text{ over } \Sigma,\\
        R' \text{ over } \Xi,\, \T_1 \cup \T_{12} \vdash R \ISA R'
      \end{subarray}%
    } \hspace{-.4cm} ax_4(R,R') \cup \\
    & \bigcup_{%
      \begin{subarray}{c}
        B \in \consc \text{ over } \Sigma, \\
        o \leadsto_{\srkb} w_{[R]}
      \end{subarray}%
    } \hspace{-.5cm} ax_5(B,R) \cup \hspace{-.3cm} \bigcup_{%
      \begin{subarray}{c}
        B_0, B_1 \text{ conc. over } \Sigma,\\
        \T_1-\text{ consist. and }\\
        \T_1 \cup \T_{12}-\text{ inconsist.}
      \end{subarray}%
    } \hspace{-.3cm} ax_1(B_0, B_1) \cup \hspace{-.3cm} \bigcup_{%
      \begin{subarray}{c}
        R_0, R_1 \text{ roles over } \Sigma,\\
        \T_1-\text{ consist. and }\\
        \T_1 \cup \T_{12}-\text{ inconsist.}
      \end{subarray}%
    } \hspace{-.3cm} ax_2(R_0, R_1)
  \end{align*}
  We need the following intermediate result:
  \begin{lemma} \label{lem:t2-der-closed} For all concepts $B', C' \in
    \Xi$ (roles $R', Q'$ over $\Xi$), if $\T_2 \vdash B' \ISA C'$
    ($\T_2 \vdash R' \ISA Q'$) then $\T_1 \cup \T_{12}$ is closed
    under inclusion between $B'$ and $C'$ ($R'$ and $Q'$).
  \end{lemma}
  \begin{proof}
    Notice that for all concepts $B'$ and $C'$ (roles $R'$ and $Q'$)
    such that $(B' \ISA C') \in \T_2$ ($(R' \ISA Q') \in \T_2$) it
    holds $\T_1 \cup \T_{12}$ is closed under inclusion between $B'$
    and $C'$ ($R'$ and $Q'$). First we prove the statement of the
    lemma for roles, if $\T_2 \vdash R' \ISA Q'$ there is a sequence
    of roles $Q_1, \dots, Q_n$ such that $Q_1=R'$, $Q_n = Q'$, and for
    each $1 \leq i < n$ one of the following holds:
    \begin{enumerate}[resume*=reqs]
    \item \label{req:t2-ind-roles} $(Q_i \ISA Q_{i+1}) \in \T_2$
    \item \label{req:t2-ind-roles-inv} $(Q_i^- \ISA Q_{i+1}^-) \in
      \T_2$
    \end{enumerate}
    We show $\T_1 \cup \T_{12}$ is closed under inclusion between $R'$
    and $Q_i$ by induction on $i$. For $i=1$ the proof is trivial,
    assume $\T_1 \cup \T_{12}$ is closed under inclusion between $R'$
    and $Q_i$, we show now its closure under inclusion between $R'$
    and $Q_{i+1}$. Let, first, \ref{req:t2-ind-roles}, we show
    \ref{req:def-incl-clo-roles-1}. Assume $\T_1 \cup \T_{12} \vdash R
    \ISA R'$ for some $R \in \consr$ over $\Sigma$, since $\T_1 \cup
    \T_{12}$ is closed under inclusion between $R'$ and $Q_i$, it
    follows by \ref{req:def-incl-clo-roles-1} $\T_1 \cup \T_{12}
    \vdash R \ISA Q_i$, then, again by closure under inclusion between
    $Q_i$ and $Q_{i+1}$ obtain $\T_1 \cup \T_{12} \vdash R \ISA
    Q_{i+1}$.

    To show \ref{req:def-incl-clo-roles-2} we need to prove
    \ref{req:def-incl-clo-conc-1} and \ref{req:def-incl-clo-conc-2}
    for $B'=\exists R'$ and $C'=\exists Q_{i+1}$. For
    \ref{req:def-incl-clo-conc-1} assume $\exists R' \in
    \ttype[\Xi]{\Uni_{\srkb}}(o)$ for some $B \in \consc$ over
    $\Sigma$; since $\T_1 \cup \T_{12}$ is closed under inclusion
    between $R'$ and $Q_i$, it follows by
    \ref{req:def-incl-clo-roles-2} that $\exists Q_i \in
    \ttype[\Xi]{\Uni_{\srkb}}(o)$, and again by closure under
    inclusion between $Q_i$ and $Q_{i+1}$ obtain $\exists Q_{i+1} \in
    \ttype[\Xi]{\Uni_{\srkb}}(o)$.  For \ref{req:def-incl-clo-conc-2}
    assume $\exists R'^- \in \ttype[\Xi]{\Uni_{\srkb}}(o)$ for some $B
    \in \consc$ over $\Sigma$ and consider two cases: $R'=Q_i$ and $R'
    \neq Q_i$. In the first case $o \leadsto_{\srkb} w_{[Q]}$ for some
    role $Q$ such that $R'^- \in \rtype[\Xi]{\Uni_{\srkb}}(o, o w_{[Q]})$ and
    $\exists Q_{i+1} \in \ttype[\Xi]{\Uni_{\srkb}}(o w_{[Q]})$ immediately
    follows, since $\T_1 \cup \T_{12}$ is closed under inclusion
    between $Q_i$ and $Q_{i+1}$.

    Assume $R' \neq Q_i$, since $\T_1 \cup \T_{12}$ is closed under
    inclusion between $R'$ and $Q_i$, it follows by
    \ref{req:def-incl-clo-roles-2} and the structure of $\srkb$ that
    $o \leadsto_{\srkb} w_{[Q]}$ for some role $Q$ over $\Sigma$ such
    that $R'^- \in \rtype[\Xi]{\Uni_{\srkb}}(o, o w_{[Q]})$ and
    $\exists Q_i \in \ttype[\Xi]{\Uni_{\srkb}}(o w_{[Q]})$. Since
    $\srkb$ is consistent, it can be easily shown that $\exists Q^-
    \in \consc$. Observe now that $\exists Q_i \in
    \ttype[\Xi]{\Uni_{\srkb[\exists Q^-]}}(o)$; then since $\T_1 \cup
    \T_{12}$ is closured under inclusion between $Q_i$ and $Q_{i+1}$
    and \ref{req:def-incl-clo-roles-2-inv}, obtain $\exists Q_{i+1}
    \in \ttype[\Xi]{\Uni_{\srkb[\exists Q^-]}}(o)$. Finally, it
    follows $\exists Q_{i+1} \in \ttype[\Xi]{\Uni_{\srkb}}(o
    w_{[Q]})$, which completes the proof for the case
    \ref{req:t2-ind-roles}. The proof for the case
    \ref{req:t2-ind-roles-inv} is analogous.

    To prove the lemma for concepts we exploit that $\T_2 \vdash B'
    \ISA C'$ implies there exists a sequence of $\Xi$ concepts $B_1,
    \dots, B_n$ such that $B_1=B'$, $B_n=C'$, and for each $1 \leq i <
    n$ one of the following holds:
    \begin{enumerate}[resume*=reqs]
    \item \label{req:t2-ind-conc} $(B_i \ISA B_{i+1}) \in \T_2$
    \item \label{req:t2-ind-conc-roles} $B_i = \exists R'$, $B_{i+1} =
      \exists Q'$, $(R' \ISA Q') \in \T_2$
    \item \label{req:t2-ind-conc-roles-inv} $B_i = \exists R'^-$,
      $B_{i+1} = \exists Q'^-$, $(R' \ISA Q') \in \T_2$
    \end{enumerate}
    We show $\T_1 \cup \T_{12}$ is closed under inclusion between $B'$
    and $B_i$ by induction on $i$. For $i=1$ the proof is trivial,
    assume $\T_1 \cup \T_{12}$ is closed under inclusion between $B'$
    and $B_i$, we show now its closure under inclusion between $B'$
    and $B_{i+1}$. First we consider the case of $B_i$ and $B_{i+1}$
    are as in \ref{req:t2-ind-conc}. To show
    \ref{req:def-incl-clo-conc-1} assume $B \in \consc$ over $\Sigma$
    and $B' \in \ttype[\Xi]{\Uni_{\srkb}}(o)$; by closure under
    inclusion between $B'$ and $B_i$ and \ref{req:def-incl-clo-conc-1}
    it follows $B_i \in \ttype[\Xi]{\Uni_{\srkb}}(o)$, then by closure
    under inclusion between $B_i$ and $B_{i+1}$ and
    \ref{req:def-incl-clo-conc-1} obtain $B_{i+1} \in
    \ttype[\Xi]{\Uni_{\srkb}}(o)$.

    To show \ref{req:def-incl-clo-conc-2} assume $B'=\exists Q'$ and
    $\exists Q'^- \in \ttype[\Xi]{\Uni_{\srkb}}(o)$. If $B'=B_i$, then
    $o \leadsto_{\srkb} w_{[Q]}$ for some role $Q$ such that $Q' \in
    \rtype[\Xi]{\Uni_{\srkb}}(o, o w_{[Q]})$ and $B_{i+1} \in
    \ttype[\Xi]{\Uni_{\srkb}}(o w_{[Q]})$ by closure under inclusion
    between $B_i$ and $B_{i+1}$, and \ref{req:def-incl-clo-conc-2}. On
    the other hand, if $B' \neq B_i$, it follows by closure under
    inclusion between $B'$ and $B_i$, and
    \ref{req:def-incl-clo-conc-2} $o \leadsto_{\srkb} w_{[Q]}$ for a
    role $Q$ over $\Sigma$ such that $Q' \in
    \rtype[\Xi]{\Uni_{\srkb}}(o, o w_{[Q]})$ and $B_i \in
    \ttype[\Xi]{\srkb}(o w_{[Q]})$. Since $\srkb$ is consistent, it
    follows $\exists R^- \in \consc$, then by closure under inclusion
    between $B_i$ and $B_{i+1}$ and \ref{req:def-incl-clo-conc-1} we
    conclude $B_{i+1} \in \ttype[\Xi]{\Uni_{\srkb}}(o w_{[R]})$, which
    concludes the proof. The proof for the cases
    \ref{req:t2-ind-conc-roles} and \ref{req:t2-ind-conc-roles-inv} is
    analogios.
  \end{proof}

  We return to the proof of ($\Leftarrow$) of Lemma
  \ref{lem:representab-conds}; we prove $\T_2$ above is a
  representation of $\T_1$ under $\T_{12}$ by showing the conditions
  \ref{req:t2-is-repr-iff-conds-conc-cons} --
  \ref{req:t2-is-repr-iff-conds-roles-gen-inv} of Lemma
  \ref{lem:t2-is-repr-iff-conds} are satisfied. We start from
  \ref{req:t2-is-repr-iff-conds-conc-equal} (consistency conditions
  will be shown in the end.) Let $B \in \consc$ over $\Sigma$, $B' \in
  \ttype[\Xi]{\Uni_{\srkb}}(o)$, then $B' \in
  \ttype[\Xi]{\Uni_{\tgkb}}(o)$ follows straightforwardly from current
  \ref{req:representab-conds-equal-conc}. Assume now some $B' \in
  \ttype[\Xi]{\Uni_{\tgkb}}(o)$, it follows $\T_{12} \vdash B \ISA C'$
  and $\T_2\vdash C' \ISA B'$ for some concept $C'$ over $\Xi$; by
  Lemma \ref{lem:t2-der-closed} it follows $C' \in
  \ttype[\Xi]{\Uni_{\srkb}}(o)$ implies $B' \in
  \ttype[\Xi]{\Uni_{\srkb}}(o)$, and since $\T_{12} \vdash B \ISA C'$
  conclude $B' \in \ttype[\Xi]{\Uni_{\srkb}}(o)$. The proof that
  \ref{req:t2-is-repr-iff-conds-roles-equal} of Lemma
  \ref{lem:representab-conds} is satisfied is analogous to the proof
  that \ref{req:t2-is-repr-iff-conds-conc-equal} is satisfied above,
  using current \ref{req:representab-conds-equal-role} and Lemma
  \ref{lem:t2-der-closed}.

  The \ref{req:t2-is-repr-iff-conds-roles-gen} of Lemma
  \ref{lem:representab-conds} follows straightforwardly from current
  \ref{req:representab-conds-gen-role}, the definition of a $\Xi$ pass
  conform with $\T_1 \cup \T_{12}$, and the structure of $\T_2$. To
  show \ref{req:t2-is-repr-iff-conds-roles-gen-inv} of Lemma
  \ref{lem:representab-conds} assume $B \in \consc$ over $\Sigma$ and
  a role $Q$ such that $o \leadsto_{\tgkb} w_{[Q]}$. We first consider
  the case $Q$ over $\Xi$, by the structure of $\T_2$ (see the proof
  that \ref{req:t2-is-repr-iff-conds-conc-equal} of Lemma
  \ref{lem:t2-is-repr-iff-conds} is satisfied) it follows $\exists Q
  \in \ttype[\Xi]{\Uni_{\srkb}}(o)$. If
  $\ttype[\Xi]{\Uni_{\tgkb}}(w_{[Q]}) = \{\exists Q^-\}$, the proof is
  done; otherwise, $\T_2 \vdash \exists Q^- \ISA C'$ for some $C' \neq
  \exists Q^-$, then by Lemma \ref{lem:t2-der-closed} and
  \ref{req:def-incl-clo-conc-2} it follows there exists $R$ such that
  $o \leadsto_{\srkb} w_{[R]}$, $Q \in \rtype[\Xi]{\Uni_{\srkb}}(o, o
  w_{[R]})$ and $C' \in \ttype[\Xi]{\Uni_{\srkb}}(o)$; also by $C'
  \neq \exists Q^-$ and the structure of $\srkb$ it follows $R$ is
  over $\Sigma$. Notice that $\exists R^- \in \consc$ since $\srkb$ is
  consistent. For each (other) $C' \in \ttype[\Xi]{\Uni_{\tgkb}}(o
  w_{[Q]})$ we show $C' \in
  \ttype[\Xi]{\Uni_{\srkb}}(w_{[R]})$. Indeed, it follows $\T_2 \vdash
  \exists Q^- \ISA C'$; then using Lemma \ref{lem:t2-der-closed},
  \ref{req:def-incl-clo-conc-1}, $\exists Q^- \in
  \ttype[\Xi]{\Uni_{\srkb[\exists R^-]}}(o)$, we can conclude $C' \in
  \ttype[\Xi]{\Uni_{\srkb[\exists R^-]}}(o)$, and so $C' \in
  \ttype[\Xi]{\Uni_{\srkb}}(o w_{[R]})$. To show
  $\rtype[\Xi]{\Uni_{\tgkb}}(o,o w_{[Q]}) \subseteq
  \rtype[\Xi]{\Uni_{\srkb}}(o,o w_{[R]})$ consider that $R \in \consr$
  and assume $Q' \in \rtype[\Xi]{\Uni_{\tgkb}}(o,o w_{[Q]})$; it
  follows $\T_2 \vdash Q \ISA Q'$, then by Lemma
  \ref{lem:t2-der-closed}, \ref{req:def-incl-clo-roles-1} and $\T_1
  \cup \T_{12} \vdash R \ISA Q$ it follows $Q' \in
  \rtype[\Xi]{\Uni_{\srkb}}(o,o w_{[R]})$, which concludes the proof.

  Consider now the case $Q$ over $\Sigma$, then, clearly, $o
  \leadsto_{\srkb} w_{[R]}$ and $Q \in \rtype{\Uni_{\srkb}}(o, o
  w_{[R]})$ for some role $R$ over $\Sigma$. We show now
  $\ttype[\Xi]{\Uni_{\tgkb}}(o w_{[Q]}) \subseteq
  \ttype{\Uni_{\srkb}}(o w_{[R]})$: let $C' \in
  \ttype[\Xi]{\Uni_{\tgkb}}(o w_{[Q]})$, then $\T_{12} \vdash \exists
  Q^- \ISA B'$ and $\T_2 \vdash B' \ISA C'$ for some $B'$ over
  $\Xi$. It follows $B' \in \ttype[\Xi]{\Uni_{\srkb}}(o w_{[R]})$,
  then by Lemma~\ref{lem:t2-der-closed} and
  \ref{req:def-incl-clo-conc-1} obtain $C' \in
  \ttype[\Xi]{\Uni_{\srkb}}(o w_{[R]})$. The proof of
  $\rtype[\Xi]{\Uni_{\tgkb}}(o, o w_{[Q]}) \subseteq
  \rtype{\Uni_{\srkb}}(o, o w_{[R]})$ is analogous.

  Now we show that the consistency conditions of
  Lemma~\ref{lem:t2-is-repr-iff-conds} are satisfied. For
  \ref{req:t2-is-repr-iff-conds-conc-cons} assume a pair $B_1, B_2$ of
  $\T_1$ consistent and $\T_1 \cup \T_{12}$-inconsistent concepts;
  then $B_1, B_2$ is $\T_2 \cup \T_{12}$-inconsistent follows easily
  from current \ref{req:representab-conds-conc-cons} and definition of
  $\T_2$. Assume $B_1, B_2$ are $\T_1$ consistent and $\T_2 \cup
  \T_{12}$-inconsistent; it follows there exists $\delta, \sigma \in
  \Delta^{\tgkb[\{B_1(o), B_2(o)\}]}$ such that one of the following
  holds:
  \begin{enumerate}[reqs]
  \item There are concepts $C, C' \in \ttype{\Uni_{\tgkb[\{B_1(o),
        B_2(o)\}]}}(\delta)$ such that $(C \AND C' \ISA \bot) \in \T_2
    \cup \T_{12}$;
  \item There are roles $Q, Q' \in \rtype{\Uni_{\tgkb[\{B_1(o),
        B_2(o)\}]}}(\delta, \sigma)$ such that $(Q \AND Q' \ISA \bot)
    \in \T_2 \cup \T_{12}$.
  \end{enumerate}
  Assume for the sake of contradiction that $B_1, B_2$ is $\T_1 \cup
  \T_{12}$ is consistent. By Lemma \ref{lem:kr-to-all-b-cons} it
  follows for each $\delta, \sigma \in \dom[\Uni_{\srkb[ \{B_1(o),
    B_2(o)\}]}]$, every $B \in \ttype[\Sigma]{\Uni_{\srkb[\{B_1(o),
      B_2(o)\}]}}(\delta)$ and $R \in
  \rtype[\Sigma]{\Uni_{\srkb[\{B_1(o), B_2(o)\}]}}(\delta, \sigma)$
  are $\T_1 \cup \T_{12}$-consistent. By the structure of $\T_2$ for
  all such $B$ and $R$ the conditions \ref{req:c1-o-subs} --
  \ref{req:c1-gen} of Lemma \ref{lem:c1} are satisfied (see the proof
  that \ref{req:t2-is-repr-iff-conds-conc-equal},
  \ref{req:t2-is-repr-iff-conds-roles-equal} and
  \ref{req:t2-is-repr-iff-conds-roles-gen-inv} of Lemma
  \ref{lem:t2-is-repr-iff-conds} are satisfied above). Then, by Lemma
  \ref{lem:c1} we have that there exist $\delta, \sigma \in
  \Delta^{\srkb[\{B_1(o), B_2(o)\}]}$ such that one of the following
  holds:
  \begin{enumerate}[reqs]
  \item \label{req:lem-representab-conds-incons-conc} There are
    concepts $C, C' \in \ttype{\Uni_{\srkb[\{B_1(o),
        B_2(o)\}]}}(\delta)$ such that $(C \AND C' \ISA \bot) \in \T_2
    \cup \T_{12}$;
  \item \label{req:lem-representab-conds-incons-role} There are roles
    $Q, Q' \in \rtype{\Uni_{\srkb[\{B_1(o), B_2(o)\}]}}(\delta,
    \sigma)$ such that $(Q \AND Q' \ISA \bot) \in \T_2 \cup \T_{12}$.
  \end{enumerate}
  Assume \ref{req:lem-representab-conds-incons-conc} and observe that
  w.l.o.g. $C'$ is over $\Xi$, whereas $C \in \Sigma \cup \Xi$. If $C
  \in \Sigma$ it follows $(C \AND C') \in \T_{12}$ and we immediately
  have the contradiction to the fact that $B_1, B_2$ is $\T_1 \cup
  \T_{12}$ is consistent. So let $C$ be over $\Xi$, it follows $(C
  \AND C') \in \T_2$, and $\T_1 \cup \T_{12}$ is closed under
  disjointness between $C$ and $C'$ by the definition of $\T_2$.
  Consider, first, the case $\delta \neq o$: by Lemma
  \ref{lem:Unitype-KBtype-ABox} $C,C' \in
  \stype[\Xi]{\Uni_{\srkb[\exists Q^-]}}(o)$ for the role $Q$ such
  that $\tail(\delta)=w_{[Q]}$. If $Q$ is over $\Sigma$ we derive the
  contradiction because $\exists Q^-$ is $\T_1 \cup
  \T_{12}$-consistent and \ref{req:def-disj-clo-conc-1}. On the other
  hand, if $Q$ is over $\Xi$, it can be seen by the structure of
  $\Uni_{\srkb[\{B_1(o), B_2(o)\}]}$ that $o \leadsto_{\srkb} w_{[Q]}$
  for some $B \in \consc$ over $\Sigma$; then $C,C' \in
  \ttype[\Xi]{\Uni_{\srkb}}(o w_{[Q]})$ and we derive the
  contradiction because of \ref{req:def-disj-clo-conc-2}. Finally,
  consider the case $\delta=o$, then by the structure of
  $\Uni_{\srkb[\{B_1(o), B_2(o)\}]}$ there are concepts $B,D \in
  \{B_1, B_2\}$ such that $C \in \ttype[\Xi]{\Uni_{\srkb}}(o)$ and $C'
  \in \ttype[\Xi]{\Uni_{\srkb[D]}}(o)$. By
  \ref{req:def-disj-clo-conc-1} we again have a contradiction.

  Assume \ref{req:lem-representab-conds-incons-role}, then again,
  assuming $Q$ is over $\Sigma$ produces an immediate contradiction;
  if, however, $Q$ is over $\Xi$, we obtain by the definition of
  $\T_2$, that $\T_1 \cup \T_{12}$ is closed under disjointness
  between $Q$ and $Q'$.  By the structure of $\Uni_{\srkb[\{B_1(o),
    B_2(o)\}]}$ we need to consider two cases: $\sigma=\delta
  w_{[R]}$, $Q, Q' \in \rtype{\Uni_{\srkb[\{B_1(o),
      B_2(o)\}]}}(\delta, \sigma)$ and $\delta = \sigma w_{[R]}$,
  $Q^-, Q'^- \in \rtype{\Uni_{\srkb[\{B_1(o), B_2(o)\}]}}(\sigma,
  \delta)$. In the first case, $o \leadsto_{\srkb} w_{[R]}$ for some
  $B \in \consc$ over $\Sigma$ and $Q,Q' \in
  \rtype[\Xi]{\Uni_{\srkb}}(o, o w_{[R]})$; using
  \ref{req:def-disj-clo-roles-2} we derive the contradiction. The
  second case is proved analogously using
  \ref{req:def-disj-clo-roles-2}.

  Thus, assuming the pair $B_1, B_2$ is $\T_1 \cup \T_{12}$-consistent
  produces a contradiction, therefore $B_1, B_2$ is $\T_1 \cup
  \T_{12}$ inconsistent. This concludes the proof that
  \ref{req:t2-is-repr-iff-conds-conc-cons} of Lemma
  \ref{lem:representab-conds} is satisfied. Analogously, using
  \ref{req:representab-conds-role-cons}, Lemma \ref{lem:c1},
  \ref{req:def-disj-clo-conc-1}, \ref{req:def-disj-clo-conc-2},
  \ref{req:def-disj-clo-roles-1}, \ref{req:def-disj-clo-roles-2}, it
  can be shown that \ref{req:t2-is-repr-iff-conds-roles-cons} of Lemma
  \ref{lem:representab-conds} is satisfied, which concludes the proof
  ($\Leftarrow$) of Lemma \ref{lem:representab-conds}.

  ($\Rightarrow$) Assume $\T_2$ is a representation for
  $\T_1$ under $\T_{12}$, we show that
  \ref{req:representab-conds-conc-cons} --
  \ref{req:representab-conds-gen-role} are satisfied. For
  \ref{req:representab-conds-conc-cons} assume a $\T_1$-consistent
  pair of concepts $B_1, B_2$, such that $B_1, B_2$ is $\T_1 \cup
  \T_{12}$-inconsistent; it follows Lemma by
  \ref{lem:t2-is-repr-iff-conds}
  \ref{req:t2-is-repr-iff-conds-conc-cons} that $\tgkb[\{B_1(o),
  B_2(o)\}]$ is inconsistent. Then, one of the following holds:
  \begin{enumerate}[reqs]
  \item \label{req:lem-representab-conds-incons-conc-tg} There are
    concepts $C, C' \in \ttype{\Uni_{\tgkb[\{B_1(o),
        B_2(o)\}]}}(\delta)$ such that $(C \AND C' \ISA \bot) \in \T_2
    \cup \T_{12}$;
  \item \label{req:lem-representab-conds-incons-role-tg} There are
    roles $Q, Q' \in \rtype{\Uni_{\tgkb[\{B_1(o), B_2(o)\}]}}(\delta,
    \sigma)$ such that $(Q \AND Q' \ISA \bot) \in \T_2 \cup \T_{12}$.
  \end{enumerate}
  Assume \ref{req:lem-representab-conds-incons-conc-tg} is the case
  and notice that w.l.o.g. we can assume $C'$ is over $\Xi$ and $C$ is
  over $\Sigma \cup \Xi$. Let, first, $\delta=o$, by the structure of
  $\tgkb[\{B_1(o), B_2(o)\}]$ it follows there are $B \in \{B_1,
  B_2\}$ and $B'$ is over $\Xi$ such that $\T_{12} \vdash B \ISA B'$
  and $\T_2 \vdash B' \ISA C'$. Suppose $C$ is over $\Xi$, then it
  follows $(C \AND C' \ISA \bot) \in \T_2$ and, again, there are $D
  \in \{B_1, B_2\}$ and $D'$ is over $\Xi$ such that $\T_{12} \vdash D
  \ISA D'$ and $\T_2 \vdash D' \ISA C$. By Lemma
  \ref{lem:t2-repr-then-closed}
  \ref{tmpt:lem-t2-repr-then-closed-disj} it follows $\T_1 \cup
  \T_{12}$ is closed under disjointness between $B'$ and $D'$, so
  \ref{req:representab-conds-conc-cons}\ref{req:representab-conds-conc-cons-1}\ref{req:representab-conds-conc-cons-1-1}
  is satisfied. Suppose $C$ is over $\Sigma$, then $(C \AND C' \ISA
  \bot) \in \T_{12}$, and by the structure of $\tgkb[\{B_1(o),
  B_2(o)\}]$ it follows $C \in \{B_1, B_2\}$. By Lemma
  \ref{lem:t2-repr-then-closed}
  \ref{tmpt:lem-t2-repr-then-closed-disj} it follows $\T_1 \cup
  \T_{12}$ is closed under inclusion between $B'$ and $C'$, so
  \ref{req:representab-conds-conc-cons}\ref{req:representab-conds-conc-cons-1}\ref{req:representab-conds-conc-cons-1-2}
  is satisfied. Consider now the case $\tail(\delta)=w_{[R]}$ for $R
  \in \Sigma$; by the structure of $\tgkb[\{B_1(o), B_2(o)\}]$ it
  follows $\exists R \in \{B_1, B_2\}$ and by Lemma
  \ref{lem:Unitype-KBtype-ABox} $\T_2 \cup \T_{12} \vdash \exists R^-
  \ISA C$, $\T_2 \cup \T_{12} \vdash \exists R^- \ISA C'$. Now we can
  repeat the argument above with $B=D=\exists R^-$ to conclude either
  that either
  \ref{req:representab-conds-conc-cons}\ref{req:representab-conds-conc-cons-2}\ref{req:representab-conds-conc-cons-2-1}
  or \ref{req:representab-conds-conc-cons-2-2} is satisfied.

  Finally, consider the case $\tail(\delta)=w_{[R']}$ with $R'$ over
  $\Xi$. By Lemma \ref{lem:Unitype-KBtype-ABox} it is the case $\T_2
  \vdash \exists R'^- \ISA C$, $\T_2 \vdash \exists R'^- \ISA C'$. If
  $o \leadsto_{\tgkb[\{B_1(o), B_2(o)\}]} w_{[R']}$, then by the
  structure of $\tgkb[\{B_1(o), B_2(o)\}]$ it follows $\T_{12} \vdash
  B \ISA B'$ and $o \leadsto_{\tup{\T_2, \{B'(o)\}}} w_{[R']}$ for
  some $B \in \{B_1, B_2\}$, $B'$ over $\Xi$; also by Lemmas
  \ref{lem:types-chase} and \ref{lem:kr-to-all-b-cons} it follows the
  concept $B'$ is $\T_2$ inconsistent. Since by Lemma
  \ref{lem:t2-repr-then-closed}
  \ref{tmpt:lem-t2-repr-then-closed-incons} $\T_1 \cup \T_{12}$ is
  closed under the disjointness between $B'$ and $B'$, it follows
  \ref{req:representab-conds-conc-cons}\ref{req:representab-conds-conc-cons-1}\ref{req:representab-conds-conc-cons-1-1}
  is satisfied. If it is not the case $o \leadsto_{\tgkb[\{B_1(o),
    B_2(o)\}]} w_{[R']}$, it follows $\exists R \in \{B_1, B_2\}$ and
  $\T_{12} \vdash \exists R^- \ISA B'$ for some $B'$ over $\Xi$, such
  that there is $\sigma \in \dom[\Uni_{\tup{\T_2, \{B'(o)\}}}]$ with
  $\tail(\sigma)=w_{[R']}$. Again, by Lemmas~\ref{lem:types-chase} and
  \ref{lem:kr-to-all-b-cons} $B'$ is $\T_2$ inconsistent, then by
  Lemma \ref{lem:t2-repr-then-closed}
  \ref{tmpt:lem-t2-repr-then-closed-incons} $\T_1 \cup \T_{12}$ is
  closed under disjointness between $B'$ and $B'$, so
  \ref{req:representab-conds-conc-cons}\ref{req:representab-conds-conc-cons-2}\ref{req:representab-conds-conc-cons-2-1}
  is satisfied.

  Assume \ref{req:lem-representab-conds-incons-role-tg} is the case
  and notice that w.l.o.g. we can assume $Q'$ is over $\Xi$ and $Q$ is
  over $\Sigma \cup \Xi$. By the structure of $\Uni_{\tgkb[\{B_1(o),
    B_2(o)\}]}$ we need to consider two cases: $\sigma=\delta  w_{[R]}$, $Q, Q' \in \rtype{\Uni_{\tgkb[\{B_1(o),
      B_2(o)\}]}}(\delta, \sigma)$ and $\delta = \sigma w_{[R]}$, $Q^-, Q'^- \in \rtype{\Uni_{\tgkb[\{B_1(o),
      B_2(o)\}]}}(\sigma, \delta)$. We show only the first case, the
  second case is analogous. Assume $\sigma=o$, $o
  \leadsto_{\tgkb[\{B_1(o), B_2(o)\}]} w_{[R]}$ for $R$ over $\Sigma$,
  and $Q$ over $\Xi$. It follows $\T_{12} \vdash R \ISA R'$ and $\T_2
  \vdash R' \ISA Q'$, and also $\T_{12} \vdash R \ISA S$ and $\T_2
  \vdash S \ISA Q'$ for some $R', S$ over $\Xi$. Since $(Q \AND Q'
  \ISA \bot) \in \T_2$ by Lemma
  \ref{lem:t2-repr-then-closed}\ref{tmpt:lem-t2-repr-then-closed-disj}
  we get $\T_1 \cup \T_{12}$ is closed under inclusion between $R'$
  and $S$, so
  \ref{req:representab-conds-conc-cons}\ref{req:representab-conds-conc-cons-2}\ref{req:representab-conds-conc-cons-2-3}
  is satisfied. Let $Q \in \Sigma$, it follows $o
  \leadsto_{\tgkb[\{B_1(o), B_2(o)\}]} w_{[R]}$ and $R=Q$. It follows
  also $\T_{12} \vdash Q \ISA R'$ and $\T_2 \vdash R' \ISA Q'$, then
  by Lemma
  \ref{lem:t2-repr-then-closed}\ref{tmpt:lem-t2-repr-then-closed-incl}
  we get $\T_1 \cup \T_{12}$ is closed under inclusion between $R'$
  and $Q'$, so, since $(Q \AND Q' \ISA \bot) \in \T_{12}$, we conclude
  \ref{req:representab-conds-conc-cons}\ref{req:representab-conds-conc-cons-2}\ref{req:representab-conds-conc-cons-2-4}
  is satisfied. Consider now the case $o \leadsto_{\tgkb[\{B_1(o),
    B_2(o)\}]} w_{[R]}$ for $R$ over $\Xi$, which implies $Q$ is over
  $\Xi$ and $(Q \AND Q' \ISA \bot) \in \T_2$; then $\T_{12} \vdash B
  \ISA B'$ and $\T_2 \vdash B' \ISA \exists R$ for some concepts $B
  \in \{B_1, B_2\}$ and $B'$ over $\Xi$. It follows $o
  \leadsto_{\tup{\T_2, \{B'(o)\}}} w_{[R]}$, then by Lemmas
  \ref{lem:types-chase} and \ref{lem:kr-to-all-b-cons} $B'$ is $\T_2$
  inconsistent, then by Lemma \ref{lem:t2-repr-then-closed}
  \ref{tmpt:lem-t2-repr-then-closed-incons} $\T_1 \cup \T_{12}$ is
  closed under disjointness between $B'$ and $B'$, so
  \ref{req:representab-conds-conc-cons}\ref{req:representab-conds-conc-cons-2}\ref{req:representab-conds-conc-cons-2-1}
  is satisfied. This concludes the proof for the case $\sigma=o$.

  Assume $\tail(\sigma)=w_{[R']}$ for $R'$ over $\Sigma$, this implies
  $o \leadsto_{\tgkb{\{B_1(o), B_2(o)\}}} w_{[R']}$, and we lead the
  proof analogously to the case above to show there is a $\T_2$
  inconsistent $B'$ such that $\T_{12} \vdash \exists R^- \ISA B'$ and
  \ref{req:representab-conds-conc-cons}\ref{req:representab-conds-conc-cons-2}\ref{req:representab-conds-conc-cons-2-3}
  is satisfied. If $\tail(\sigma)=w_{[R']}$ for $R'$ over $\Xi$ it can
  be easily verified
  \ref{req:representab-conds-conc-cons}\ref{req:representab-conds-conc-cons-2}\ref{req:representab-conds-conc-cons-2-1}
  or
  \ref{req:representab-conds-conc-cons}\ref{req:representab-conds-conc-cons-2}\ref{req:representab-conds-conc-cons-2-3}
  is satisfied. This concludes the proof that
  \ref{req:representab-conds-conc-cons} is satisfied; then
  \ref{req:representab-conds-role-cons} can be shown analogously.

  To show \ref{req:representab-conds-equal-conc} is satisfied assume
  $B \in \consc$ over $\Sigma$ and $B' \in
  \ttype[\Xi]{\Uni_{\srkb}}(o)$. By
  Lemma~\ref{lem:t2-is-repr-iff-conds}
  \ref{req:t2-is-repr-iff-conds-conc-equal} it follows $B' \in
  \ttype[\Xi]{\Uni_{\tgkb}}(o)$, so there exists $C'$ over $\Xi$ such
  that $\T_{12} \vdash B \ISA C'$ and $\T_2 \vdash C' \ISA B'$. By
  Lemma~\ref{lem:t2-repr-then-closed}
  \ref{tmpt:lem-t2-repr-then-closed-incl} it follows $\T_1 \cup
  \T_{12}$ is closed under inclusion between $C'$ and $B'$; then
  \ref{req:representab-conds-equal-role} can be shown analogously.

  Finally, we show \ref{req:representab-conds-gen-role} is satisfied;
  assume $B \in \consc$ over $\Sigma$ and $o \leadsto_{\srkb} w_{[R]}$
  for some role $R$, by Lemma~\ref{lem:t2-is-repr-iff-conds}
  \ref{req:t2-is-repr-iff-conds-roles-gen} it follows there exists $y
  \in \dom[\Uni_{\tgkb}]$ such that $\ttype[\Xi]{\Uni_{\srkb}}(o
  w_{[R]}) \subseteq \ttype[\Xi]{\Uni_{\tgkb}} (y)$, and
  $\rtype[\Xi]{\Uni_{\srkb}}(o,o w_{[R]}) \subseteq
  \rtype[\Xi]{\Uni_{\tgkb}} (o,y)$. By the structure of $\tgkb$ it
  follows there exists a sequence of concepts $\tup{C_0, \dots, C_n}=
  \tup{B, \exists Q_1^-, \dots, \exists Q_n^-}$ such that $\T_2 \cup
  \T_{12} \vdash C_i \ISA \exists Q$ for all $0 \leq i < n$ and roles
  $Q$ such that $C_{i+1} = \exists Q^-$, $\T_2 \cup \T_{12} \vdash C_n
  \ISA B'$ for all $B' \in \ttype[\Xi]{\Uni_{\srkb}}(o w_{[R]})$, and
  $\rtype[\Xi]{\Uni_{\srkb}}(o,o w_{[R]}) \neq \emptyset$ implies
  $n=1$ and $\T_2 \cup \T_{12} \vdash Q \ISA R'$ for all $R' \in
  \rtype[\Xi]{\Uni_{\srkb}}(o,o w_{[R]})$ and $Q$ such that $C_1 =
  \exists Q^-$. 
  We define a generating pass for $B$ conform with $\T_1 \cup \T_{12}$
  as follows: $L(C_n)=\stype[\Xi]{\srkb}(w_{[R]})$,
  $L(C_1,C_n)=\qtype[\Xi]{\srkb}(o,w_{[R]})$, $L(C_i)=\{\exists Q \mid
  C_{i+1}= \exists Q^-, B \neq \exists Q\}$ for all $0 \leq i < n$,
  and $L(C_i, C_j)=\emptyset$ for $j \neq i+1$. It can be
  straightforwardly verified that \ref{req:xi-pass-def-chain} holds,
  then also \ref{req:xi-pass-def-conc} and \ref{req:xi-pass-def-role}
  follow using Lemma \ref{lem:t2-repr-then-closed}. We have shown
  \ref{req:representab-conds-gen-role} is satisfied, which concludes
  the proof ($\Rightarrow$) of Lemma \ref{lem:representab-conds}.
\end{proof}
\begin{theorem}\label{th:non-empt-repres-nlogspace-compl}
  The non-emptyness problem for $\UCQ$-representability is
  \NLOGSPACE-complete.
\end{theorem}
\begin{proof}
  As in the case of Theorem~\ref{th:memb-repres-nlogspace-compl}, the
  lower bound is shown by the reduction from the directed graph
  reachability problem, however, we need a slightly more involved
  encoding.
  \begin{lemma}
    The non-emptyness problem for $\UCQ$-representability is
    \NLOGSPACE-hard.
  \end{lemma}
  \begin{proof}
    To encode the graph $\mathcal{G} = (\mathcal{V}, \mathcal{E})$, we
    need a set of $\Sigma$-concept names $\{V_i \mid v_i \in
    \mathcal{V}\} \cup \{S, F, X, Y\}$ and a set of $\Xi$-concept
    names $\{V_i' \mid v_i \in \mathcal{V}\} \cup \{S', X',
    Y'\}$. Consider the TBox
$$\T_1=\{ V_i \sqsubseteq V_j \mid (v_i,
v_j) \in \mathcal{E}\} \cup \{S \sqsubseteq V_k, V_m \sqsubseteq F, X
\sqsubseteq Y\},$$
where $v_k$ and $v_m$ are, respectively, the initial and final
vertices. Then, let 
$$\T_{12} = \{V_i \sqsubseteq V_i' \mid v_i \in
\mathcal{V}\} \cup \{S \sqsubseteq S', S \sqsubseteq X', F \sqsubseteq
Y', X \sqsubseteq X', Y \sqsubseteq Y'\};$$
we will show:
  \begin{proposition}
    There is a directed path from $v_k$ to $v_m$ in $\mathcal{G}$ iff
    there exists a representation for $\T_1$ under $\M=(\Sigma, \Xi,
    \T_{12})$.
  \end{proposition}
  Indeed, using Lemma~\ref{lem:representab-conds}, there exists a
  representation iff the
  condition~\ref{req:representab-conds-equal-conc} is satisfied. By
  the structure of $\T_1 \cup \T_{12}$ one can see that it is the case
  iff $\T_1 \cup \T_{12}$ is closed under the inclusion between $X'$
  and $Y'$. The latter is the case iff $\T_1 \cup \T_{12} \vdash S
  \sqsubseteq X'$ implies $\T_1 \cup \T_{12} \vdash S \sqsubseteq Y'$,
  and that holds iff $\T_1 \vdash S \sqsubseteq F$, which is the case
  iff there exists a path from $v_k$ to $v_m$ in $\mathcal{G}$. This
  completes the proof of
  Lemma~\ref{th:non-empt-repres-nlogspace-compl}. 
\end{proof}

To show the upper bound, we prove that the
conditions~\ref{req:representab-conds-equal-conc}--\ref{req:representab-conds-role-cons}
of Lemma~\ref{lem:representab-conds} can be checked in \NLOGSPACE. In
fact, these conditions can be checked using the algorithm, based on
directed graph reachability solving procedure, similar to the proof of
Theorem~\ref{th:memb-repres-nlogspace-compl}. The only new case is the
condition ~\ref{req:representab-conds-gen-role}; to verify that there
exists a generating pass $\pi = (\langle C_0, \dots C_n \rangle, L)$
for a concept $B$ conform with $\T_1 \cup \T_{12}$, we can use the
following procedure, running in \NLOGSPACE. First, we take $C_0 = B$
and decide, if the pass ends here (i.e., $n=1$). If we decided so, it
only remains to take $L(C_0) =\ttype[\Xi]{\Uni_{\srkb}}(o w_{[R]})$,
for $\srkb$ and $R$ as in the
condition~\ref{req:representab-conds-gen-role}, and
verify~\ref{req:xi-pass-def-conc}. This verification can be performed
in \NLOGSPACE, similarly to the method described in the proof of
Theorem~\ref{th:memb-repres-nlogspace-compl}. If, on the other hand,
we decide, that the pass continues, we ``guess'' $C_1 = \exists Q^-$
for some role $Q$, and verify that for some $L(C_0) \subseteq
\{\exists Q\}$ the \ref{req:xi-pass-def-chain} and
\ref{req:xi-pass-def-conc} are satisfied. Now, if we decide that the
pass stops, it remains to take $L(C_1) =\ttype[\Xi]{\Uni_{\srkb}}(o
w_{[R]})$ and $L(C_0, C_1)=\rtype[\Xi]{\Uni_{\srkb}}(o, o w_{[R]})$ ,
for $\srkb$ and $R$ as in the
condition~\ref{req:representab-conds-gen-role}, and
verify~\ref{req:xi-pass-def-conc} and~\ref{req:xi-pass-def-role}. If,
on the contrary, we decide that the pass continues, we can ``forget''
$C_0$, ``guess'' $C_2$, and proceed with it in the same way, as we did
with $C_1$. Finally, when we reach the concept $C_n$, such that the
algorithm decides to stop, it remains to
verify~\ref{req:xi-pass-def-conc} for $L(C_n)
=\ttype[\Xi]{\Uni_{\srkb}}(o w_{[R]})$. It should be clear that
whenever the generating pass $\pi = (\langle C_0, \dots C_n \rangle,
L)$ for a concept $B$ conform with $\T_1 \cup \T_{12}$ exists, we can
find it by the above non-determinictic procedure.
\end{proof}


\fi

\end{document}